\documentclass[12pt, draftclsnofoot, onecolumn]{IEEEtran}
\usepackage[caption=false]{subfig}
\usepackage[short]{optidef}
\usepackage{amsmath,amsfonts, amsthm, amssymb}
\usepackage{xcolor}
\usepackage{array}
\usepackage{subfig}
\usepackage{textcomp}
\usepackage{stfloats}
\usepackage{url}
\usepackage{verbatim}
\usepackage{graphicx}
\usepackage{cite}
\usepackage{multirow}
\usepackage{makecell}
\usepackage{enumitem}
\usepackage{cleveref}
\usepackage{dsfont}
\usepackage{bm}
\usepackage{booktabs} 
\usepackage[linesnumbered,ruled,vlined]{algorithm2e}

\DeclareMathOperator*{\argmin}{arg\,min}
\newcommand{\twonormsquare}[1]{\big\lVert#1\big\rVert^2}
\newcommand{\norm}[1]{\big\lVert#1\big\rVert}
\newcommand{\innerproduct}[2]{\left\langle #1, #2 \right\rangle}
\newcommand{\expnumber}[2]{{#1}\mathrm{e}{#2}}

\DeclareMathOperator*{\ssum}{\textstyle\sum}

\newtheorem{lemma}{Lemma}

\newtheorem{definition}{Definition}
\newtheorem{fact}{Fact}
\newtheorem{assumption}{Assumption}
\newtheorem{theorem}{Theorem} 
\newtheorem{corollary}{Corollary}
\allowdisplaybreaks

\begin{document}

\title{Asynchronous Multi-Model Dynamic Federated Learning over Wireless Networks: Theory, Modeling, and Optimization}

\author{Zhan-Lun Chang,~\IEEEmembership{Member,~IEEE,}
        Seyyedali Hosseinalipour,~\IEEEmembership{Member,~IEEE,} 
        Mung Chiang,~\IEEEmembership{Fellow,~IEEE,} Christopher G. Brinton,~\IEEEmembership{Senior Member,~IEEE}
}



\maketitle

\begin{abstract}

Federated learning (FL) has emerged as a key technique for distributed machine learning (ML). Most literature on FL has focused on ML model training for (i) a \textit{single task/model}, with (ii) a \textit{synchronous} scheme for updating model parameters, and (iii) a \textit{static} data distribution setting across devices, which is often not realistic in practical wireless environments. To address this, we develop {\tt DMA-FL} considering \textit{d}ynamic FL with \textit{m}ultiple downstream tasks/models over an \textit{a}synchronous model update architecture. We first characterize convergence via introducing scheduling tensors and rectangular functions to capture the impact of system parameters on learning performance. Our analysis sheds light on the joint impact of device training variables (e.g., number of local gradient descent steps), asynchronous scheduling decisions (i.e., when a device trains a task), and dynamic data drifts on the performance of ML training for different tasks. Leveraging these results, we formulate an optimization for jointly configuring resource allocation and device scheduling to strike an efficient trade-off between energy consumption and ML performance. Our solver for the resulting non-convex mixed integer program employs constraint relaxations and successive convex approximations with convergence guarantees. Through numerical experiments, we reveal that {\tt DMA-FL} substantially improves the performance-efficiency tradeoff.

\end{abstract}


\section{Introduction}
The proliferation of intelligent Internet-of-Things (IoT) devices (e.g., mobile phones and smart vehicles) has caused unprecedented growth in the amount of generated data at the network edge \cite{zhang2018synergy}. There is a strong demand to leverage this data to enable machine learning (ML)-driven services for both user applications (e.g., object tracking for self-driving cars) and network optimization (e.g., wireless signal denoising). Latency and privacy constraints associated with transferring the collected data to a central location for training (i.e., at the edge or cloud servers) have led to research on distributing ML over the network edge \cite{hosseinalipour2020federated}.

\subsection{Federated Learning (FL) and Practical Considerations}
Federated learning (FL) in particular has attracted significant attention in recent years as a solution for distributed ML \cite{mcmahan2017communication,smith2017federated}. The basic premise of FL is to conduct ML model training through two processes repeated in sequence:
\begin{enumerate}[label = (\roman*)]
    \item \textit{Local updates}: Devices update their local models based on their local datasets, often through gradient descent iterations.
    
    \item \textit{Global aggregations}: The local models of the devices are pulled by a server periodically to obtain a new global model, usually through weighted averaging, which is then synchronized across the devices to begin the next local training round.
\end{enumerate}

The common implementation described here is often referred to as the FedAvg algorithm~\cite{mcmahan2017communication}. Several research directions have focused on expanding this framework to account for realistic factors of the wireless edge. One line of work has considered various dimensions of heterogeneity, e.g., variations among device resources and statistical diversity across local datasets~\cite{nishio2019client}. Another direction has considered security vulnerabilities of FL, e.g., model poisoning and backdoor attacks~\cite{10191260,li2022multitentacle}. Three other important factors also warrant careful consideration:

\subsubsection{Multiple Tasks/Models} Many contemporary edge intelligence settings require IoT devices engaging in training \textit{multiple tasks simultaneously} (e.g., smart cars need models for lane tracking, pedestrian detection, and asphalt condition classification~\cite{bello2018new,karuppuswamy2000detection,chen2015deepdriving}). Each of these tasks may require training a separate neural network (NN) on disparate datasets. Nevertheless, current implementations of FL are mainly focused on the system design for training a single task. Multiple tasks will induce competition for the limited resources of devices. These resources must be carefully allocated given comparative task attributes such as target performance, model size, and relative importance.

\subsubsection{Asynchronous Aggregations} Conventional FL considers a synchronized aggregation process where all sampled devices upload their local models simultaneously. However, the server may receive the models at different times due to two types of resource heterogeneity in (i) computation capabilities, leading to different local model training times, and (ii) device-server channel conditions, leading to latency in uplink transmissions. Waiting for all  models to arrive at the server can introduce prohibitive service delays in the presence of stragglers. This has motivated recent investigations into  asynchronous FL\cite{xie2019asynchronous,lian2018asynchronous,zheng2017asynchronous, chen2020asynchronous}, where the server sequentially updates the global model upon reception of any new local model. However, the impact of the \textit{order} in which the devices upload their models to the server (i.e., device scheduling) has yet to be carefully investigated. 

\subsubsection{Dynamic Data Statistics} In a realistic system, the data generated at the network edge is  \textit{time-varying} (e.g., images captured by an IoT camera at different times of day, or under a changing environment). Thus, the \textit{drift} of the performance of the global model under data variations should be explicitly taken into account. In particular, efficient device scheduling should consider both model drift and resource constraints, where devices with faster data variations and more abundant resources engage in more rapid model training and a higher frequency of model exchange with the server. 

The integration of multi-task, asynchronous operation, and dynamic data considerations in heterogeneous FL creates unique challenges. With dynamic datasets and non-iid data statistics, the order in which devices participate for each task becomes important, especially in a multi-task learning context. Specifically, for a task whose data changes rapidly, more frequent transmission of the models between the device and the server is needed to obtain good model performance. Similarly, devices with datasets that have unique properties may need to upload their local models sooner than others, since increasing the period of local model training will further bias their local models. Resource heterogeneity adds yet another challenge in the asynchronous setting: the capabilities of devices should be considered jointly with the importance and drift of their local models in determining device scheduling. Further, in the multi-task setting, the available resources of each device need to be carefully divided/allocated to different tasks. Our objective in this work is to develop the first FL solution which systematically models and optimizes over these interdependencies.


\section{Related Work}\label{sec:relatedW}

\textbf{Single-Model Synchronous FL}: Many research efforts have been devoted to ``conventional'' FL, which is characterized by a single model and synchronous operation, i.e. as in the initial FedAvg algorithm \cite{mcmahan2017communication}. See e.g., \cite{imteaj2021survey} for a recent survey of techniques. One key direction of these works has been on characterizing the model training process under heterogeneous system conditions and employing the results to improve FL convergence speeds, e.g., as in \cite{nishio2019client, abad2020hierarchical}.

In this category, the most direct precursors to our work are those considering device scheduling optimization in conventional FL, e.g., \cite{shi2020joint,taik2021data,amiri2021convergence}. \cite{shi2020joint} proposed a joint device scheduling and resource allocation policy to maximize the model accuracy within a given total training time budget for latency constrained wireless FL. \cite{taik2021data} proposed data-aware wireless scheduling algorithm to minimize the completion time and the transmission energy consumption. \cite{amiri2021convergence} developed innovative scheduling and resource allocation policies that determine the subset of devices to transmit in each round and how resources should be allocated among participating devices.   

\textbf{Multi-Task/Model Synchronous FL}: A few recent works have considered multi-model/task FL \cite{smith2017federated,sattler2020clustered,zhou2022efficient,nguyen2020toward}.  \cite{smith2017federated} proposed a system-level optimization to address the high communication cost, node heterogeneity, and fault tolerance aspects of FL with multiple tasks to train. \cite{sattler2020clustered} leveraged geometric properties of loss functions to cluster clients into groups with trainable data distributions, with each cluster corresponding to a different task. \cite{zhou2022efficient} proposed a reinforcement learning-based technique to tackle the device-task assignment problem. \cite{nguyen2020toward} investigated joint network resource optimization and hyperparameter control for multi-task FL. However, these works focus on the traditional synchronous FL setting.   

\textbf{Asynchronous FL}: Asynchronous FL has been investigated in \cite{xie2019asynchronous,ma2021fedsa,lian2018asynchronous,zheng2017asynchronous, chen2020asynchronous, chen2019communication}. \cite{xie2019asynchronous} for the first time analyzed the model convergence of FL under asynchronous arrivals from devices. \cite{ma2021fedsa} analyzed the semi-asynchronous FL where the parameter server aggregates a certain number of local models by their arrival order in each round. \cite{lian2018asynchronous} proposed decentralized stochastic gradient descent (SGD) to solve the communication bottleneck that arises in asynchronous FL with a congested server. \cite{zheng2017asynchronous} leveraged efficient approximation techniques to develop a methodology that compensates for uplink transmission delays in asynchronous FL. \cite{chen2020asynchronous} proposed an asynchronous FL framework with feature representation learning at the server and dynamic local updates at the clients to handle straggler effects. This work is based on a strong assumption in the convergence analysis that the server will select a single device to transmit the global model to and then quickly receive a model update from it to conduct the next aggregation. \cite{chen2019communication} proposed a synchronous learning strategy on the clients and a temporally weighted aggregation of the local models on the server to make use of the previously trained local models.

\textbf{Online/Dynamic FL}: There exist some research on dynamic/online FL \cite{gauthier2021resource,giorgas2020online,li2019online}. \cite{gauthier2021resource} made use of partial-sharing-based communication to reduce the communication overhead in online asynchronous FL. \cite{giorgas2020online} proposed an experimental technique to boost performance.  \cite{li2019online} studied multi-task FL with the aim of learning model parameters for new incoming devices. Nevertheless,
the ``online" aspect of FL considered in \cite{gauthier2021resource,giorgas2020online,li2019online} refers to dataset/model sampling or device arrival dynamics, which is different than the notion of  data variations we are interested in. The few recent works \cite{hosseinalipour2023parallel,dynamicFedl} which have considered online FL in our context  focused on the conventional single-task and synchronous FL settings. \cite{durmus2021federated} proposed a dynamic regularizer for each device at each round, so that in the limit the global and device solutions are aligned.

None of these existing works have taken advantage of the full potential of asynchronous FL, as they have not investigated the impact of the \textit{orders} in which the devices (i) receive the global model and (ii) return their local models to the server. We refer to this as the scheduling of the devices for FL. In this work, we study device scheduling for fully-asynchronous and multi-model FL, without any restrictive conditions on device participation.

\subsection{Outline and Summary of Contributions}
In this work, we propose {\tt DMA-FL}, a novel methodology for \underline{d}ynamic/online \underline{m}ulti-model \underline{a}synchronous FL over heterogeneous networks. In doing so, we make the following contributions:
\begin{itemize}
    \item We formulate {\tt DMA-FL} to consider ML over real-world edge settings where devices (i) continuously collect and discard their local data (ii) have heterogeneous communication/ computation resources, and (iii) train multiple ML models locally (\Cref{Sec:model}).
    
    \item We analytically characterize the ML model convergence performance of {\tt DMA-FL} (\Cref{sec:analysis}). Through the introduction of rectangular (\textit{rect}) functions scheduling tensors to capture device scheduling, we relax current assumptions used to derive performance bounds for asynchronous FL. Our analysis reveals insights into the effect several system and learning parameters on model training performance, including asynchronous scheduling decisions, network-wide task staleness, concept drifts incurred during idle and active training periods, and the number of gradient iterations conducted.
    \item Leveraging these convergence results, we formulate the joint device scheduling and resource allocation optimization problem for {\tt DMA-FL}, considering the tradeoff between multi-task ML model performance and network resource consumption (\Cref{Sec:optimization}). We investigate the behavior of the resulting NP-Hard problem by scrutinizing the ML convergence bound, providing new insights into the scheduling variables and resource allocation. Our solution methodology employs a series of transformations, decompositions, and successive convex programming to guarantee convergence to a stationary point of a relaxed version of the original non-convex mixed integer program.
    
    \item Through numerical experiments (\Cref{sec:num}), we reveal the superiority of {\tt DMA-FL} over baseline methods in terms of ML model performance and network resource savings, in the presence of dynamic data variations and multiple tasks. We also show how {\tt DMA-FL} can adapt to account for varying task importance and keep track of drastic data variations via reducing the idle time and local model training time.
\end{itemize}

\begin{figure*}
    \centering
    \includegraphics[width=\linewidth]{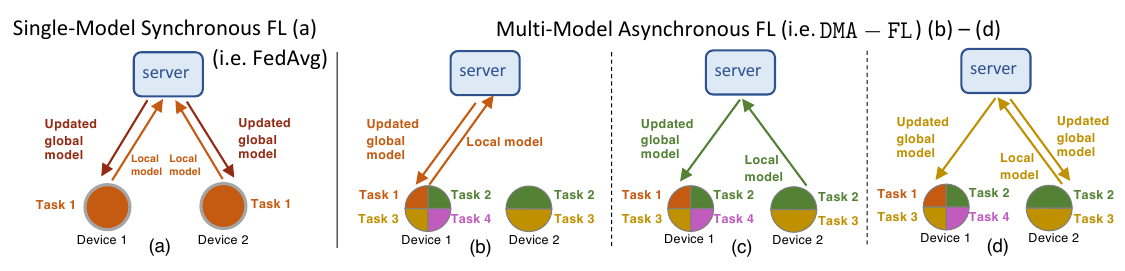}
\caption{Difference in architectures between single-model synchronous FL and our proposed {\tt DMA-FL} methodology. In the single-model synchronous FL, the server needs to wait for all the trained local model for a single task before it can perform model aggregation, after which the new global model is broadcast back to all devices. In contrast, for {\tt DMA-FL}, the server performs the global aggregation instantly when receiving one trained local model for any task. The new global model for that task can be transmitted to one or more devices which train model for that task.}
\label{fig:architecture_overview}
\end{figure*}

\section{System Model}\label{Sec:model}
\subsection{Setup and Overview}
\noindent  We consider an edge network of $I$ devices collected via the set $\mathcal{I} = \{1,\cdots, I\}$ connected upstream to a server. There are $J$ ML tasks (e.g., facial recognition, keyboard next word prediction)  gathered via the set $\mathcal{J} = \{1, \cdots, J\}$. Each task $j\in\mathcal{J}$ is associated with a unique ML model (e.g., a neural network) characterized by $M_j\in\mathbb{Z}^+$ parameters. The devices engage in uplink/downlink FL model transfers to/from the server to train the $J$ models. We consider asynchrony in this process both across the devices (i.e., devices engage in
uplink/downlink transfer of their models with the server at different times) and across the tasks (i.e., each device only updates a portion of its tasks at each uplink/downlink communication).

In {\tt DMA-FL}, training of each task $j\in\mathcal{J}$ is conducted via a sequence of global aggregations, indexed by $g \in \mathcal{G}_j = \{1,\cdots, G_j\}$ with $G_j$ denoting the total number of model aggregations for the respective task.\footnote{For notational simplicity, we have dropped the task index $j$ from the aggregation index $g$.} An aggregation for task $j$ is triggered when the server receives model parameters on that task from a device, at which point the server updates the current global model based on the single received model.

We elaborate on the difference in architectures between single-model synchronous FL (i.e. FedAvg) and {\tt DMA-FL} in Fig. \ref{fig:architecture_overview}. The server in single-model synchronous FL has to wait for all trained local models before performing global aggregation to generate an updated global model, which is broadcast back to sampled devices. In contrast, in {\tt DMA-FL}, if the server gets any trained local model for task $j$, it performs the global aggregation to yield an updated global model for task $j$, which can be transmitted to multiple devices activating their local training if they are not training task $j$ at that moment.

As part of {\tt DMA-FL}, we develop fully-asynchronous FL, a generalization of asynchronous FL~\cite{xie2019asynchronous,lian2018asynchronous,zheng2017asynchronous, chen2020asynchronous} where for the training of each task $j$, (i) there is no constraint on the number of devices receiving the global model from the server at any time and (ii) the order in which the devices return their updated models to the server can be different compared to the order in which they had received the global model (e.g., due to different communication/computation resources). This makes our convergence bounds, obtained for \textit{each task} (Sec.~\ref{subsec:interpret}), unique from the existing literature.

\subsection{ML Model Training}
\subsubsection{Task Formulation} 
In {\tt DMA-FL}, we consider the set of datapoints $\mathcal{D}_j^{(g)}$ for each task $j \in \mathcal{J}$. Each edge device $i \in \mathcal{I}$ contains a subset of these datapoints $\mathcal{D}_{i,j}^{(g)}$, where $\mathcal{D}_j^{(g)} = \cup_{i \in \mathcal{I}} \mathcal{D}_{i,j}^{(g)}$. For the global model for task $j$ at the $g$-th global aggregation, $\forall \mathbf{w}_j^{(g)}\in\mathbb{R}^{M_j}$, we define the \textit{global loss} function for task $j$ across these datapoints as
\begin{equation} \label{def:global_loss_function}
    F_j\left(\mathbf{w}_j^{(g)}, \mathcal{D}_j^{(g)} \right) \triangleq \frac{1}{ |\mathcal{I}| }\sum_{i\in\mathcal{I}} F_{i,j}(\mathbf{w}_j^{(g)}, \mathcal{D}_{i,j}^{(g)}), 
\end{equation}
Also,  $ F_{i,j}\big(\mathbf{w}, \mathcal{D}_{i,j}^{(g)}\big) \triangleq \frac{1}{D_{i,j}^{(g)}}\sum_{d \in \mathcal{D}_{i,j}^{(g)}}L_{i,j}\big(\mathbf{w}, d\big)$ is the \textit{local loss} function at device $i$ used to measure the performance of model parameter $\mathbf{w}$ for task $j$. Here,  $L_{i,j}\big(\mathbf{w}, d\big)$ measures the loss of data point $d$ under model parameter $\mathbf{w}$ and  $D_{i,j}^{(g)}=|\mathcal{D}_{i,j}^{(g)}|$.

The goal of ML model training is to minimize the instantaneous global loss function, which is used in real-time for the downstream tasks at the devices. Temporal variations of data make the local and global loss functions time-varying as well. Thus, for each task $j$, the optimal global model parameters form a sequence $\{(\mathbf{w}_j^{(g)})^\star\}_{g=1}^{G_j}$ where
\begin{equation}\label{eq:globGoal}
(\mathbf{w}_j^{(g)})^\star = \argmin_{\mathbf{w} \in \mathbb{R}^{M_j}} F_j(\mathbf{w}, \mathcal{D}_j^{(g)}),~ ~ g\in \mathcal{G}_j.
\end{equation}
While each task has its own dataset, training processes of tasks are coupled due to the limited computation/communication resources of devices, which have to be shared. We will address the resource allocation and device scheduling in Sec.~\ref{Sec:optimization}. We next introduce the local training and global aggregation steps of {\tt DMA-FL} to solve~\eqref{eq:globGoal} for a given scheduling decision.

\subsubsection{Local Updates and Global Aggregations\label{subsubsec:local_training_global_aggregation}}
Model training for each task $j \in \mathcal{J}$ starts with broadcasting a global model $\mathbf{w}_j^{(0)}$ from the server, which begins aggregation interval $g = 0$. We let  $\mathcal{T}_{i,j}^{\mathsf{L}, (g)}$ denote the \underline{l}ocal period  capturing the time period that begins when device $i$ receives the global model parameter $\mathbf{w}_j^{(g)}$ from the server and ends when the device transmits its updated local model parameter back to the server. As depicted in Fig. \ref{fig:local_period_def}, we divide the local period   $\mathcal{T}_{i,j}^{\mathsf{L}, (g)}$ into {four} parts: (i) \textit{\underline{Id}le period}  $\mathcal{T}_{i,j}^{\mathsf{ID},(g)}$: The device remains idle on this task. (ii) \textit{\underline{D}ownlink} transmission period  $\mathcal{T}_{i,j}^{\mathsf{D},(g)}$: The server sends the global model $\mathbf{w}_j^{(g)}$ to the device (iii) \textit{Local \underline{c}omputation period}  $\mathcal{T}_{i,j}^{\mathsf{C},(g)}$: The device performs local model training. (iv) \textit{\underline{U}plink transmission period} $\mathcal{T}_{i,j}^{\mathsf{U}, (g)}$: The device uploads its updated local model.
In the following, we explain the purpose and operation of each period.


\begin{figure*}
    \centering
    \includegraphics[width=0.8\linewidth]{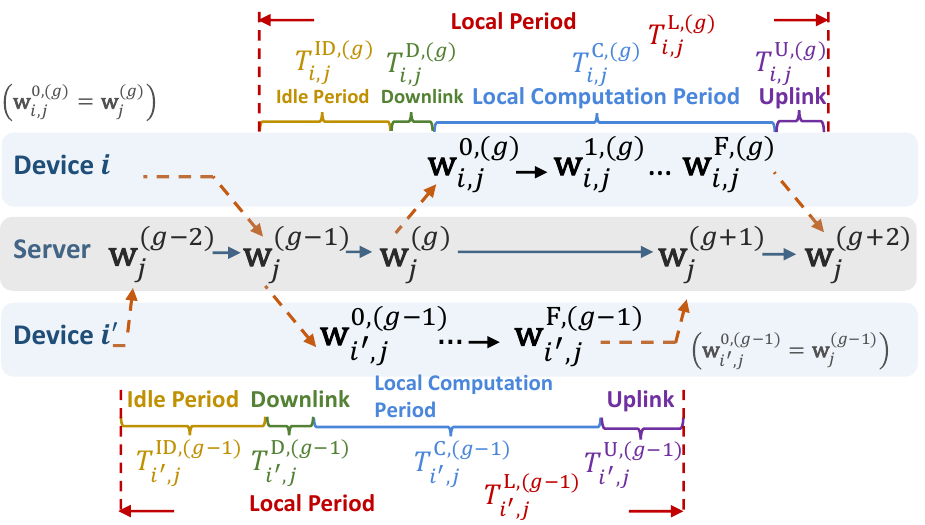}
\caption{Example timeline of local periods for task $j$ at two devices $i$ and $i'$. Since we consider asynchronous FL, the server updates the global model from $\bm{w}_j^{(g)}$ to $\bm{w}_j^{(g+1)}$ for any $g \in \mathcal{G}_j$ whenever it receives any trained local model. The definition of each local period encompasses the time span between two consecutive uplink transmissions. It comprises four distinct periods, namely the idle period, downlink transmission period, local computation period, and uplink transmission period.}
\label{fig:local_period_def}
\end{figure*}

\textbf{Idle Time and Downlink Transmission.}
To conserve energy and network resources, we consider that devices may wish to remain idle temporarily in-between local training. Formally, each device $i$, after sending its updated local model of task $j$ to the server, remains idle for a period of time  $T_{i,j}^{\mathsf{ID},(g)}=|\mathcal{T}_{i,j}^{\mathsf{ID},(g)}|$ before receiving the global model $\mathbf{w}_j^{(g)}$ from the server. Reception of $\mathbf{w}_j^{(g)}$ will then trigger local model training. The length of the idle period  will be optimized in {Sec.~\ref{sec:optimization_problem}}. 

\textbf{Local Model Training.} 
 During the computation period  $\mathcal{T}_{i,j}^{\mathsf{C}, (g)}$, device $i$ utilizes mini-batch stochastic gradient descent (SGD) for local ML model training. Due to heterogeneous computation capabilities, we consider that devices employ (i) different numbers of SGD iterations and (ii) different mini-batch sizes in their training processes. Formally, let $e_{i,j}^{(g)}$ denote the number of SGD iterations in  $\mathcal{T}_{i,j}^{\mathsf{C},(g)}$.
 At SGD iteration  $\ell \in \{1,\cdots, e_{i,j}^{(g)}\}$, device $i$ updates its model $\mathbf{w}_{i,j}^{\ell, (g)}$ for task $j$ as
\begin{equation}\label{eq:locIter}
\mathbf{w}^{\ell,(g)}_{i,j} = \mathbf{w}_{i,j}^{\ell-1,(g)} - \eta_j^{(g)} \nabla F_{i,j}^{\mathsf{R}}\big(\mathbf{w}_{i,j}^{\ell-1,(g)}, \mathcal{B}_{i,j}^{\ell,(g)}\big) 
\end{equation}
where $\eta_j^{(g)}$ is the learning rate for global aggregation $g$ and  $\mathcal{B}^{\ell,(g)}_{i,j}$ is the mini-batch dataset, containing ${B}^{\ell,(g)}_{i,j}=|\mathcal{B}^{\ell,(g)}_{i,j}|$ datapoints, and $F_{i,j}^{\mathsf{R}}$ is the \underline{r}egularized local loss function, defined as\footnote{Regularized loss functions have been suggested in existing works and achieved notable success upon having non-iid data across the devices \cite{xie2019asynchronous}.}
\begin{equation}\label{eq:regularized_local_loss_function}
 F_{i,j}^{\mathsf{R}}\big(\mathbf{w}_{i,j}^{\ell-1,(g)}, \mathcal{D}_{i,j}^{(g)}\big) = F_{i,j}\big(\mathbf{w}_{i,j}^{\ell-1,(g)}, \mathcal{D}_{i,j}^{(g)}\big) + \frac{\rho}{2} \left\lVert \mathbf{w}_{i,j}^{\ell-1,(g)} - \mathbf{w}_{i,j}^{0,(g)} \right\rVert^2
\end{equation}

\noindent where $\mathbf{w}_{i,j}^{0,(g)}$ is the initial model received from the server and $\rho\in\mathbb{R}^+$ is the regularization weight. 
We refer to the \underline{f}inal model obtained after the SGD iterations as
\begin{equation}
\mathbf{w}_{i,j}^{\mathsf{F},(g)} \triangleq \mathbf{w}_{i,j}^{l, (g)}\big|_{l = e_{i,j}^{(g)}}
\end{equation}

\textbf{Uplink Transfer and Global Aggregation.}  In the \textit{uplink period} $\mathcal{T}_{i,j}^{\mathsf{U}, (g)}$, device $i$ transmits its final local model  $\mathbf{w}_{i,j}^{\mathsf{F},(g)}$ to the server.  Due to the asynchronous nature of {\tt DMA-FL}, as depicted in Fig. \ref{fig:local_period_def}, the server may have received updates to the model for task $j$ during $\mathcal{T}_{i,j}^{\mathsf{L},(g)}$ from other devices $i' \neq i$. Thus, let $g' \in \mathcal{G}_j$ denote the updated global aggregation index for task $j$ at the server, with  $\mathbf{w}_j^{(g')}$ denoting the current global model parameter. The server aggregates model parameter of task $j$ as \cite{xie2019asynchronous} 
\begin{equation} \label{eq:global_aggregation_rule}
\mathbf{w}_{j}^{(g'+1)} = (1-\alpha_j)\mathbf{w}_{j}^{(g')} + \alpha_j \mathbf{w}_{i,j}^{\mathsf{F},(g)} ,
\end{equation}
where $\alpha_j\in (0,1)$ is the aggregation weight for task $j$. Weighted sum used in~\eqref{eq:global_aggregation_rule} is commonly seen in the asynchronous FL literature \cite{xie2019asynchronous,chen2019communication, chen2020asynchronous}. It weights the existing (global) information with the new (local) information for the task.

Fig. \ref{fig:local_period_def} and Fig. \ref{fig:system_architecture} illustrate these processes across multiple tasks in {\tt DMA-FL}. {\tt DMA-FL} implements a general multi-model training structure where at each time instance, each device is either idle or engaged in ML model training for one or more of its local tasks.

\begin{figure*}
    \centering
    \includegraphics[width=0.8\linewidth]{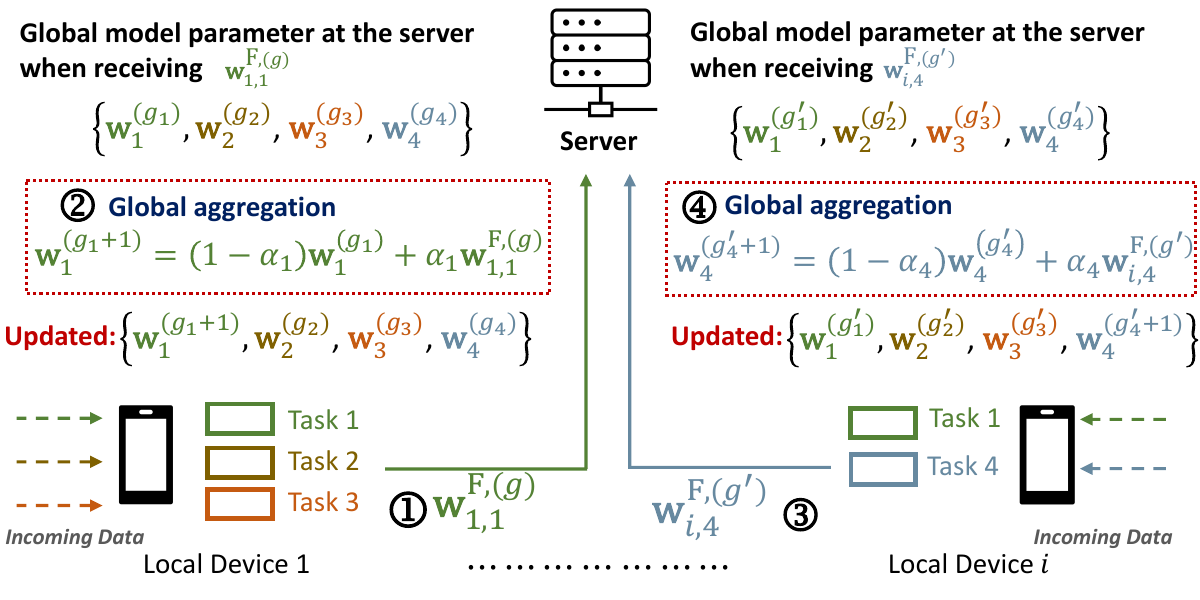}
\caption{Illustration of asynchronous local updates and the corresponding global aggregations at the server for multiple tasks in {\tt DMA-FL} when dynamic data variations are present. The server updates the global model for task $j$ with task dependent aggregation weight $\alpha_j$  only if it receives the trained local model for task $j$.  Models for other tasks at the server are not changed until its corresponding trained local model is transmitted to the server.}
\label{fig:system_architecture}
\end{figure*}

\section{Convergence Analysis}\label{sec:analysis}

\noindent We next conduct a convergence analysis for {\tt DMA-FL}. Given a known device scheduling and resource allocation, training of each task can be analyzed individually. As a result,
we carry out the analysis in Sec.~\ref{subsec:interpret} focusing on a single task $j\in\mathcal{J}$. We later optimize the device scheduling and resource allocation over multiple tasks in Sec. \ref{sec:optimization_problem}.

Our convergence analysis generalizes existing results in asynchronous FL  via (i) removing the strict assumptions on device participation, and (ii) characterizing the model training performance under an arbitrary device participation order. Our analysis will draw a connection between device scheduling and network resource allocation, which are incorporated in the convergence bounds.

\subsection{Assumptions and Definitions} \label{subsec:assumptions_definitions}
Our analysis employs the following assumptions, which are mainly extensions of existing ones in literature~\cite{xie2019asynchronous,lian2018asynchronous,zheng2017asynchronous, chen2020asynchronous,fallah2020personalized,9562522,9488906} to our setting.
\begin{assumption}[Boundedness]
The global loss function $F_j, \forall j$ is bounded below,
\begin{equation}
    \min_{\mathbf{w} \in \mathbb{R}^{M_j}} F_j(\mathbf{w}) > - \infty
\end{equation}
\end{assumption}
\begin{assumption}[Smooth Global and Local Loss Functions]
The local loss functions  $F_{i,j}, \; \forall i,j$ are $\beta_j$-smooth: 
\begin{equation}
   \norm{ \nabla F_{i,j}(\mathbf{w})\hspace{-.4mm} -\hspace{-.4mm} \nabla F_{i,j}(\mathbf{w}')} \hspace{-.4mm}\leq \hspace{-.4mm}\beta_j\hspace{-.4mm} \norm{\mathbf{w} - \mathbf{w}'}\hspace{-.4mm}, ~ \forall \mathbf{w}, \mathbf{w}' \in \mathbb{R}^{M_j} \hspace{-1.9mm} 
\end{equation}
\end{assumption}
\begin{assumption}[Data Variability]
     The local data variability at device $i$ for task $j$ is bounded by a finite constant  $\Theta_{i,j} \geq 0$, satisfying the following  $\forall d, d' \in\mathcal{D}_{i,j}^{(g)},$ $\forall \mathbf{w}\in\mathbb{R}^{M_j}$: 
     \begin{equation}
         \norm{\nabla L_{i,j}^{\mathsf{R}}(\mathbf{w}, d) - \nabla L_{i,j}^{\mathsf{R}}(\mathbf{w}, d')} \leq \Theta_{i,j} \norm{d - d'}.
     \end{equation}
     We further define  $\Theta_j = \max_{i\in\mathcal{I}}{\Theta_{i,j}}$. 
\end{assumption}
\begin{definition}[Weak Convexity]
A differentiable function $f(\bm{x})$ is $\rho$-weakly convex if, for $\rho \geq 0$, function $h(\bm{x}) = f(\bm{x})+\frac{\rho}{2}\twonormsquare{\bm{x}}$ is convex. $f(\bm{x})$ is convex if $\rho = 0$ and non-convex if $\rho > 0$.
\end{definition}

As stated in Definition 1, weak convexity only requires that the loss function becomes convex after the addition of a regularization term. As a result, our theoretical analysis covers a larger class of ML algorithms than the convex class; for example, two-layer neural networks \cite{richards2021learning}  and generative adversarial networks \cite{liu2021first} with smooth activation functions are weakly convex. Several works in the ML domain have also considered weak convexity as one of their main assumptions when conducting analysis~\cite{goujon2023learning, zhu2023distributionally}.

\begin{assumption}[Weak Convex Global Loss Function]
The global function  $F_j$, $\; \forall j \in \mathcal{J}$ is  $\rho$-weakly convex.
\end{assumption}
\begin{assumption}[Model Dissimilarity]
For any local model  $\mathbf{w}_{i,j}^{\ell, (g)}$, realized over global dataset $\mathcal{D}_j^{(g)}$ and  local dataset  $\mathcal{D}_{i,j}^{(g)}$  at device $i$ for task $j$, the difference between the gradient of the global loss  $F_j$ and the local loss $F_{i,j}$ is bounded by a constant  $\delta_{i,j}^{(g)}$ satisfying
\begin{equation}
\twonormsquare{\nabla F_j(\mathbf{w}_{i,j}^{\ell, (g)}, \mathcal{D}_j^{(g)}) - \nabla F_{i,j}(\mathbf{w}_{i,j}^{\ell, (g)}, \mathcal{D}_{i,j}^{(g)})} \leq \delta_{i,j}^{(g)}.
\end{equation}
\end{assumption}
The values of $\delta_{i,j}^{(g)}$ are indicative of the dataset heterogeneity (non-i.i.d.-ness) across devices. This will play an important role in tuning the number of local SGD updates across the devices and the device scheduling decisions (see Sec.~\ref{sec:optimization_problem}). Roughly speaking, more SGD iterations in-between global aggregations can be tolerated in devices that have small $\delta_{i,j}^{(g)}$ without risking local model bias.

\subsection{Data Evolution and  Device Scheduling} 

We introduce metrics to capture temporal data dynamics and device scheduling in our analysis.
\subsubsection{Data Evolution}
 We consider that device $i$'s dataset for task $j$ will evolve over training period as it collects more data. We let $\mathcal{D}_{i,j}(t)$ represent the local dataset at wall clock time $t \in \mathcal{T}_{i,j}^{L,(g)}$.
 We are interested in how the data evolution impacts current model performance, which we refer to as \textit{concept drift}. Intuitively, a device may collect data at a higher rate when it is not currently allocating resources for model training. We differentiate between these as \textit{idle} and \textit{active} concept drifts, where we expect the former to be larger in general.
 
\begin{definition}(Idle and Active Concept Drift) \label{def:concept_drift_active_concept_drift}
 We define the \textit{\underline{id}le concept drift}  $\Delta_{i,j}^{\mathsf{ID}}(t)$ in the unit of second for  $t \in \mathcal{T}_{i,j}^{\mathsf{ID},(g)} \cup  \mathcal{T}_{i,j}^{\mathsf{D},(g)} \cup \mathcal{T}_{i,j}^{\mathsf{U},(g)}$, i.e., during the idle and sync of training interval $g$, as the maximum potential variation of local loss performance. Formally, $\forall \mathbf{w}$, we have
    \vspace{-2mm}
    \begin{equation} \label{def:concept_drift}
      \hspace{-2mm}  \frac{1}{|\mathcal{I}|}\big( F_{i,j}\left(\mathbf{w}, \mathcal{D}_{i,j}(t)\right) -F_{i,j}\left(\mathbf{w}, \mathcal{D}_{i,j}(t-1)\right) \big) \leq \Delta_{i,j}^{\mathsf{ID}}(t).      \hspace{-2mm} 
    \end{equation}
    Similarly, we define the \textit{\underline{ac}tive concept drift}  $\Delta_{i,j}^{\mathsf{AC}}(t)$ for  $t \in \mathcal{T}_{i,j}^{\mathsf{C},(g)}$, i.e., during the computation interval,  as follows:
    \begin{equation} \label{def:active_concept_drift}
      \hspace{-2mm}  \frac{1}{|\mathcal{I}|}\big( F_{i,j}\left(\mathbf{w}, \mathcal{D}_{i,j}(t)\right) -F_{i,j}\left(\mathbf{w}, \mathcal{D}_{i,j}(t-1)\right) \big) \leq \Delta_{i,j}^{\mathsf{AC}}(t).      \hspace{-2mm} 
    \end{equation}
\end{definition}
A large value of $\Delta_{i,j}^{\mathsf{ID}}(t)$ or $\Delta_{i,j}^{\mathsf{AC}}(t)$ implies a large deviation in the local loss during the local period. Intuitively, to have a better global model, device-task pairs with higher concept drift should update the server more frequently to track their local data variations.

Let $\mathbf{w}_j^{(g'-1)}$ denote the global model parameter for task $j$ at the server
when device $i$ transmits the model parameter  $\mathbf{w}_{i,j}^{\mathsf{F},(g)}$ after training in local period  $\mathcal{T}_{i,j}^{\mathsf{L}, (g)}$ to obtain the global model $\mathbf{w}_j^{(g')}$. The combined impact of the idle and active concept drift between the local loss function where the model training started,  $F_{i,j}(\mathbf{w}, \mathcal{D}_{i,j}^{(g)})$, and the one where the local training concluded, $F_{i,j}(\mathbf{w}, \mathcal{D}_{i,j}^{(g')})$, is given by
\begin{equation} \label{def:concept_active_concept_two_local}
\begin{aligned}
    &\frac{1}{|\mathcal{I}|}\left( F_{i,j}(\mathbf{w}, \mathcal{D}_{i,j}^{(g')}) -F_{i,j}(\mathbf{w}, \mathcal{D}_{i,j}^{(g)}) \right) \leq \sum_{t \in  \mathcal{T}_{i,j}^{\mathsf{ID},(g)} \cup  \mathcal{T}_{i,j}^{\mathsf{D},(g)} \cup \mathcal{T}_{i,j}^{\mathsf{U},(g)} }\Delta_{i,j}^{\mathsf{ID}}(t)+ \sum_{t \in \mathcal{T}_{i,j}^{\mathsf{C},(g)}}\Delta_{i,j}^{\mathsf{AC}}(t) .
\end{aligned}
\end{equation}
    
\subsubsection{Device Scheduling}
We propose \textit{scheduling tensors} to capture asynchronous device scheduling, which are later optimized in Sec. \ref{sec:optimization_problem}.
\begin{definition}[Device Scheduling] \label{def:device_scheduling_matrices}

We define $R_{i,j}^{(g)} = 1$ if device $i$ \underline{r}eceives $\mathbf{w}_j^{(g)}$ for task $j$ (i.e., before $g$ is changed) at $g$-th global aggregation and $0$ otherwise. Similarly,  $U_{i,j}^{(g)} = 1$ if device $i$ \underline{u}ploads its local model parameter at $g$-th global aggregation and $0$ otherwise. We further define the device scheduling tensor  $\bm{X}=[X_{i,j}^{g,g'}]_{g,g'\in\mathcal{G}_j,i\in\mathcal{I},j\in\mathcal{J}}$, where
\begin{equation}\label{eq:device_scheduling_cases}
    \hspace{-1mm} X_{i,j}^{g,g'} = 
    \begin{cases}
        1 & \text{if device $i$ receives $\mathbf{w}_j^{(g)}$ and uploads $\mathbf{w}_{i,j}^{\mathsf{F},(g)}$ to complete aggregation $g'$} \\
        0 & \text{otherwise}.
    \end{cases}   
\end{equation}
\end{definition}

\begin{definition}[Staleness]
For task $j$, the staleness $K_j$ is the maximum number of global aggregations passed by any device without having reported its local model to the server, i.e., $K_j = \max_{i\in\mathcal{I}, g' \in \mathcal{G}} \mathcal{K}_{i,j}^{(g')}$, where
$\mathcal{K}_{i,j}^{(g')}= \{|g'-g|: X^{g,g'}_{i,j}=1 \}$ is the staleness for device $i$ in global aggregation $g'$.
\end{definition}
\begin{fact}
For $0 \leq g' - g \leq K_j$, each element of scheduling tensor $X_{i,j}^{g, g'}$ can be constructed using $R_{i,j}^{(g)}$ and $U_{i,j}^{(g)}$ as 
\begin{equation} \label{const:device_scheduling_matrices}
X_{i,j}^{g, g'} = U_{i,j}^{(g)}R_{i,j}^{(g')}\prod_{k = g+1}^{g'-1}(1-U_{i,j}^{(k)}).
\end{equation}
\end{fact}

Finally, our convergence analysis in Sec.~\ref{subsec:interpret} and optimization in Sec.~\ref{sec:optimization_problem} will require integrating the impact of concept drift from~\eqref{def:concept_active_concept_two_local} into device scheduling. To capture this, we introduce the following rectangular (rect) functions.

\begin{definition}[Capturing Concept Drift via Rect Functions] \label{def:rectangular_functions}
We define rect function $y_{i,j}^{\mathsf{ID}, (g)}(t) =  \mathds{1}_{\mathcal{T}_{i,j}^{\mathsf{ID},(g)} \cup  \mathcal{T}_{i,j}^{\mathsf{D},(g)} \cup \mathcal{T}_{i,j}^{\mathsf{U},(g)}}(t)$, which is $1$ when  $t\in \mathcal{T}_{i,j}^{\mathsf{ID},(g)} \cup  \mathcal{T}_{i,j}^{\mathsf{D},(g)} \cup \mathcal{T}_{i,j}^{\mathsf{U},(g)} $, i.e., during idle drift. Similarly, the rect function during active drift is $y_{i,j}^{\mathsf{AC}, (g)}(t) = \mathds{1}_{\mathcal{T}_{i,j}^{\mathsf{C}, (g)}}(t)$, which is $1$ when $t\in \mathcal{T}_{i,j}^{\mathsf{C}, (g)}$  .
\end{definition}

Fig.~\ref{fig:rectangular_functions} visualizes these definitions. Using them, we can transform the variable summation bounds in \eqref{def:concept_active_concept_two_local} to fixed bounds:
\begin{align}
    \sum_{t \in \mathcal{T}_{i,j}^{\mathsf{ID},(g)} \cup  \mathcal{T}_{i,j}^{\mathsf{D},(g)} \cup \mathcal{T}_{i,j}^{\mathsf{U},(g)} }\Delta_{i,j}^{\mathsf{ID}}(t) &=  \sum_{t = 0}^{T^{G_j}} y_{i,j}^{\mathsf{ID}, (g)}(t) \Delta_{i,j}^{\mathsf{ID}}(t) \\
    \sum_{t \in \mathcal{T}_{i,j}^{\mathsf{C},(g)}}\Delta_{i,j}^{\mathsf{AC}}(t) &= \sum_{t =0}^{T^{G_j}} y_{i,j}^{\mathsf{AC}, (g)}(t) \Delta_{i,j}^{\mathsf{AC}}(t)
\end{align}
where $T^{G_j}$ is the total time for all $G_j$ global aggregations  of task $j$.

\begin{figure}[!t]
    \centering
    \includegraphics[width=0.8\linewidth]{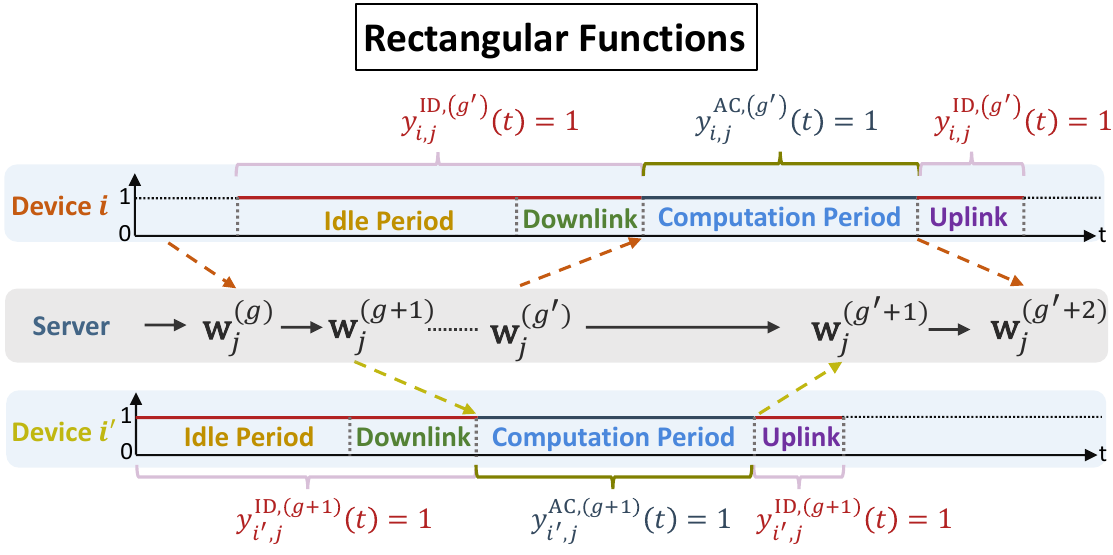}
    \caption{Illustration of how rectangular (rect) functions are used to capture the idle (ID) and active (AC) concept drift for different devices on a task. The aggregation index on the rect function corresponds to the index on the model received from the server.}
    \label{fig:rectangular_functions}
\end{figure}

\subsection{Model Performance Characterization}\label{subsec:interpret}
We are now ready to conduct our convergence analysis. We first find an expression on the difference between two global model parameters.

\begin{lemma}[Recursive Relationship between two Global Models] \label{lemma:recursive_diff_consecutive_global}
Let $\Psi_j^{g,g'}=\mathbf{w}_j^{(g)}-\mathbf{w}_j^{(g')}$ denote the difference between two global models obtained at  aggregations $g$ and $g'$, $g \leq g'$, for task $j$. Under an arbitrary device scheduling $\bm{X}$, we have
\begin{equation} \label{eq:recursive_consecutive_glboal} \hspace*{-7pt}
\Psi_j^{g,g'} = \alpha_j\sum_{k'= g}^{g'-1}\sum_{k = 0}^{k'}\left(\left(\sum_{i \in \mathcal{I}}X_{i,j}^{k, k'}\right)\Psi_j^{k,k'} - \sum_{i \in \mathcal{I}}X_{i,j}^{k, k'}\mathbf{a}_{i,j}^{(k)}  \right)\hspace*{-.7mm}, \hspace*{-6pt}
\end{equation}
where $\mathbf{a}_{i,j}^{(g)} = \eta_{j}^{(g)}\sum_{\ell = 0}^{e_{i,j}^{(g)}-1} \nabla F_{i,j}^{\mathsf{R}}\big(\mathbf{w}_{i,j}^{\ell, (g)}\big)$. If $g =  g'$, $\Psi_j^{g,g'} = \bm{0}$.
\end{lemma}
\begin{proof}
The result is obtained via expanding 
the global aggregation rule \eqref{eq:global_aggregation_rule} considering the device scheduling in \eqref{eq:device_scheduling_cases}. \qedhere
\end{proof}

Based on Lemma 1, we next present our main result, which characterizes the convergence of ML model training.
\begin{theorem}[Model Convergence]
Suppose that during model training $\twonormsquare{\nabla F_{i,j}(\mathbf{w},d)} \leq V_1$, $\lVert \nabla F_{i,j}^{\mathsf{R}}(\mathbf{w},d) \rVert^2 \leq V_2$, $ \forall d, \mathbf{w}, i, j$ for positive constants $V_1,V_2$. Let $\tilde{S}_{i,j}^{(g)}$ denote the variance of local dataset $\mathcal{D}_{i,j}^{(g)}$ and $e_j^{\mathsf{min}} \leq e_{i,j}^{(g)} \leq  e_j^{\mathsf{max}}, \; \forall i, j, g $ for two positive constants $e_j^{\mathsf{min}}$ and $ e_j^{\mathsf{max}}$. 
Also, let $\eta_j^{\mathsf{min}} = \min\{\eta_j^{(g)}\}_{g = 0}^{G_j-1}, \eta_j^{\mathsf{max}} = \max\{\eta_j^{(g)}\}_{g = 0}^{G_j-1}$. The cumulative average gradient of the global loss for task $j$, denoted $\mathsf{Conv}_j$, is bounded as (\ref{weighted_sum_weight_local_number_updates_concept_drift_integration_y}).
\end{theorem}
\begin{proof}
Please refer to Appendix A of our technical report \cite{chang2023asynchronous}.
\end{proof}

\begin{corollary}[Special Cases of (\ref{weighted_sum_weight_local_number_updates_concept_drift_integration_y})] \label{corollary:simplified_left_hand_side}
If whenever the server updates the global model, it activates \textit{at least} one device for local model training with the updated model, quantity $\mathsf{Conv}_j$ bounded in \eqref{weighted_sum_weight_local_number_updates_concept_drift_integration_y}, reduces to a conventional minimum norm of the global loss function gradient \cite{xie2019asynchronous}, i.e., $\min_{g= 0}^{G_j-1}\mathsf{E}[\lVert\nabla F_j(\mathbf{w}_j^{(g)})\rVert^2]$. If the server activates exactly one device after every update, $\mathsf{Conv}_j$ further simplifies to the conventional average  gradient norm \cite{hosseinalipour2023parallel}, i.e.,  $\sum_{g=0}^{G_j-1}\mathsf{E}[\lVert\nabla F_j(\mathbf{w}_j^{(g)})\rVert^2]/G_j$.
\end{corollary}
\begin{proof}
    Please refer to Appendix B of our technical report\cite{chang2023asynchronous}.
\end{proof}
There are several critical distinctions between our convergence analysis and those in existing asynchronous FL. Firstly, our research introduces Lemma 1, an integral component of our work that delineates the recursive relationship between two global models. For example, in \cite{xie2019asynchronous}, this effect would have been encapsulated by the term $\lVert x_\tau - x_{t-1} \rVert$, but was not considered in detail: the authors bounded this term by assuming that the trained local model aggregated with global model $x_\tau$ is based on $x_\tau$. Secondly, our analysis integrates network/device characteristics such as CPU frequencies, mini-batch size, the number of local SGD iteration and scheduling decisions (via the scheduling tensor $\bm{X}$) into the convergence bound. Thirdly and most importantly, we capture data dynamics modeled by concept drift through term (f) in \eqref{weighted_sum_weight_local_number_updates_concept_drift_integration_y}. As a result, the techniques used in the proof and the final bounds are majorly different from \cite{xie2019asynchronous}.

\begin{table*}[!t]
\centering
\begin{minipage}{\textwidth}
{
\begin{equation} \label{weighted_sum_weight_local_number_updates_concept_drift_integration_y}
    \begin{aligned}
    & \mathsf{Conv}_j \triangleq \frac{1}{G_j}\sum_{g' = 0}^{G_j-1}\sum_{g=0}^{g'}\sum_{i \in \mathcal{I}} X_{i,j}^{g, g'}\mathsf{E}\left[\twonormsquare{\nabla F_j(\mathbf{w}_j^{(g)})}\right] \leq \underbrace{\frac{2}{G_j\alpha_j\eta_j^{\mathsf{min}}}\mathsf{E} \left[ F_j(\mathbf{w}_j^{(0)}) - F_j(\mathbf{w}_j^{(G_j)}) \right]}_{(a)} \\
    &+\underbrace{\frac{1}{G_j\eta_j^{\mathsf{min}}}\sum_{g' = 0}^{G_j-1}\sum_{g=0}^{g'}\eta_{j}^{(g)}\sum_{i \in \mathcal{I}} X_{i,j}^{g, g'}e_{i,j}^{(g)}\delta_{i,j}^{(g)}}_{(b)} + \underbrace{\frac{4\beta_j}{G_j\eta_j^{\mathsf{min}}} \sum_{g' = 0}^{G_j-1}\sum_{g=0}^{g'}\left( \eta_{j}^{(g)} \right)^2\sum_{i \in \mathcal{I}} X_{i,j}^{g, g'}e_{i,j}^{(g)}\left(1-\frac{B_{i,j}^{(g)}}{D_{i,j}^{(g)}}\right)\frac{\left( D_{i,j}^{(g)}-1 \right) \Theta_j^2}{B_{i,j}^{(g)}D_{i,j}^{(g)}}\tilde{S}_{i,j}^{(g)}}_{(c)} \\
    &+ \underbrace{ \frac{2\rho}{G_j\eta_j^{\mathsf{min}}}\sum_{g' = 0}^{G_j-1}\sum_{g=0}^{g'}\left(\eta_j^{(g)}\right)^2\sum_{i \in \mathcal{I}} X_{i,j}^{g, g'} e_{i,j}^{(g)}V_2\left(\frac{\rho}{2}\eta_j^{(g)}(e_{i,j}^{(g)}-1) + e_{i,j}^{(g)}\right)}_{(d)} \\
    & +\underbrace{ \frac{1}{G_j\eta_j^{\mathsf{min}}}\sum_{g' = 0}^{G_j-1}\left(\sqrt{\frac{4 K_j\alpha_j  \left(e_j^{\mathsf{max}}\right)^2\left(\eta_j^{\mathsf{max}}\right)^2 V_1 V_2\left(\left(K_j\alpha_j\right)^{g'}- 1\right)}{K_j\alpha_j -1}} + \frac{ (\beta+2)K_j\alpha_j  \left(e_j^{\mathsf{max}}\right)^2\left(\eta_j^{\mathsf{max}}\right)^2 V_2\left(\left(K_j\alpha_j\right)^{g'}-1\right)}{K_j\alpha_j-1}\right)}_{(e)} \\[-.4em]
    & +\underbrace{ \frac{2}{G_j\eta_j^{\mathsf{min}}}\sum_{g' = 0}^{G_j-1}\sum_{g=0}^{g'}\sum_{i \in \mathcal{I}} X_{i,j}^{g, g'} \left(\sum_{t=0}^{T^{G_j}}y_{i,j}^{\mathsf{AC}, (g)}(t)\Delta_{i,j}^{\mathsf{AC}}(t) + \sum_{t=0}^{T^{G_j}}y_{i,j}^{\mathsf{ID}, (g)}(t)\Delta_{i,j}^{\mathsf{ID}}(t)\right)}_{(f)}  
    \end{aligned}
\end{equation}
}
\hrule
\end{minipage}
\end{table*}

\textbf{Interpretation of Results}: The bound in \eqref{weighted_sum_weight_local_number_updates_concept_drift_integration_y} captures the joint impact of ML hyperparameters, resource allocation, and device scheduling on the ML model performance. 

\begin{itemize}
\item Term $(a)$ resembles results found in conventional FL bounds when asynchrony is not considered.
\item Term $(b)$ captures the joint impact of the scheduling decisions $X_{i,j}^{g,g'}$, model dissimilarity  $\delta_{i,j}^{(g)}$, and number of SGD iterations   $e_{i,j}^{(g)}$. As the devices with larger model dissimilarity conduct more SGD updates, the bound increases significantly. These devices should be scheduled to upload their local model sooner than others since increasing the period of local model training will further bias their local models. More generally, the presence of variable $X_{i,j}^{g,g'}$ throughout each term in the bound shows how the scheduling across devices will affect model training performance in dynamic asynchronous FL.
\item Term $(c)$ captures the impact of SGD noise through the local data variability $\Theta_j$, sampling variance  $\tilde{S}_{i,j}^{(g)}$, and mini-batch size  $B_{i,j}^{\ell,(g)}$. Given a fixed sampling variance and local data variability, a larger mini-batch size leads to a smaller bound.
\item Building upon $(b)$, term $(d)$ shows another way that larger $e_{i,j}^{(g)}$ impacts the bound, in terms of accumulated gradient norms $V_2$ when devices are scheduled (observed from the product of  $X_{i,j}^{g,g'}$, $e_{i,j}^{(g)}$, and $V_2$). The bound is reduced by scheduling devices with a smaller gradient accumulation. 
\item Term $(e)$ introduces the impact of network-wide staleness $K_j$. As $K_j$ increases, the bound gets worse, which is caused by including local models derived based on outdated global models in the aggregation. Both the aggregation weighting coefficient $\alpha_j$ and the staleness $K_j$ are directly related to the asynchronous operation of the system. Their coupling here shows that a smaller weighting coefficient could alleviate the impact of staleness on the global model.
\item Term $(f)$ shows the impact of concept drift on the bound. Specifically, we see that the bound increases with larger idle and active drift values ($\Delta_{i,j}^{\mathsf{ID}}(t)y_{i,j}^{\mathsf{ID}, (g)}(t)$ and $\Delta_{i,j}^{\mathsf{AC}}(t)y_{i,j}^{\mathsf{AC}, (g)}(t)$). This term reveals that reducing idle time $\mathcal{T}_{i,j}^{\mathsf{ID}, (g)}$ to accelerate the local training period  for devices with higher drifts will lead to improvements in convergence. However, reducing the idle time leads to more active periods of device computation and upstream/downstream communication, which in turn results in more energy consumption. Therefore, devices having smaller concept drifts can stay in the idle period longer to save the energy without damaging ML performance. This tradeoff will be explicitly considered in our formulation in Sec. V, where our optimization problem balances the gain in ML model performance with the associated network cost (in terms of energy). 
\end{itemize}

\section{Optimization Methodology}\label{Sec:optimization}
To obtain a resource allocation strategy for {\tt DMA-FL}, we first model the communication and computation processes under heterogeneity (Sec.~\ref{sec:compcom}). Then, we formulate resource-aware {\tt DMA-FL} as an optimization problem (Sec.~\ref{sec:optimization_problem}).  Finally, we investigate the characteristics of the optimization problem and obtain its solution (Sec.~\ref{subsec:soluProb}). 

Overall, our methodology leverages the relationships from Sec. IV to configure device scheduling and resource allocation for asynchronous FL, in the presence of dynamic data variations and device heterogeneity, according to the objective of striking a balance between multi-task ML quality and energy consumption.

\subsection{Computation and Communication Modeling}\label{sec:compcom}

\subsubsection{Local Model Computation}
For each device $i\in\mathcal{I}$, let $a_{i,j}$ denote the number of CPU cycles needed to process one data sample of task $j\in\mathcal{J}$. Since some tasks may involve training deep neural networks with up to billions of model parameters (e.g., consider AlexNet  \cite{krizhevsky2012imagenet}), while others may involve in simpler/shallower models, in general $a_{i,j}\neq a_{i,j'}$. The local computation time of model $j$ at device $i$ upon conducting $e_{i,j}^{(g)}$ mini-batch SGD iterations with mini-batch sizes $B_{i,j}^{(g)}$  is
\begin{equation} \label{eq:local_computation_time}
    T_{i,j}^{\mathsf{C},(g)} = R_{i,j}^{(g)}a_{i,j} e_{i,j}^{(g)} B_{i,j}^{(g)} / f_{i,j}^{(g)},
\end{equation} 
where $f_{i,j}^{(g)}$ is the respective CPU frequency of the device.  The computation energy consumption of the device is modeled as
\begin{equation} \label{eq:local_computation_energy}
    E_{i,j}^{\mathsf{C},(g)} =  R_{i,j}^{(g)} \xi_i e_{i,j}^{(g)} a_{i,j} B_{i,j}^{(g)}\big(f_{i,j}^{(g)}\big)^2
\end{equation}
where $\xi_i$ is the effective chipset capacitance \cite{tran2019federated}.

\subsubsection{Model Transmission}
For each device $i \in \mathcal{I}$, let $h_i^{\mathsf{U},(g)}$ denote its channel gain to the BS at the time of global aggregation $g$. The data rate of the device to the server is given by
\begin{equation}\label{eq:rateUP}
    r_{i}^{\mathsf{U},(g)} = B_i^{\mathsf{U}}\log \left(1 + { |h_i^{\mathsf{U},(g)}|^2 p_i^{\mathsf{U}}}\big/({N_0 B_i^{\mathsf{U}}})\right),
\end{equation}
where $B_i^{\mathsf{U}}$ is the uplink bandwidth allocated to the device, $p_i^{\mathsf{U}}$ is the uplink transmit power of the device, and $N_0$ is the noise spectral density.
Letting $\sigma_{j}$ denote the number of bits required to represent one of the $M_{j}$ elements of model $j$, the delay and energy consumption  of transmitting the local model parameter of task $j$ from device $i$  to the BS are modeled
\begin{equation}
T_{i,j}^{\mathsf{U},(g)} = U_{i,j}^{(g)}\sigma_{j} M_{j}/r_{i}^{\mathsf{U},(g)} ,~
E_{i,j}^{\mathsf{U},(g)} = p_i^{\mathsf{U}}T_{i,j}^{\mathsf{U},(g)}, \label{eq:uplink_transmission}
\end{equation}
respectively. Similarly, the downlink data rate from the BS to device $i$ at each  $g \in \{1, \cdots, G\}$ is given by
\begin{equation}\label{eq:rateDOWN}
    r_{i}^{\mathsf{D},(g)} = B_i^{\mathsf{D}}\log \left(1 + {|h_i^{\mathsf{D},(g)}|^2 p_i^{\mathsf{D}}}\big/({N_0 B_i^{\mathsf{D}}})\right),
\end{equation}
where $B_i^{\mathsf{D}}$ is downlink bandwidth, $h_i^{\mathsf{D},(g)}$ is the downlink channel gain, and $p^{\mathsf{D}}$ is the transmit power of the BS. Subsequently, the downlink delay and energy consumption of transmitting the model parameters of task $j$ to device $i$ at global aggregation $g$ are
\begin{equation}
T_{i,j}^{\mathsf{D},(g)} = R_{i,j}^{(g)}\sigma_{j} M_{j}/r_{i}^{\mathsf{D},(g)} ,~
E_{i,j}^{\mathsf{D},(g)} = p_i^{\mathsf{D}}T_{i,j}^{\mathsf{D},(g)}, \label{eq:downlink_transmission}
\end{equation}
respectively.
Consequently, the length of the local period is  
\begin{equation}\label{const:local_period_time}
T_{i,j}^{\mathsf{L}, (g)} = R_{i,j}^{(g)} T_{i,j}^{\mathsf{ID}, (g)} + T_{i,j}^{\mathsf{C},(g)}  + T_{i,j}^{\mathsf{U},(g)} + T_{i,j}^{\mathsf{D},(g)}. 
\end{equation}

\subsection{Optimization Problem}\label{sec:optimization_problem}
\subsubsection{Problem Formulation}
Before formulating our optimization problem, we need to first define the auxiliary functions
\begin{alignat}{2}
    &Q_{i,j}^{\mathsf{ID},(g)}(t) =  &&\Big(t-\ssum_{j \in \mathcal{J}}\ssum_{k=0}^{g-1}T_{i,j}^{\mathsf{L},(k)} \Big) \times \Big(t-\big(\ssum_{j \in \mathcal{J}}\ssum_{k=0}^{g}T_{i,j}^{\mathsf{L},(k)} + R_{i,j}^{(g)}T_{i,j}^{\mathsf{ID},(g)} \big)\Big) \label{const:auxiliary_idle} \\
    &Q_{i,j}^{\mathsf{AC},(g)}(t) =  &&\Big(t-\big(\ssum_{j \in \mathcal{J}}\ssum_{k=0}^{g-1}T_{i,j}^{\mathsf{L},(k)} + R_{i,j}^{(g)}T_{i,j}^{\mathsf{ID},(g)} + T_{i,j}^{\mathsf{D},(g)} \big)\Big) \Big(t-\big(\ssum_{j \in \mathcal{J}}\ssum_{k=0}^{g-1}T_{i,j}^{\mathsf{L},(k)} + R_{i,j}^{(g)}T_{i,j}^{\mathsf{ID},(g)}\nonumber \\
    & &&+ T_{i,j}^{\mathsf{D},(g)} + T_{i,j}^{\mathsf{C},(g)} \big)\Big) \label{const:auxiliary_active} 
\end{alignat}
These auxiliary functions are created to capture the start and the end of both active and idle concept drift. We formulate the joint device scheduling and resource allocation for {\tt DMA-FL} as the following optimization problem $\bm{\mathcal{P}}$, transforming the scheduling decisions to optimization constraints:
\begin{mini!}[2]
{}{c_1\ssum_{j \in \mathcal{J}}\gamma_j\mathsf{Conv}_j +\ssum_{i \in \mathcal{I}}\ssum_{j \in \mathcal{J}} \ssum_{g=0}^{G_j-1}\big[\chi_j \ssum_{j \in \mathcal{J}}\big(c_2(E_{i,j}^{\mathsf{U},(g)}+ E_{i,j}^{\mathsf{C},(g)})+ c_3 E_{i,j}^{\mathsf{D},(g)}\big) \big]\big/G_j}
{\label{eq:Example1}}
{(\bm{\mathcal{P}}):}
\addConstraint{\eqref{const:device_scheduling_matrices} ,\eqref{eq:local_computation_time}, \eqref{eq:local_computation_energy},
\eqref{eq:uplink_transmission},
\eqref{eq:downlink_transmission},
\eqref{const:local_period_time} \nonumber}{}
\addConstraint{\ssum_{g=0}^{G_j-1}T_{i,j}^{\mathsf{L}, (g)} + T_{i,j}^{\mathsf{ID}, \mathsf{F}}}{= T_{i,j}^{\mathsf{QoE}} \quad  \label{const:QoE}}{i \in \mathcal{I}, j \in \mathcal{J}}
\addConstraint{\ssum_{j \in \mathcal{J}} \ssum_{g=0}^{G_j-1}\Big(E_{i,j}^{\mathsf{U},(g)}+E_{i,j}^{\mathsf{C},(g)}\Big)}{\leq E^{\mathsf{B}}_i \quad \label{const:energy_budget}}{i\in\mathcal{I}}
\addConstraint{G_j \leq \ssum_{g= 0}^{G_j-1}\ssum_{i \in \mathcal{I}}R_{i,j}^{(g)}}{\leq G_j + K_j \quad \label{const:downlink_staleness}}{j \in \mathcal{J}}
\addConstraint{\ssum_{i \in \mathcal{I}} U_{i,j}^{(g)}}{= 1 \quad \label{const:uplink_transfer_only_one}}{j \in \mathcal{J}, g \in \mathcal{G}_j\setminus\{G_j\}}
\addConstraint{\ssum_{i \in \mathcal{I}}\ssum_{g=0}^{G_j -1} U_{i,j}^{(g)} \quad \label{const:uplink_equal_global_aggregation_num}}{= G_j}{j \in \mathcal{J}}
\addConstraint{(1- R_{i,j}^{(g)})T_{i,j}^{\mathsf{ID},(g)}}{ = 0 \quad \label{const:receive_idle}}{i \in \mathcal{I}, j \in \mathcal{J}, g \in \mathcal{G}_j}
\addConstraint{\ssum_{i \in \mathcal{I}}(U_{i,j}^{(g)}\ssum_{k = 0}^{g-1}T_{i,j}^{\mathsf{L}, (k)}) }{\leq \ssum_{i \in \mathcal{I}}(U_{i,j}^{(g+1)}\ssum_{k = 0}^{g}T_{i,j}^{\mathsf{L}, (k)}) \quad \label{const:make_sure_order}}{ j \in \mathcal{J}}
\addConstraint{Q_{i,j}^{\mathsf{ID},(g)}(t)y_{i,j}^{\mathsf{ID}, (g)}(t) }{\leq 0 \quad \label{const:idle_zero}}{i \in \mathcal{I}, j \in \mathcal{J}, g \in \mathcal{G}_j}
\addConstraint{Q_{i,j}^{\mathsf{ID},(g)}(t)(y_{i,j}^{\mathsf{ID}, (g)}(t) - 1) }{\leq 0 \quad \label{const:idle_one}}{i \in \mathcal{I}, j \in \mathcal{J}, g \in \mathcal{G}_j}
\addConstraint{Q_{i,j}^{\mathsf{AC},(g)}(t)y_{i,j}^{\mathsf{AC}, (g)}(t) }{\leq 0 \quad \label{const:active_zero}}{i \in \mathcal{I}, j \in \mathcal{J}, g \in \mathcal{G}_j}
\addConstraint{Q_{i,j}^{\mathsf{AC},(g)}(t)(y_{i,j}^{\mathsf{AC}, (g)}(t) - 1)}{\leq 0 \quad \label{const:active_one}}{i \in \mathcal{I}, j \in \mathcal{J}, g \in \mathcal{G}_j}
\addConstraint{f_{i}^{\mathsf{min}}\leq \textstyle\ssum_{j \in \mathcal{J}}f_{i,j}^{(g)}}{\leq f_{i}^{\mathsf{max}} \quad \label{const:CPU_lower_upper}}{i \in \mathcal{I}, g \in \mathcal{G}_j}
\addConstraint{e_j^{\mathsf{min}} \leq e_{i,j}^{(g)}}{\leq e_j^{\mathsf{max}} \quad \label{const:num_local_SGD_lower}}{i \in \mathcal{I}, j \in \mathcal{J}, g \in \mathcal{G}_j}
\addConstraint{1\leq B_{i,j}^{(g)}}{\leq D_{i,j}^{(g)}\quad \label{const:sample_size_lower_upper}}{i \in \mathcal{I}, j \in \mathcal{J}, g \in \mathcal{G}_j}
\addConstraint{T_{i,j}^{\mathsf{ID}, (g)}, T_{i,j}^{\mathsf{U}, (g)}, T_{i,j}^{\mathsf{D}, (g)}, T_{i,j}^{\mathsf{C}, (g)} }{\geq 0 \quad \label{const:idle_range}}{i \in \mathcal{I}, j \in \mathcal{J}, g \in \mathcal{G}_j}
\addConstraint{T_{i,j}^{\mathsf{ID}, \mathsf{F}}}{\geq 0 \quad \label{const:final_idle_range}}{i \in \mathcal{I}, j \in \mathcal{J}}
\addConstraint{R_{i,j}^{(g)}, U_{i,j}^{(g)}}{\in \{0,1\} \quad \label{const:downlink_transfer_range}}{i \in \mathcal{I}, j \in \mathcal{J}, g \in \mathcal{G}_j}
\end{mini!}
\begin{align*}
    \textrm{\textbf{variables}:} \quad &e_j^{\mathsf{max}}, e_j^{\mathsf{min}} \{\textrm{$\bm{f}^{(g)}, \bm{B}^{(g)}, \bm{e}^{(g)} , \bm{R}^{(g)}, \bm{U}^{(g)}, \bm{T}^{\mathsf{ID}, (g)}, \bm{T}^{\mathsf{L}, (g)},\bm{T}^{\mathsf{U}, (g)}, \bm{T}^{\mathsf{D}, (g)}, \bm{T}^{\mathsf{C}, (g)}, \bm{T}^{\mathsf{ID}, \mathsf{F}}$} \\
    &\textrm{$\bm{Q}^{\mathsf{ID}, (g)}, \bm{Q}^{\mathsf{AC}, (g)}\}$} 
\end{align*}

\noindent Problem $\bm{\mathcal{P}}$  optimizes the trade-off between multi-task ML quality (i.e., the first term in the objective) and  energy consumption (the second term). Constants $c_1,c_2,c_3\geq 0$ in the objective  weigh the importance of model performance, local energy consumption at devices, and the energy consumption at the BS (if BS is not a concern, $c_3=0$). The problem aims to find the optimal resource allocation  across global aggregations (i.e., CPU frequency 
$\bm{f}^{(g)}=\{f_{i,j}^{(g)}\}_{i\in\mathcal{I},j\in\mathcal{J}}$, mini-batch size  $\bm{B}^{(g)}=\{B_{i,j}^{(g)}\}_{i\in\mathcal{I},j\in\mathcal{J}}$, and number of SGD iterations $\bm{e}^{(g)} = \{e_{i,j}^{(g)}\}_{i\in\mathcal{I},j\in\mathcal{J}}$), the scheduling of the devices (i.e., the downlink and uplink transmission indicators $\bm{R}^{(g)} = \{R_{i,j}^{(g)}\}_{i\in\mathcal{I},j\in\mathcal{J}}$ and  $\bm{U}^{(g)}=\{U_{i,j}^{(g)}\}_{i\in\mathcal{I},j\in\mathcal{J}}$)
, the idle time and local period (i.e.  $\bm{T}^{\mathsf{ID}, (g)} = \{\mathcal{T}_{i,j}^{\mathsf{ID}, (g)}\}_{i \in \mathcal{I}, j \in \mathcal{J}}$ and $\bm{T}^{\mathsf{L}, (g)} = \{\mathcal{T}_{i,j}^{\mathsf{L}, (g)}\}_{i \in \mathcal{I}, j \in \mathcal{J}}$) and the rectangular functions capturing the idle/active concept drift as well as associated auxiliary functions (i.e. $\bm{y}^{\mathsf{AC}, (g)} =\{y_{i,j}^{\mathsf{AC}, (g)}(t)\}_{i \in \mathcal{I}, j \in \mathcal{J}}$, $\bm{y}^{\mathsf{ID}, (g)} = \{y_{i,j}^{\mathsf{ID}, (g)}(t)\}_{i \in \mathcal{I}, j \in \mathcal{J}}$, $\bm{Q}^{\mathsf{ID}, (g)} = \{Q_{i,j}^{\mathsf{ID}, (g)}\}_{i \in \mathcal{I}, j \in \mathcal{J}}$, and $\bm{Q}^{\mathsf{AC}, (g)} = \{Q_{i,j}^{\mathsf{AC}, (g)}\}_{i \in \mathcal{I}, j \in \mathcal{J}}$). Also, $\gamma_j$, in the first term of the objective function is the assigned \textit{weight/importance to model} $j$'s performance, enabling prioritization of different models (e.g., some models may be used for more important applications). Similarly, $\chi_j \geq 0$ is the assigned weight to the energy consumption of task $j$.
  
The problem captures a quality of experience (QoE) constraint in \eqref{const:QoE} by restricting the time window for the execution of each model $j$ to $T_j^{\mathsf{QoE}}$. It also considers an energy \underline{b}udget in \eqref{const:energy_budget} via $ E^{\mathsf{B}}_i$, $i\in\mathcal{I}$. Constraints \eqref{const:downlink_staleness}-\eqref{const:make_sure_order} ensure correct device scheduling, guaranteeing sequential reception of the models at the server, and the correct sequence of uplink and downlink transmission in device-BS communications. Constraints \eqref{const:idle_zero}-\eqref{const:active_one} ensure a correct \textit{policy} for the rect functions. 
In particular, variables $Q_{i,j}^{\mathsf{AC}, (g)}$ and  $Q_{i,j}^{\mathsf{ID}, (g)}$ are added to capture the start and the end of the active and idle concept drift in \eqref{const:auxiliary_idle} and \eqref{const:auxiliary_active}. In \eqref{const:idle_zero} and \eqref{const:idle_one}, we ensure $y_{i,j}^{\mathsf{ID}, (g)}(t)$ takes the value of $1$ when the devices are in the idle, uplink transmission, and downlink transmission period; and $0$  otherwise. Similarly, in \eqref{const:active_zero} and \eqref{const:active_one}, we guarantee  $y_{i,j}^{\mathsf{AC}, (g)}(t)$ is $1$ when devices are performing local model training; and $0$ otherwise. Constraints \eqref{const:CPU_lower_upper}-\eqref{const:downlink_transfer_range} ensure the feasibility of the problem. Note that we have $i \in \mathcal{I}, j \in \mathcal{J}$, $R_{i,j}^{(0)}T_{i,j}^{\mathsf{ID},(0)} = 0$.

\subsubsection{Nuances and Behavior of the Solution of $\bm{\mathcal{P}}$} We point out several important properties in problem $\bm{\mathcal{P}}$'s solution behavior: 
\begin{itemize}
    \item Upon increasing $\gamma_j$ for task $j$, the solution will allocate more communication and computation resources (i.e., higher CPU speed $f_{i,j}^{(g)}$, larger mini-batch size $B_{i,j}^{(g)}$, and more careful tuning of SGD iterations  $e_{i,j}^{(g)}$) across devices $i\in\mathcal{I}$ for training this task. Also, the scheduling variables will favor more frequent reception of model parameters of task $j$ at the server for global aggregations as compared to other tasks.
    \item Fixing $c_1, c_3$, upon increasing $c_2$, the solution would favor lower power consumption at the devices over a better model accuracy, which will reflect in the resource allocation (i.e., lower CPU speed, smaller mini-batch sizes, and fewer SGD iterations) and in device scheduling (i.e., less frequent uplink transmissions. 
    \item Fixing $c_1, c_2$, upon increasing $c_3$ the solution will favor more frequent model training at those devices with closer proximity to the BS.
    \item A small value of  $T_j^{\mathsf{QoE}}$ for task $j$ (i.e., a shorter time window of execution) implies that the training of task $j$ should be conducted faster  compared to other tasks via allocating more computation/communication resources and scheduling the devices to have more frequent updates of this task $j$.
    \item From $\mathsf{Conv}_j$ in the objective, the solution will incorporate the behaviors from our convergence analysis mentioned in~Sec.~\ref{subsec:interpret} to have more efficient model training.
    \item The optimization schedules the devices differently in the cold vs. warm model regimes. In particular,   inspecting~\eqref{weighted_sum_weight_local_number_updates_concept_drift_integration_y}, the solution favors allocation of resources to model $j$ with a larger initial error  $F_j(\mathbf{w}_j^{(0)})$  to compensate for a high loss. However, for warm models (i.e., those with lower initial errors) the solution favors the allocation of fewer network resources. 
\end{itemize}
\subsection{Solution Design} \label{subsec:soluProb}
Based on the behavior of~\eqref{weighted_sum_weight_local_number_updates_concept_drift_integration_y} and the constraints of $\bm{\mathcal{P}}$, we conclude that $\bm{\mathcal{P}}$ is a mixed-integer non-convex optimization. To overcome this, we first relax the integer variables and then solve the problem through successive convex approximations. Our choice of successive convex approximation was inspired by several works in the past decade which have established its theoretical guarantees \cite{liu2019stochastic, scutari2016parallel}, as well as its popularity for handling non-convex problems that arise in the wireless communications domain \cite{kaleva2012weighted, tian2022successive}. We provide the pseudo-code of the proposed {\tt DMA-FL} and successive convex approximation in Algorithm \ref{alg:sca} for clarity.

\begin{algorithm}[t]
    \label{alg:sca}
    \caption{Proposed {\tt DMA-FL} with Successive Convex Approximation}
    \SetKwInput{Input}{Input}
    \SetKwInput{Output}{Output}
    \Input{Optimization problem $\bm{\mathcal{P}}$, and step size $\varepsilon$}
    \Output{The converged solution $\bm{v}$}
    Transform $\bm{\mathcal{P}}$ via \eqref{const:force_downlink_transfer}, \eqref{const:force_uplink_transfer}, \eqref{const:force_ma_concept}, and \eqref{const:force_mi_concept} into $\widehat{\bm{\mathcal{P}}}$ shown in \eqref{problem:p_hat} \;
    Initialize a feasible point $\bm{v}_0$ of the problem $\widehat{\bm{\mathcal{P}}}$ \;
    Initial the iteration number $m = 0$\;
    \While{$\bm{v}_m$ has not converged}{
        Employ the proximal gradient method outlined in \eqref{convexified_nonconvex_objective}, (\ref{eq:convexified_nonconvex_inequality_constraint}), and \eqref{eq:convexified_nonconvex_equality_constraint} to convexify $\widehat{\bm{\mathcal{P}}}$ at the current solution $\bm{v}_{m}$ to obtain a surrogate problem $\widehat{\bm{\mathcal{P}}}^{(m)}$ expressed in \eqref{problem:p_hat_m}\;
        Solve the surrogate problem $\widehat{\bm{\mathcal{P}}}^{(m)}$ via convex optimization techniques to get $\widehat{\bm{v}}_{c}(\bm{v}_{m})$ \;
        Update the solution according to \eqref{sca_iterate_rule}: $\bm{v}_{m+1} = \bm{v}_{m} + \varepsilon (\widehat{\bm{v}}_{c}(\bm{v}_{m}) - \bm{v}_{m})$ \;
        $m \leftarrow m+1$ \;
    }
    \Return{The converged solution $\bm{v}$}

\end{algorithm}
\subsubsection{Transforming Integer Variables} Suppose we relax all the integer variables (i.e. $R_{i,j}^{(g)}$, $U_{i,j}^{(g)}$, $X_{i,j}^{g, g_1}$, $y_{i,j}^{\mathsf{ID},(g)}(t)$, and $y_{i,j}^{\mathsf{AC},(g)}(t)$) to continuous variables within range $[0,1]$. Then, we can force them to take binary values by incorporating the following constraints:
\begin{alignat}{2}
R_{i,j}^{(g)}(1-R_{i,j}^{(g)}) &\leq 0, \quad &&i \in \mathcal{I}, j \in \mathcal{J}, g \in \mathcal{G}_j \label{const:force_downlink_transfer} \\
 U_{i,j}^{(g)}(1-U_{i,j}^{(g)}) &\leq 0,\quad &&i \in \mathcal{I}, j \in \mathcal{J}, g \in \mathcal{G}_j  \label{const:force_uplink_transfer} \\
y_{i,j}^{\mathsf{ID},(g)}(t)\big(1 - y_{i,j}^{\mathsf{ID},(g)}(t)\big) &\leq 0,\quad &&i \in \mathcal{I}, j \in \mathcal{J}, g \in \mathcal{G}_j \label{const:force_ma_concept} \\
y_{i,j}^{\mathsf{AC},(g)}(t)\big(1 - y_{i,j}^{\mathsf{AC},(g)}(t)\big) &\leq 0,\quad &&i \in \mathcal{I}, j \in \mathcal{J}, g \in \mathcal{G}_j \label{const:force_mi_concept} 
\end{alignat}
The above constraints along with \eqref{const:uplink_transfer_only_one}, and \eqref{const:idle_zero}-\eqref{const:active_one} ensure that the indicated continuous variables take binary values in the feasible region. 
Those variables guarantee that only one device would upload/receive one of the task's model parameter to/from the server at any global aggregation and the rectangular functions takes correct values. Using \eqref{const:device_scheduling_matrices} along with the above two results, scheduling variable $X_{i,j}^{g, g_1}$ in turn takes binary values. 

\subsubsection{Decomposition of {$\bm{\mathcal{P}}$} into Convex and Non-convex Parts}  We denote the objective function of $\bm{\mathcal{P}}$ as $\bm{\mathcal{O}}(\bm{v})$, where $\bm{v}$ encapsulates all variables of $\bm{\mathcal{P}}$. The constraints of $\bm{\mathcal{P}}$ can be divided into four vectors: {c}onvex \underline{eq}ualities  $\mathbf{C}^{\mathsf{EQ}}(\bm{v})$ (i.e., \eqref{eq:uplink_transmission}, \eqref{eq:downlink_transmission}, \eqref{const:QoE}, \eqref{const:uplink_transfer_only_one} and \eqref{const:uplink_equal_global_aggregation_num}), {c}onvex \underline{i}n\underline{e}qualities   $\mathbf{C}^{\mathsf{IE}}(\bm{v})$ (i.e., \eqref{const:energy_budget}, \eqref{const:downlink_staleness}, and \eqref{const:CPU_lower_upper}-\eqref{const:final_idle_range}), {n}onconvex \underline{eq}ualities   $\mathbf{N}^{\mathsf{EQ}}(\bm{v})$ (i.e., \eqref{const:device_scheduling_matrices}, \eqref{eq:local_computation_time}, \eqref{eq:local_computation_energy}, \eqref{const:local_period_time}, and \eqref{const:receive_idle}), and {n}onconvex \underline{i}n\underline{e}qualities   $\mathbf{N}^{\mathsf{IE}}(\bm{v})$ (i.e., \eqref{const:make_sure_order}-\eqref{const:active_one}, and \eqref{const:force_downlink_transfer}- \eqref{const:force_mi_concept}). Thus, $\bm{\mathcal{P}}$ can be written as $\widehat{\bm{\mathcal{P}}}$ below
\begin{align} \label{problem:p_hat}
&(\widehat{\bm{\mathcal{P}}}): \min\limits_{\bm{v}}~~~\bm{\mathcal{O}}(\bm{v}) \\
& \textrm{s.t.}~~\mathbf{C}^{\mathsf{EQ}}(\bm{v})=\bm{0}, \mathbf{N}^{\mathsf{EQ}}(\bm{v}) =\bm{0}, \mathbf{C}^{\mathsf{IE}}(\bm{v})\leq \bm{0}, \mathbf{N}^{\mathsf{IE}}(\bm{v}) \leq \bm{0}. \nonumber
\end{align}


We next present our successive convex methodology, which is inspired by the method in \cite{scutari2016parallel}. 
\subsubsection{Successive Convex Approximation} We solve $\widehat{\bm{\mathcal{P}}}$ through a sequence of approximations indexed by $m$. 
Let $\bm{v}_{0}$ denote the initial point/solution of the method that satisfies constraints of $\widehat{\bm{\mathcal{P}}}$. At each iteration $m$, we convexify $\widehat{\bm{\mathcal{P}}}$ at the current solution $\bm{v}_{m}$ to obtain a surrogate problem $\widehat{\bm{\mathcal{P}}}^{(m)}$. Denoting the solution to $\widehat{\bm{\mathcal{P}}}^{(m)}$ as $\widehat{\bm{v}}_{c}(\bm{v}_{m})$, we update the variables as between iterations as
\begin{equation} \label{sca_iterate_rule}
  \bm{v}_{m+1} = \bm{v}_{m} + \varepsilon (\widehat{\bm{v}}_{c}(\bm{v}_{m}) - \bm{v}_{m}) .
\end{equation}
We leverage a proximal gradient method \cite{parikh2014proximal} to convexify $\widehat{\bm{\mathcal{P}}}$. To this end,  we relax and convexify the objective function  $\bm{\mathcal{O}}(\bm{v})$ and the non-convex constraint vectors (i.e.,  $\mathbf{N}^{\mathsf{EQ}}(\bm{v})$ and  $\mathbf{N}^{\mathsf{IE}}(\bm{v})$)  such that the relaxed constraints upper-bound the original ones. Specifically, at iteration $m$, given the current solution $\bm{v}_{m}$, we obtain the convex approximation of the objective  $\bm{\mathcal{O}}$ for $\lambda > 0$, denoted by  $\widehat{\bm{\mathcal{O}}}$, as
\begin{equation} \label{convexified_nonconvex_objective}
\widehat{\bm{\mathcal{O}}}(\bm{v}; \bm{v}_{m}) = \bm{\mathcal{O}}(\bm{v}_{m}) + \nabla\bm{\mathcal{O}}(\bm{v}_{m})^\top (\bm{v} - \bm{v}_{m})+ \frac{\lambda}{2}\twonormsquare{\bm{v} - \bm{v}_{m}}
\end{equation}
and the convex approximations of non-convex constraints as
\begin{alignat}{2}
& \widehat{\mathbf{N}}^{\mathsf{IE}}(\bm{v}; \bm{v}_{m}) &&= \mathbf{N}^{\mathsf{IE}}(\bm{v}_{m}) + \nabla\mathbf{N}^{\mathsf{IE}}(\bm{v}_{m})^\top (\bm{v} - \bm{v}_{m}) + \frac{L_{\mathsf{IE}}}{2}\twonormsquare{\bm{v} - \bm{v}_{m}} ,
 \label{eq:convexified_nonconvex_inequality_constraint} \\
& \widehat{\mathbf{N}}^{\mathsf{EQ}}(\bm{v}; \bm{v}_{m}) &&= \mathbf{N}^{\mathsf{EQ}}(\bm{v}_{m}) + \nabla\mathbf{N}^{\mathsf{EQ}}(\bm{v}_{m})^\top (\bm{v} - \bm{v}_{m})+ \frac{L_{\mathsf{EQ}}}{2}\twonormsquare{\bm{v} - \bm{v}_{m}}. \label{eq:convexified_nonconvex_equality_constraint}
\end{alignat}
In \eqref{eq:convexified_nonconvex_inequality_constraint} and \eqref{eq:convexified_nonconvex_equality_constraint}, $L_{\mathsf{IE}}$ and $L_{\mathsf{EQ}}$ are the Lipschitz constants that are characteristic of $\mathbf{N}^{\mathsf{IE}}$ and $\mathbf{N}^{\mathsf{EQ}}$ respectively. The above formulation implies that $\widehat{\mathbf{N}}^{\mathsf{IE}}(\bm{v}; \bm{v}_m) \geq \mathbf{N}^{\mathsf{IE}}(\bm{v})$ and $\widehat{\mathbf{N}}^{\mathsf{EQ}}(\bm{v}; \bm{v}_m) \geq \mathbf{N}^{\mathsf{EQ}}(\bm{v})$ \cite{bertsekas2015parallel}. The proximal-based relaxation in \eqref{convexified_nonconvex_objective} also ensures the strong convexity of the surrogate function $\widehat{\bm{\mathcal{O}}}(\bm{v}; \bm{v}_m)$. At each iteration $m$, we arrive at the following relaxed convex approximation of $\widehat{\bm{\mathcal{P}}}$:
\begin{align} \label{problem:p_hat_m}
(\widehat{\bm{\mathcal{P}}}^{(m)})\hspace{-.5mm}:&\max_{\bm{\Lambda},\bm{\Omega}\geq 0}\min_{\bm{v}}\widehat{\bm{\mathcal{O}}}(\bm{v}; \bm{v}_{m}) + \bm{\Lambda}^\top \widehat{\mathbf{N}}^{\mathsf{EQ}}(\bm{v}; \bm{v}_m) \bm{\Omega}^\top \widehat{\mathbf{N}}^{\mathsf{IE}}(\bm{v}; \bm{v}_m)  \\
&\textrm{s.t.} ~~\mathbf{C}^{\mathsf{EQ}}(\bm{v}; \bm{v}_m) = \bm{0}, \mathbf{C}^{\mathsf{IE}}(\bm{v}; \bm{v}_m) \leq \bm{0}, \nonumber
\end{align}
where  $\bm{\Lambda}$ and  $\bm{\Omega}$ are the Lagrangian multipliers associated with  $\widehat{\mathbf{N}}^{\mathsf{EQ}}(\bm{v}; \bm{v}_m)$ and  $\widehat{\mathbf{N}}^{\mathsf{IE}}(\bm{v}; \bm{v}_m)$, respectively. It can be verified that the objective of $\widehat{\bm{\mathcal{P}}}^{{(m)}}$ is strongly convex and the constraints of $\widehat{\bm{\mathcal{P}}}^{{(m)}}$ are convex. After this series of transformations we have designed, the final problem can be solved via convex optimization techniques. It can be shown the sequence $\{\bm{v}_m\}$ is feasible for $\widehat{\bm{\mathcal{P}}}$ and non-increasing, which asymptotically reaches a stationary solution of $\widehat{\bm{\mathcal{P}}}$. A formal proof of this for a similarly structured non-convex mixed integer program can be found in Appendix E of \cite{ganguly2023multi}.

\section{Numerical Evaluation}\label{sec:num}
\subsection{Simulation Setup}

\subsubsection{System settings} We incorporate the effect of fading in channel gains  $h_i^{\mathsf{U}, (g)}$,  $h_i^{\mathsf{D}, (g)}$ (in~\eqref{eq:rateUP} and~\eqref{eq:rateDOWN}). For the uplink channel, we consider $h_i^{\mathsf{U},(g)}= \sqrt{\beta_{i}^{(g)}} u_{i}^{(g)}$
 where $u_{i}^{(g)} \sim \mathcal{CN}(0,1)$ captures Rayleigh fading, and 
$\beta_{i}^{(g)} = \beta_0 - 10\widetilde{\alpha}\log_{10}(d^{(g)}_{i}/d_0)$~\cite{tse2005fundamentals}. Here, $\beta_0=-30$dB, $d_0=1$m, $\widetilde{\alpha}=3$, and $d^{(g)}_{i}$ is the  distance between device $i$ and the base station (BS), where the server resides, at each global aggregation $g$. The downlink channel gain  $h_i^{\mathsf{D},(g)}$ is generated using the same approach. In our system, 10 devices are randomly placed in a circular area with radius of 25m, with the BS in the center. Other specific settings are tabulated in Table~\ref{tab:param_values}.

\begin{table}
  \caption{Neural network architectures used for each task.}
  \centering
  \label{tab:nn_all_tasks}
  \resizebox{0.6\textwidth}{!}{
  \begin{tabular}{cccc}
    \toprule
    Tasks & Layer (type) & Output Shape & \# of Params \\
    \midrule
    \multirow{4}{*}{SVHN} &Conv2D-1 & [-1, 32, 28, 28] &   2,432\\
    & Conv2D-2 & [-1, 64, 10, 10] & 51,264 \\
    & Linear-3 & [-1, 256] & 409,856 \\
    & Linear-4 & [-1, 10] & 2,570 \\ \midrule
    MNIST & Linear-1 & [-1, 10] & 7,850\\ \midrule
    Fashion-MNIST & Linear-1 & [-1, 10] & 7,850\\
  \bottomrule
\end{tabular}}
\end{table}

\begin{table}
    \centering
  \caption{Default network/system characteristics employed in simulations.}
  \label{tab:param_values}
  \resizebox{0.8\textwidth}{!}{
  \begin{tabular}{cc|cc|cc}
    \toprule
     \textbf{Param} & \textbf{Value} & \textbf{Param} & \textbf{Value} & \textbf{Param} & \textbf{Value} \\\hline
    $V_1$ & 2 & $ V_2$ & 5 & $\rho$ & 1\\ 
    $\xi_i$ & $[\expnumber{2}{-22}, \expnumber{2}{-19}]$ & $a_{i,j}$ & $[\expnumber{5}{2}, \expnumber{5}{3}]$ & $\sigma_j$ & $4096$\\
    $p_i^{\mathsf{U}}$ & $250$mW & $p_i^{\mathsf{D}}$ & $100$ mW & $B_i^{\mathsf{U}}$ & $1$MHz \\
    $B_i^{\mathsf{D}}$ & $100$ KHz & $K_j$ & $5$ & $(c_1,c_2,c_3)$ & $(\expnumber{1}{-9}, 1, 1)$ \\
  \bottomrule 
\end{tabular}}
{\\ \vspace{2mm} * Interval [a,b] means sampling the value between $a$ and $b$ according to the uniform distribution.}
\end{table}

\subsubsection{Task settings and Dataset partition}  We consider three real-world classification tasks, based on the standard MNIST, Fashion-MNIST, and SVHN datasets. Each of these are popular datasets for image recognition employed in FL research \cite{li2022federated, durmus2021federated}. The specific NN architectures for different tasks are given in Table I. The higher complexity of the SVHN task is consistent with its significantly larger NN model.

We have run experiments considering three non-iid data partitioning strategies: \textit{(a) 2-label partitioning:} Every device has access to data from only $2$ of the $10$ labels for each task. The selection of labels is conducted randomly across devices. This type of partitioning is widely considered, e.g., in \cite{han2023federated}. \textit{(b) Dirichlet partitioning:} Each device is allocated its fraction of labels according to a Dirichlet distribution for each task, with parameter $\beta = 0.5$. This type of partitioning has also been widely considered, e.g., in \cite{durmus2021federated, li2022federated}. \textit{(c) Varied partitioning:} $n$ devices are allocated labels from $n$ classes (e.g., 3 devices have 3 classes). In each case, all local datasets have the same size.
Due to space limitations, we only show the simulation results obtained from the Dirichlet partitioning here. The results for the other two partitionings can be found in Appendix C of our technical report \cite{chang2023asynchronous}. The results from each case are qualitatively consistent.

For optimization parameters, $\gamma_j$ and $\chi_j$ are set to $1$ by default for all tasks $j$. All other parameters are measured using the assumptions and definitions in \Cref{subsec:assumptions_definitions}. We implemented the concept drift by adding more data to the local dataset across the global aggregation periodically until the end of the training. In our simulations, we adopted the polynomial formula in \cite{xie2019asynchronous} to vary the weighted coefficient $\alpha_j$ based on the staleness of the local model.

\subsection{Results and Discussions}

\begin{figure}[t]
\centering
\includegraphics[width = \textwidth]{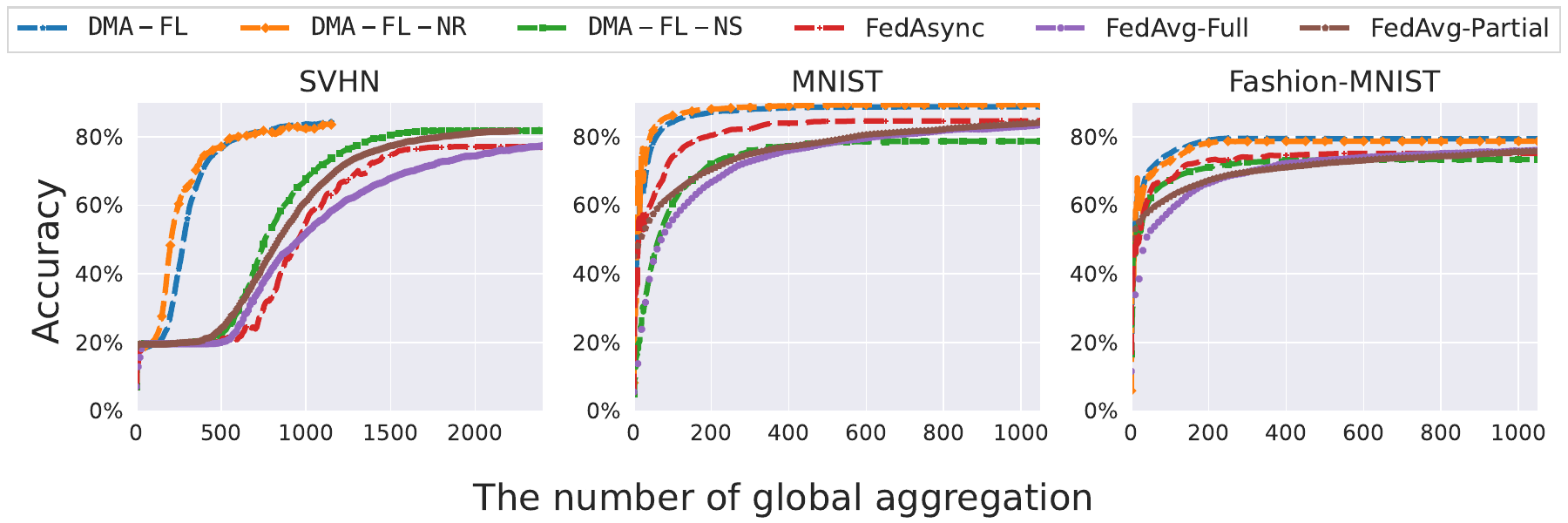}
\vspace{-12mm}
\caption{ML training convergence of all schemes for all tasks, with labels distributed according to Dirichlet distribution. {\tt DMA-FL} and {\tt DMA-FL-NR} outperform all other baselines on all tasks. As we will see in Figure \ref{fig:acc_energy_dirichlet}, the {\tt DMA-FL-NR} also incur significantly higher energy consumption on each task. Thus, {\tt DMA-FL} has the best trade-off between the performance and energy consumption.}
\label{fig:acc_iteration_dirichlet}
\vspace{3mm}
\includegraphics[width = \textwidth]{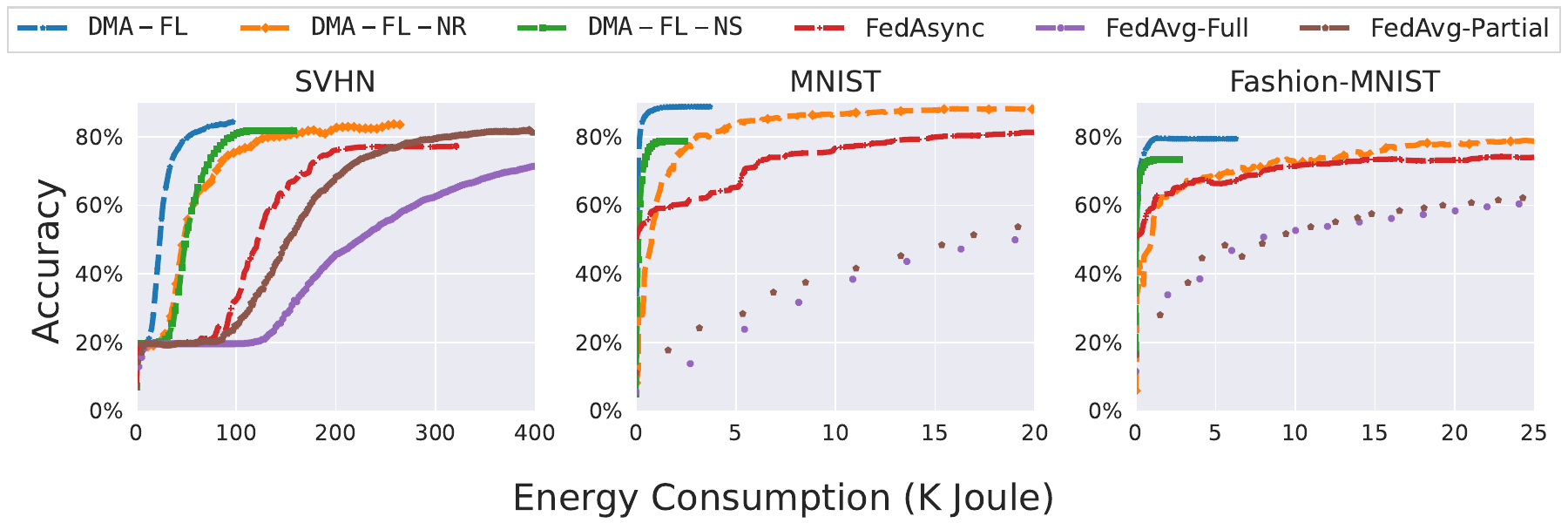}
\vspace{-12mm}
\caption{Accuracy vs. energy consumption trade-off obtained by each method across all tasks with labels distributed according to Dirichlet distribution. Our proposed scheme can reach a target level of accuracy with significantly less energy consumption in comparison with all the baselines on each task.}
\label{fig:acc_energy_dirichlet}
\end{figure}

\subsubsection{Comparison with baselines} We first compare the performance of {\tt DMA-FL} with optimized device scheduling and resource allocation (obtained through solving $\mathcal{P}$) against several baselines: 
\begin{itemize}

\item[(i)] FedAsync \cite{xie2019asynchronous}: This is conventional fully asynchronous FL. This scheme serves the baseline without optimization over both the device scheduling and resource allocation.
\item[(ii)] FedAvg-Full: This scheme is a variant of FedAvg \cite{mcmahan2017communication}. In this scheme, every device participates in the global aggregation in a synchronous manner. Each global aggregation is performed when the server receives all the trained local models.
\item[(iii)] FedAvg-Partial: This scheme is a variant of FedAvg employed in \cite{ma2021fedsa}. In this scheme, ach global aggregation is performed when the server receives a fixed number of the local models.
\item[(iv)] {\tt DMA-FL-NR}: This is our scheme with optimization over device scheduling variables (i.e., $\bm{R}^{(g)}$ and $\bm{U}^{(g)}$) but not resource allocation variables (i.e., CPU frequencies $\bm{f}^{(g)}$, mini-batch size $\bm{B}^{(g)}$, the number of local SGD $\bm{e}^{(g)}$, and idle time $\bm{T}^{\mathsf{ID}, (g)}$). This baseline helps us assess the importance of resource optimization in our problem setting.
\item[(v)] {\tt DMA-FL-NS}: This is our proposed scheme with optimization over resource allocation variables but not device scheduling variables. This baseline helps us assess the importance of device scheduling in our problem setting.
\end{itemize}


Fig. \ref{fig:acc_energy_dirichlet} compares the convergence behavior of the algorithms for a non-iid partitioning according to the Dirichlet distribution. Fig. \ref{fig:acc_energy_dirichlet} presents the corresponding energy consumption plots to reach the accuracy levels in Fig. \ref{fig:acc_iteration_dirichlet}. In Fig. \ref{fig:acc_iteration_dirichlet}, we see that {\tt DMA-FL} and {\tt DMA-FL-NR} obtain improvements in training performance over global aggregations compared to other baselines on all tasks. The improvements on SVHN are most substantial, consistent with this task being the most complex and thus having the largest loss contribution to $\bm{\mathcal{P}}$'s objective. Most importantly, in Fig. \ref{fig:acc_energy_dirichlet}, we see that {\tt DMA-FL} obtains a substantially better training accuracy vs. energy consumption tradeoff compared to each baseline. The marginal advantage of {\tt DMA-FL-NR} on SVHN from Fig. \ref{fig:acc_energy_dirichlet} comes with substantially higher energy consumption requirements to reach target accuracy levels in Fig. \ref{fig:acc_energy_dirichlet}. This validates the gains provided by {\tt DMA-FL}'s joint optimization of device scheduling and resource allocation.

We also see that the baseline asynchronous schemes ({\tt DMA-FL-NR}, and {\tt DMA-FL-NS}, and FedAsync) outperform the synchronous schemes (FedAvg-Full and FedAvg-Partial) in terms of energy consumption in Fig. \ref{fig:acc_energy_dirichlet}. The asynchronous training styles are more resource efficient inherently since they can skip engaging devices with higher energy consumption during specific model aggregation iterations. Among the asynchronous schemes, we can verify the benefits of optimized resource allocation in terms of improved energy efficiency by comparing {\tt DMA-FL} and {\tt DMA-FL-NS} with {\tt DMA-FL-NR} and FedAsync. {\tt DMA-FL} and {\tt DMA-FL-NS} obtain a better accuracy-energy tradeoff than the other asynchronous schemes in Fig. \ref{fig:acc_energy_dirichlet}. On the other hand, in Fig. \ref{fig:acc_iteration_dirichlet}, we see that {\tt DMA-FL-NR} performs closest to {\tt DMA-FL} (even outperforming it for MNIST), obtaining better convergence over global aggregations than {\tt DMA-FL-NS}. Since {\tt DMA-FL-NR} is not considering resource consumption, it optimizes the device scheduling for convergence speed, but consumes significant energy on each task. {\tt DMA-FL} balances both of these objectives to obtain the best overall performance.

\subsubsection{Impact of task importance} We next study the impact of task importance (i.e., $\gamma_j$) in the objective of $\mathcal{P}$ on model performance and resource savings. Figs. \ref{fig:task_importance} and \ref{fig:ratio_local_SGD} give the results.

In Fig. \ref{fig:task_importance_a}, we fixed the model importance of two models and increase the importance of one specific model. For SVHN, when it is emphasized, it has parameter $\gamma_1/\chi_1 = \expnumber{1}{7}$ and when it is regular (it is not emphasized), it has parameter $\gamma_1/\chi_1 = \expnumber{1}{1}$. For MNIST and Fashion-MNIST, when they are emphasized, they have parameters $\gamma_i/\chi_i = \expnumber{1}{9}, \, \forall i \in \{2,3\}$, and when they are regular, they have parameters  $\gamma_i/\chi_i = \expnumber{1}{-9}, \; \forall i \in \{2,3\}$. As can be seen, the performance of the emphasized task is significantly boosted as compared to the regular ones. However, this comes at the price of more energy consumption as shown in Fig. \ref{fig:task_importance_b}. In Fig. \ref{fig:task_importance_b}, the left y-axis is the energy consumption for the SVHN task, and the right y-axis is the energy consumption for MNIST and Fashion-MNIST; the need for two different y-axis scales is consistent with SVHN employing a more complex neural network architecture in Table~\ref{tab:nn_all_tasks}. 

\begin{figure*}
    \centering
    \setkeys{Gin}{width=0.5\linewidth}
    \subfloat[Accuracy vs the number of global aggregations]{\label{fig:task_importance_a}\includegraphics{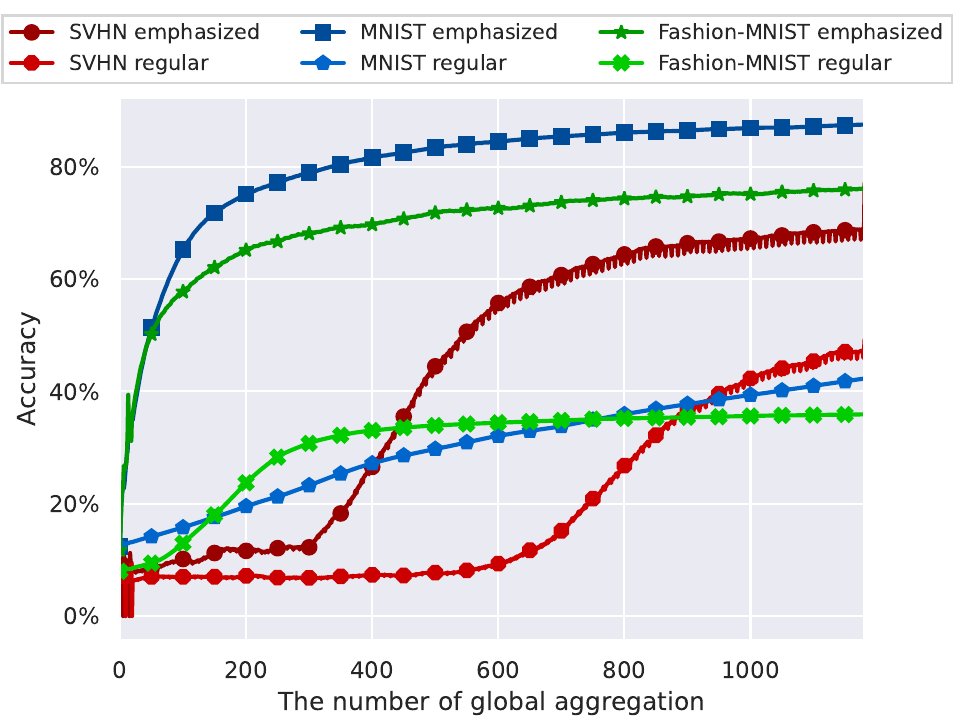}}
    \subfloat[Energy consumption vs the number of global aggregations]{\label{fig:task_importance_b}\includegraphics{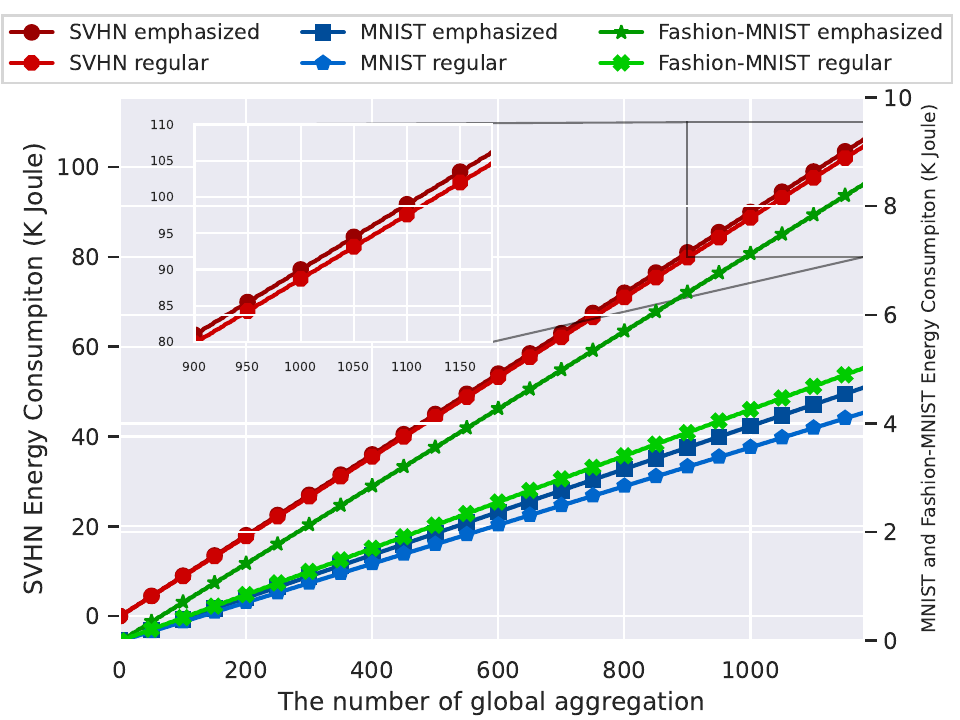}}
\caption{Impact of varying task importance in {\tt DMA-FL}. We see that each emphasized task experiences (a) a boost in performance (b) at the price of additional energy consumption.}
\vspace{-3mm}
\label{fig:task_importance}
\end{figure*}

\begin{figure*}[!t]
    \centering
    \setkeys{Gin}{width=0.33\linewidth}
    \subfloat[SVHN emphasized]{\label{fig:local_SGD_ratio_a}\includegraphics{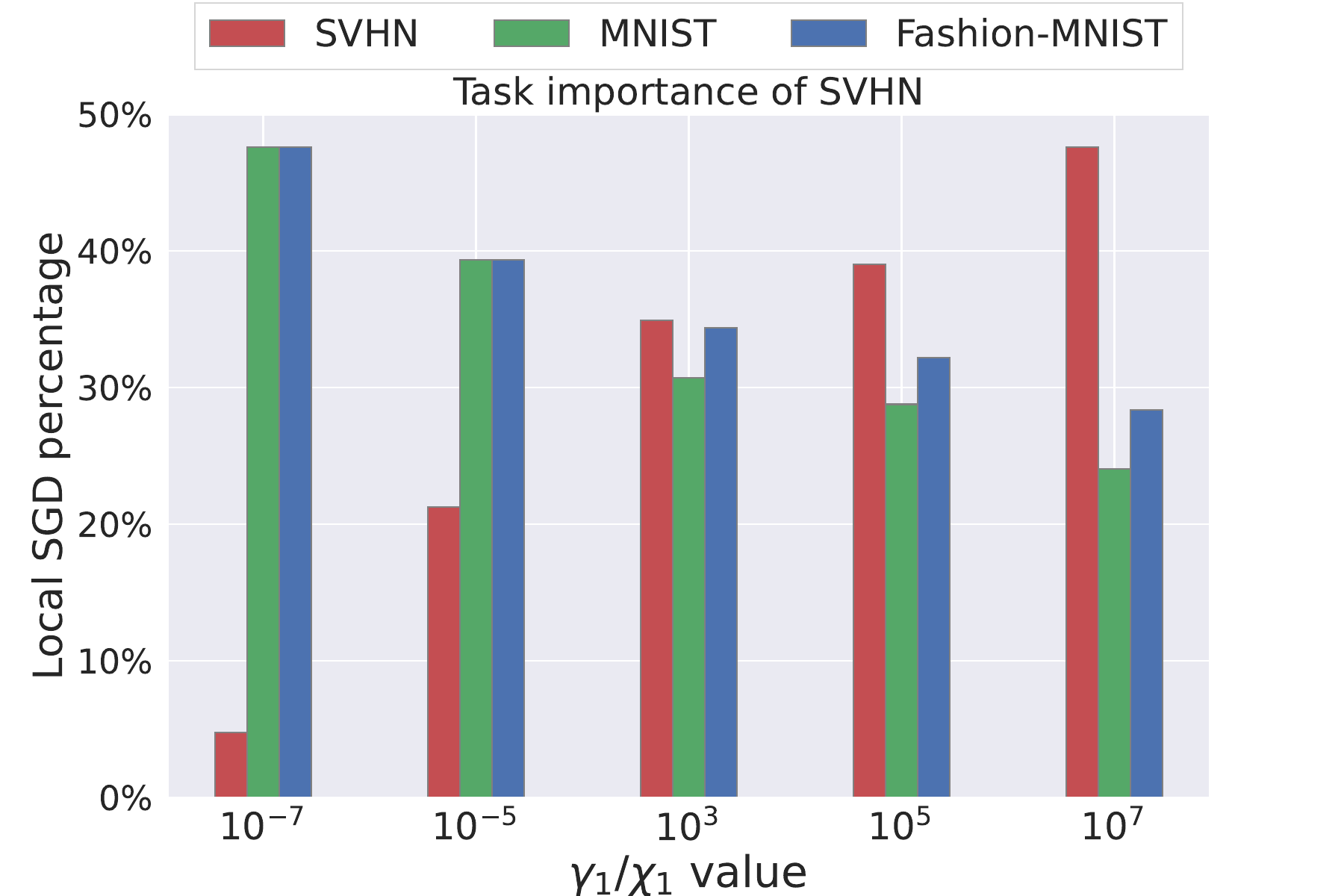}}
    \subfloat[MNIST emphasized]{\label{fig:local_SGD_ratio_b}\includegraphics{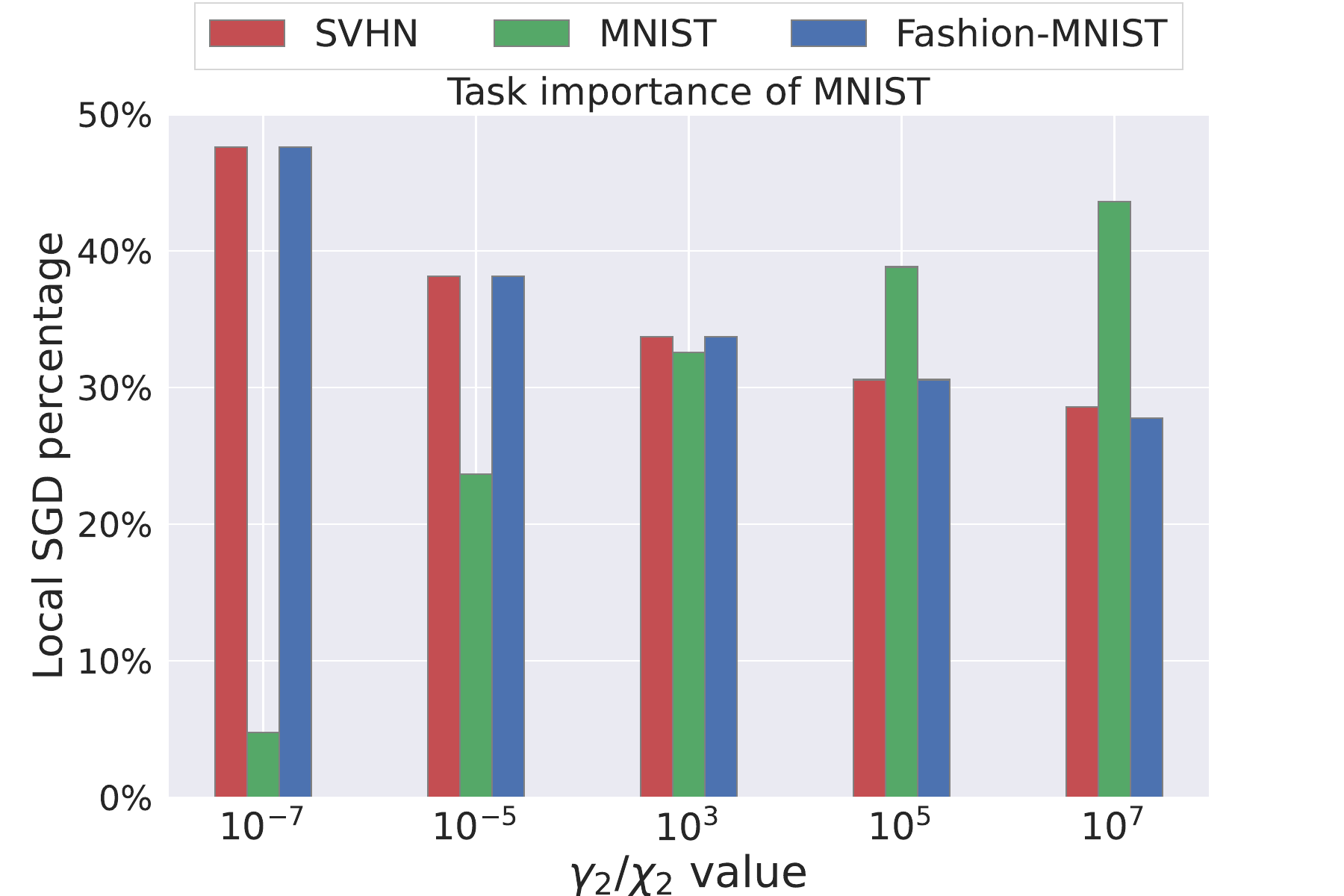}}
    \subfloat[Fashion-MNIST emphasized]{\label{fig:local_SGD_ratio_c}\includegraphics{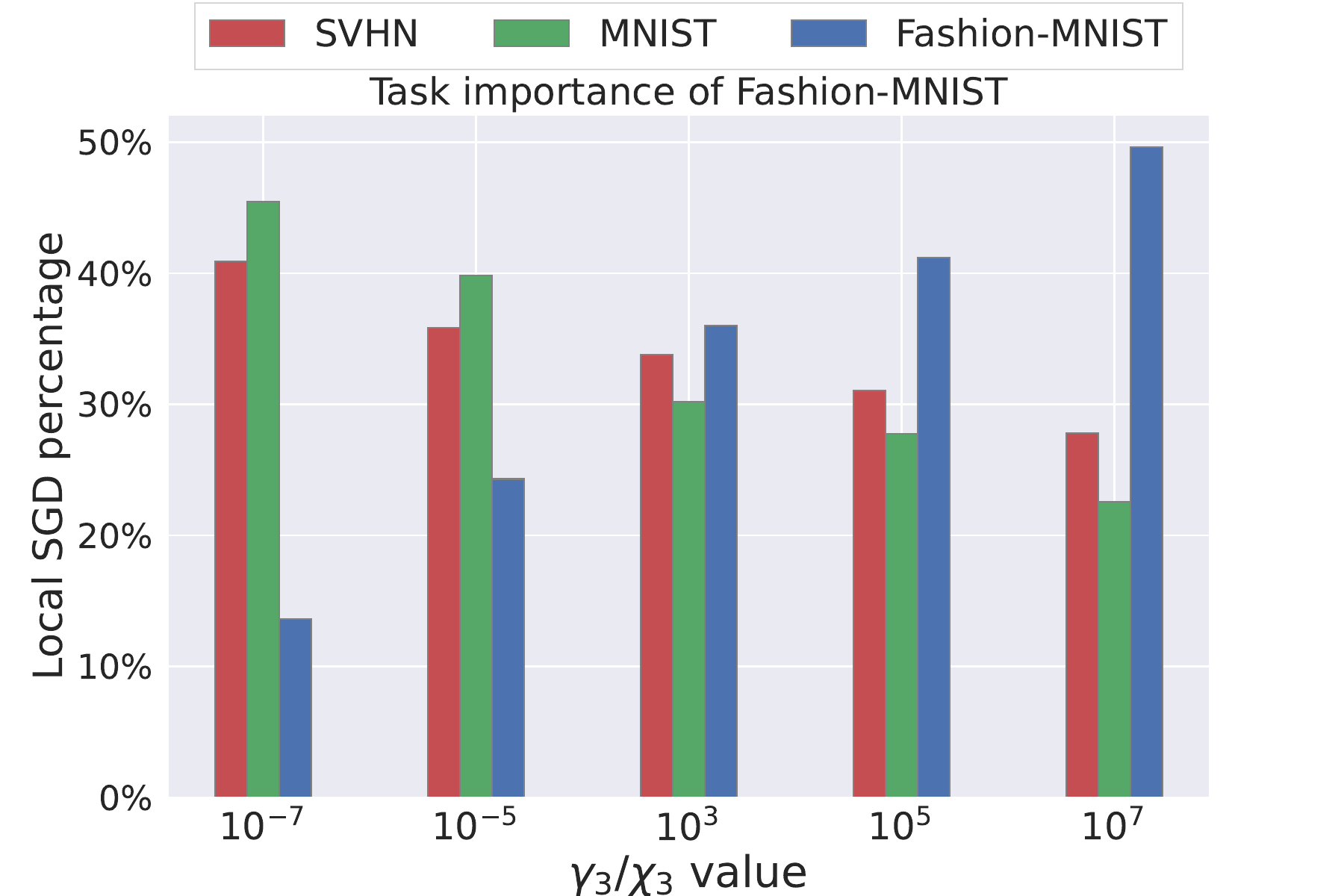}}
\caption{The percentage of local SGD iterations allocated to different tasks as the importance assigned to one task increases. We see that the local SGD percentage increases with the task importance in each case.}
\label{fig:ratio_local_SGD}
\vspace{-6mm}
\end{figure*}

Finally, we study the relationship between performance gain and resource allocation. Since the model performance relies heavily on the number of local SGD, we more closely analyze its behavior in Fig. \ref{fig:ratio_local_SGD}. This figure shows that as we increase the importance of one task (i.e. the value of $\gamma_i/\chi_i$ gets larger), the percentage of local SGD iterations allocated to that task increases, which is one of the reasons behind the performance gain in Fig. \ref{fig:task_importance_a}.

\subsubsection{Impact of concept drift}
Lastly, we study the impact of concept drift on resource allocation among three tasks, captured by term (f) in \eqref{weighted_sum_weight_local_number_updates_concept_drift_integration_y}. The results are shown in Figs. \ref{fig:local_computation_time} and \ref{fig:idle_time}. 

In Fig. \ref{fig:local_computation_time}, we increase the active concept drift for one task and fix that of the other two tasks. For all tasks shown, the local computation time of the task whose active concept drift is increased decreases because of product $y_{i,j}^{\mathsf{AC}, (g)}(t)\Delta_{i,j}^{\mathsf{AC}}(t)$ in term (f) and the definition of  $y_{i,j}^{\mathsf{AC}, (g)}(t)$ in Definition \ref{def:rectangular_functions}. In particular, since higher active concept drift of one task implies more drastic data variations during local computation period of that task, the local computation time of the task has been reduced to avoid the local model being trained on the outdated data resulting in poor model performance. Also, in view of the fixed energy budget $E_i^{\mathsf{B}}$ imposed by \eqref{const:energy_budget}, to have a higher overall model accuracy over all tasks, more resources have been allocated to other two tasks, leading to longer local computation time.

\begin{figure*}[!t]
    \centering
    \setkeys{Gin}{width=0.32\linewidth}
    \subfloat[SVHN]{\label{fig:local_computation_time_a}\includegraphics{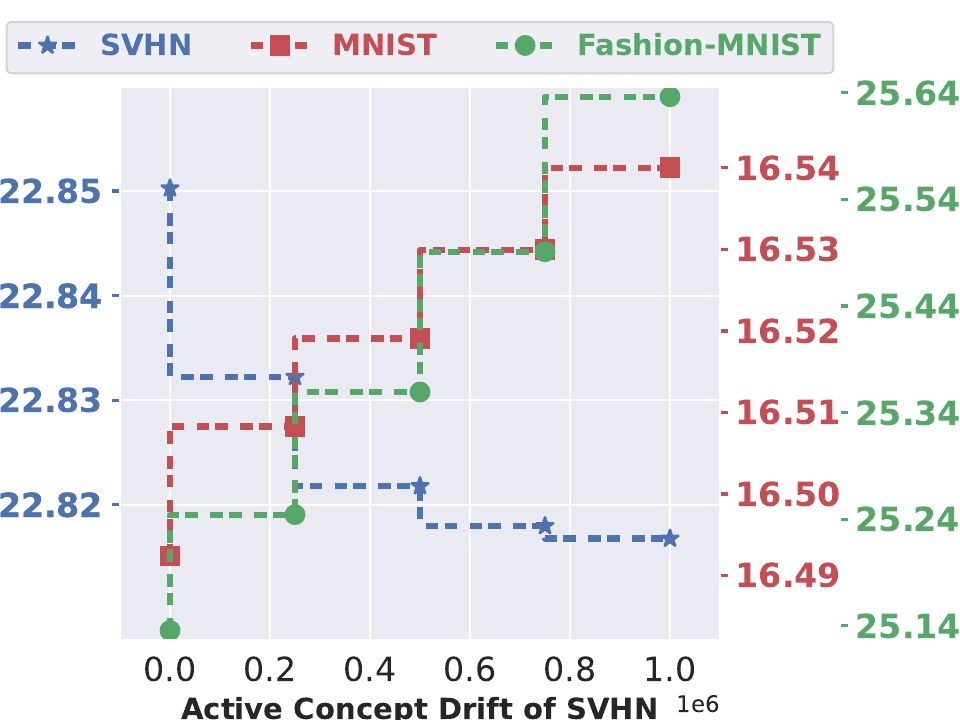}}
    \hspace{0.002\linewidth}
    \subfloat[MNIST]{\label{fig:local_computation_time_b}\includegraphics{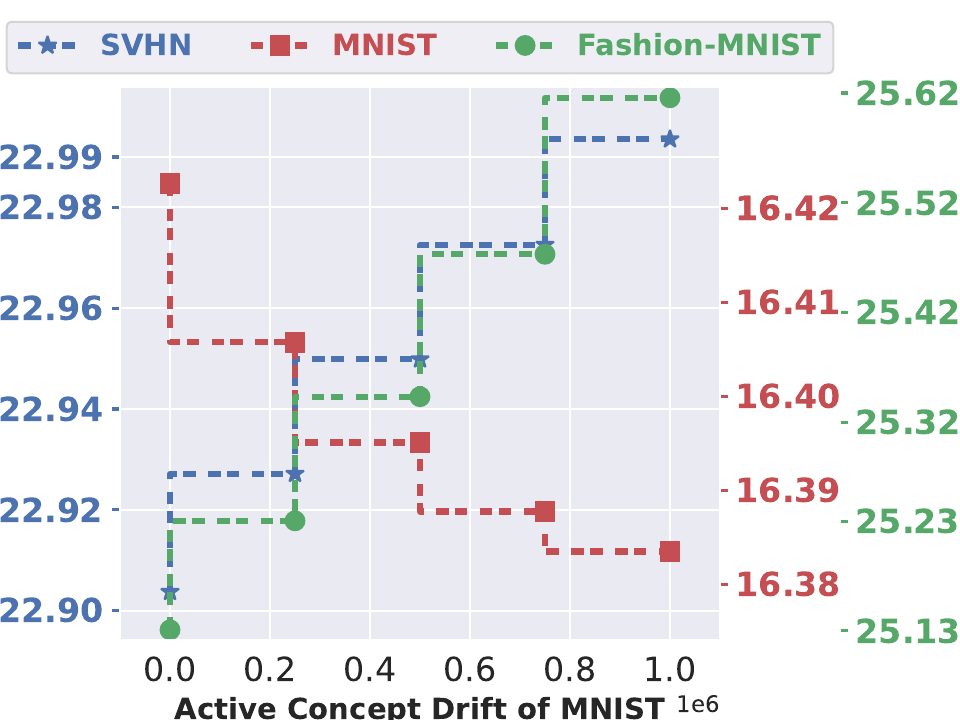}}
    \hspace{0.002\linewidth}
    \subfloat[Fashion-MNIST]{\label{fig:local_computation_time_c}\includegraphics{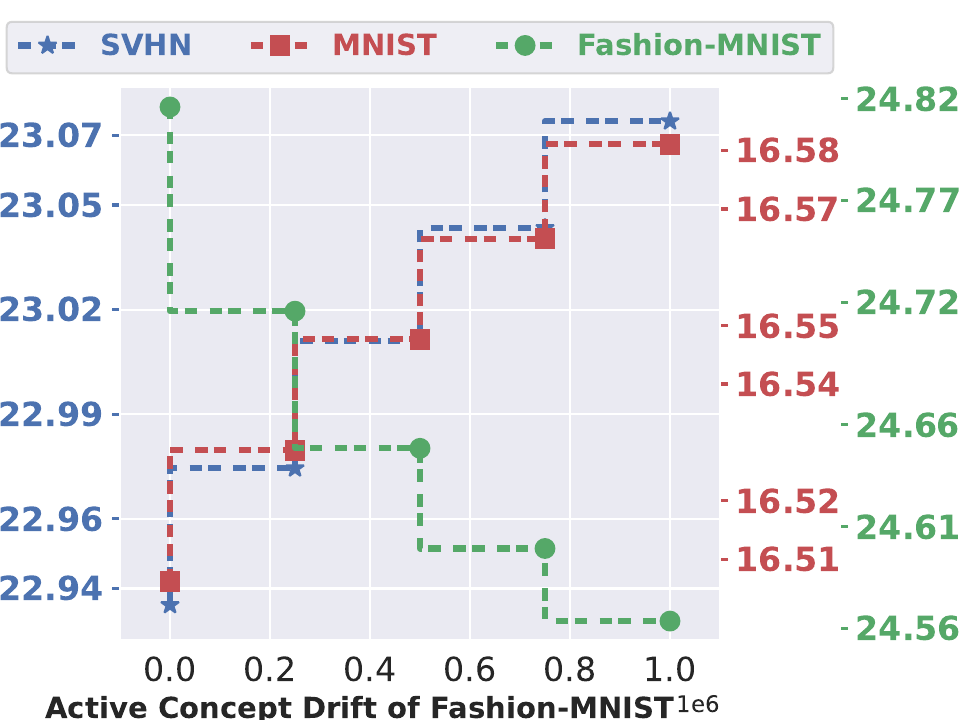}}
\caption{Local computation time of the three tasks under varying active concept drift. The local computation time of each task decreases as its active concept drift increases to keep track of the dynamic data variations.}
\label{fig:local_computation_time}
\vspace{-3mm}
\end{figure*}

\begin{figure*}[!t]
    \centering
    \setkeys{Gin}{width=0.32\linewidth}
    \subfloat[SVHN]{\label{fig:idle_time_a}\includegraphics{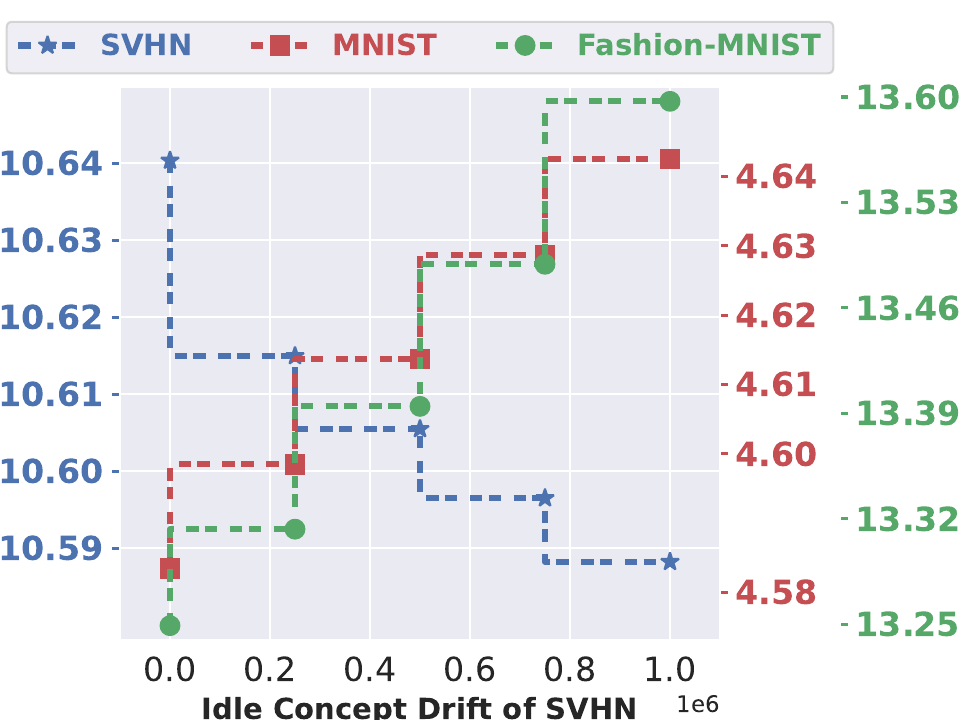}}
    \hspace{0.003\linewidth}
    \subfloat[MNIST]{\label{fig:idle_time_b}\includegraphics{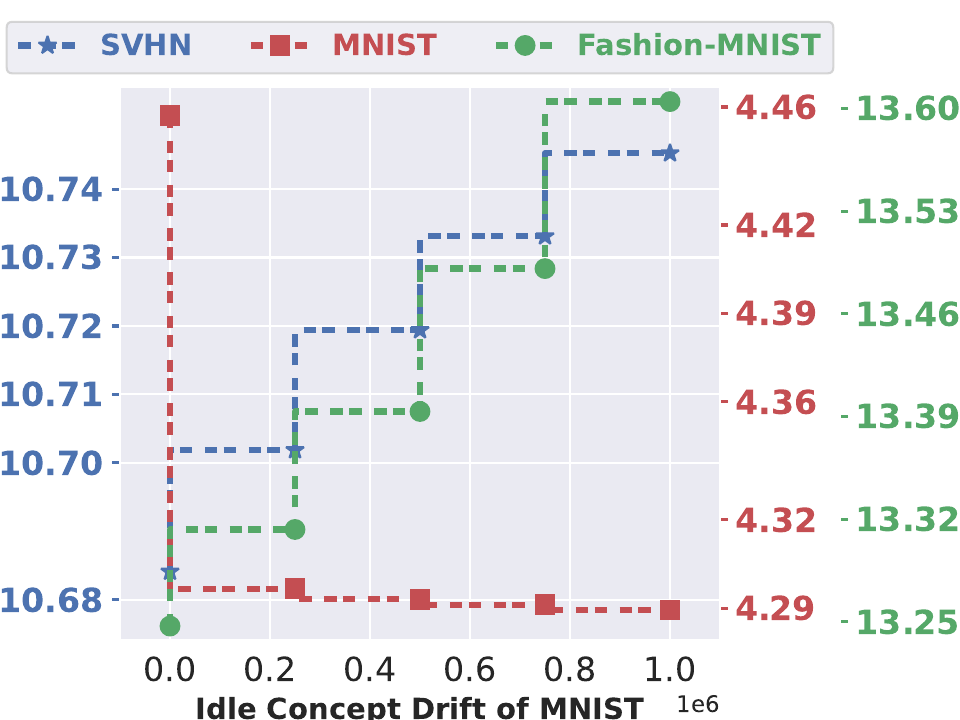}}
    \hspace{0.003\linewidth}
    \subfloat[Fashion-MNIST]{\label{fig:idle_time_c}\includegraphics{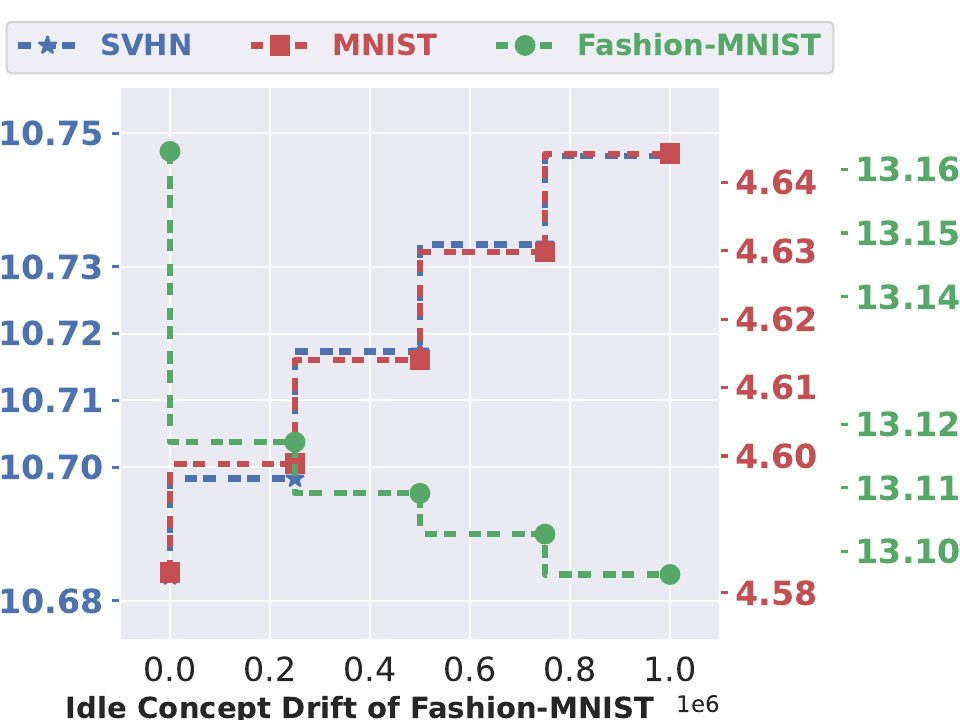}}
\caption{Idle time of the three tasks under varying idle concept drift. The idle time of each task decreases as its idle concept drift increases to keep track of the dynamic data variations.}
\label{fig:idle_time}
\vspace{-7mm}
\end{figure*}

In Fig \ref{fig:idle_time}, we increase the idle concept drift for one task and fix that of the other two tasks. Likewise, for all tasks shown, the idle time of the task whose idle concept drift is increased decreases due to the product  $y_{i,j}^{\mathsf{ID}, (g)}(t)\Delta_{i,j}^{\mathsf{ID}}(t)$ in term (f) and the definition of $y_{i,j}^{\mathsf{ID}, (g)}(t)$ in Definition \ref{def:rectangular_functions}. In particular, higher idle concept drift of one task means more severe data variations during idle period of that task. As a result, the idle time of that task has been decreased to mitigate the adverse impact on the model performance. Since there is the QoE constraint \eqref{const:QoE}, reducing the idle time leads to longer periods of device computation and upstream/downstream communication, which in turn results
in more energy consumption for that task. In view of the fixed energy budget in constraint \eqref{const:energy_budget}, periods of device computation and upstream/downstream communication have been reduced for the other two tasks (i.e. shorter  $T_{i,j}^{\mathsf{L}, (g)}$). Because of the QoE constraint \eqref{const:QoE}, shorter  $T_{i,j}^{\mathsf{L}, (g)}$ will then result in longer idle time for the other two tasks.

\section{Conclusion}
\noindent We introduced {\tt DMA-FL}, a methodology for dynamic/online multi-task asynchronous FL over heterogeneous networks. We introduced a new set of scheduling tensors and rectangular functions using which we carried out the ML convergence analysis of {\tt DMA-FL} capturing dynamic data variations (i.e. concept drift) under arbitrary device participation order. We then formulated resource-aware {\tt DMA-FL} that jointly  optimizes device scheduling and resource allocation aiming to minimize the loss of different ML models trained under {\tt DMA-FL} and devices energy consumption. We solved the problem through integer variable relaxation and successive convex approximations, which provide convergence to the stationary point. Through numerical simulations, we revealed the performance gains  of {\tt DMA-FL} compared to state of the art methods and studied the impact of different network settings on device scheduling and resource allocation.

\bibliographystyle{IEEEtran}
\bibliography{MMMT_asynchronous_FL}
\newpage
\onecolumn
\appendices
\section{Proof of the convergence bound (\ref{weighted_sum_weight_local_number_updates_concept_drift_integration_y})} \label{apx:convergence_bound}
This definition of smoothness is equivalent to the following
\begin{equation}
    F_{i,j}(\mathbf{w}) \leq F_{i,j}(\mathbf{w}') + \innerproduct{\nabla F_{i,j}(\mathbf{w}')}{\mathbf{w} - \mathbf{w}'} + \frac{\beta}{2}\twonormsquare{\mathbf{w} - \mathbf{w}'} \quad \forall \mathbf{w}, \mathbf{w}' \in \mathbb{R}^{M_{i,j}} 
\end{equation}
Because $F_j(\cdot, \cdot)$ is $\beta$-smoothness, for any $\ell$, we have
\begin{align}  
    &\mathsf{E}\left[ F_j(\mathbf{w}_{i,j}^{\ell,(g)}) \right] \nonumber \\
    &\leq \mathsf{E}\left[ F_j(\mathbf{w}_{i,j}^{\ell-1,(g)}) \right] + \mathsf{E}\left[\innerproduct{\nabla F_j(\mathbf{w}_{i,j}^{\ell-1,(g)})}{\mathbf{w}_{i,j}^{\mathsf{F},(g)} - \mathbf{w}_{i,j}^{\ell-1,(g)}}\right] +\frac{\beta}{2} \mathsf{E}\left[ \twonormsquare{\mathbf{w}_{i,j}^{\ell,(g)} - \mathbf{w}_{i,j}^{\ell-1,(g)}}  \right]\nonumber \\
    & \leq \mathsf{E}\left[ F_j(\mathbf{w}_{i,j}^{\ell-1,(g)}) \right] - \eta_{j}^{(g)} \underbrace{\mathsf{E}\left[\innerproduct{\nabla F_j(\mathbf{w}_{i,j}^{\ell-1,(g)})}{ \nabla F_{i,j}^{\mathsf{R}}(\mathbf{w}_{i,j}^{\ell-1,(g)}, \mathcal{B}_{i,j}^{\ell-1,(g)})}\right]}_{(a)} \label{start_no_surrogate} \\
    &+ \frac{\beta}{2}\left( \eta_{j}^{(g)} \right)^2 \underbrace{\mathsf{E}\left[ \twonormsquare{\nabla F_{i,j}^{\mathsf{R}}(\mathbf{w}_{i,j}^{\ell-1,(g)}, \mathcal{B}_{i,j}^{\ell-1,(g)})}  \right]}_{(b)} \nonumber 
\end{align} 
For any two real valued vectors $\bm{a}$ and $\bm{b}$ with the same length, we have $2\innerproduct{\bm{a}}{\bm{b}} = \twonormsquare{\bm{a}} + \twonormsquare{\bm{b}} - \twonormsquare{\bm{a}-\bm{b}}$. Using this fact, we can rewrite $(a)$-term as follows.
\begin{equation}
    \begin{aligned} \label{expanded_a_term}
    &\mathsf{E}\left[\innerproduct{\nabla F_j(\mathbf{w}_{i,j}^{\ell-1,(g)})}{ \nabla F_{i,j}^{\mathsf{R}}(\mathbf{w}_{i,j}^{\ell-1,(g)}, \mathcal{B}_{i,j}^{\ell-1,(g)})}\right] \\
    &= \mathsf{E}\left[\innerproduct{\nabla F_j(\mathbf{w}_{i,j}^{\ell-1,(g)})}{ \nabla F_{i,j}^{\mathsf{R}}(\mathbf{w}_{i,j}^{\ell-1,(g)}, \mathcal{D}_{i,j}^{\ell-1})}\right] \\
    &= \frac{1}{2}\mathsf{E}\left[\twonormsquare{\nabla F_j(\mathbf{w}_{i,j}^{\ell-1,(g)})}\right] + \frac{1}{2}\mathsf{E}\left[\twonormsquare{\nabla F_{i,j}^{\mathsf{R}}(\mathbf{w}_{i,j}^{\ell-1,(g)}, \mathcal{D}_{i,j}^{\ell-1})}\right]  \\
    &- \frac{1}{2}\mathsf{E}\left[\twonormsquare{\nabla F_j(\mathbf{w}_{i,j}^{\ell-1,(g)}) - \nabla F_{i,j}^{\mathsf{R}}(\mathbf{w}_{i,j}^{\ell-1,(g)}, \mathcal{D}_{i,j}^{\ell-1})}\right] 
\end{aligned}
\end{equation}
Now we have to bound $(b)$ in (\ref{start_no_surrogate}). 
\begin{align} 
    & \mathsf{E}\left[ \twonormsquare{\nabla F_{i,j}^{\mathsf{R}}(\mathbf{w}_{i,j}^{\ell-1,(g)}, \mathcal{B}_{i,j}^{\ell-1,(g)})} \right]  \label{surrogate_b_term}\\
    & = \mathsf{E}\left[ \left\lVert\frac{1}{B_{i,j}^{\ell-1,(g)}}\sum_{d \in \mathcal{B}_{i,j}^{\ell-1,(g)}}\nabla L_{i,j}^{(g)}\big(\mathbf{w}_{i,j}^{\ell-1,(g)},d\big)\right\rVert^2 \right]\nonumber \\
    & = \mathsf{E}\left[\left\lVert\frac{\sum\limits_{d \in \mathcal{B}_{i,j}^{\ell-1,(g)}} \nabla L_{i,j}^{(g)}(\mathbf{w}_{i,j}^{\ell-1,(g)},d) }{B_{i,j}^{\ell-1,(g)}} - \frac{\sum\limits_{d \in \mathcal{D}_{i,j}^{\ell-1}} \nabla L_{i,j}^{(g)}(\mathbf{w}_{i,j}^{\ell-1,(g)},d) }{D_{i,j}^{\ell-1}} + \frac{\sum\limits_{d \in \mathcal{D}_{i,j}^{\ell-1}} \nabla L_{i,j}^{(g)}(\mathbf{w}_{i,j}^{\ell-1,(g)},d) }{D_{i,j}^{\ell-1}}\right\rVert^2 \right] \nonumber \\
    & \leq 2\mathsf{E}\left[ \left\lVert\frac{\sum_{d \in \mathcal{B}_{i,j}^{\ell-1,(g)}} \nabla L_{i,j}^{(g)}(\mathbf{w}_{i,j}^{\ell-1,(g)},d) }{B_{i,j}^{\ell-1,(g)}} - \frac{\sum_{d \in \mathcal{D}_{i,j}^{\ell-1}} \nabla L_{i,j}^{(g)}(\mathbf{w}_{i,j}^{\ell-1,(g)},d) }{D_{i,j}^{\ell-1}}\right\rVert^2 \right] \nonumber\\
    &+ 2 \mathsf{E}\left[ \left\lVert\frac{\sum_{d \in \mathcal{D}_{i,j}^{\ell-1}} \nabla L_{i,j}^{(g)}(\mathbf{w}_{i,j}^{\ell-1,(g)},d) }{D_{i,j}^{\ell-1}}\right\rVert^2 \right]\nonumber\\
    & \stackrel{(i)}{=} 2(1-\frac{B_{i,j}^{\ell-1,(g)}}{D_{i,j}^{\ell-1}})\frac{S_{i,j}^{\ell-1,(g)}}{B_{i,j}^{\ell-1,(g)}} +  2 \mathsf{E}\left[ \left\lVert\frac{\sum_{d \in \mathcal{D}_{i,j}^{\ell-1}} \nabla L_{i,j}^{(g)}(\mathbf{w}_{i,j}^{\ell-1,(g)},d) }{D_{i,j}^{\ell-1}}\right\rVert^2 \right]\nonumber \\
    & = 2(1-\frac{B_{i,j}^{\ell-1,(g)}}{D_{i,j}^{\ell-1}})\frac{S_{i,j}^{\ell-1,(g)}}{B_{i,j}^{\ell-1,(g)}} + 2 \mathsf{E}\left[ \left\lVert\nabla F_{i,j}^{\mathsf{R}}(\mathbf{w}_{i,j}^{\ell-1,(g)}, \mathcal{D}_{i,j}^{\ell-1})\right\rVert^2 \right] \nonumber
\end{align}
In $(i)$, we make use of the variance of sample mean in page 53 of the book ``Sampling: Design and Analysis''. $S_{i,j}^{\ell-1,(g)}$ denotes the variance of the gradients of regularized loss function evaluated at the particular local descent iteration $\ell-1$ in the particular local training period $g$ for the parameter $\mathbf{w}_{i,j}^{\ell-1,(g)}$. We can further relate the variance of the gradient to sample variance of data using the local data variability as follows
\begin{align} 
    &S_{i,j}^{\ell-1,(g)} = \frac{1}{D_{i,j}^{\ell-1}-1} \sum_{d \in \mathcal{D}_{i,j}^{\ell-1}} \left\lVert\nabla L_{i,j}^{(g)}(\mathbf{w}_{i,j}^{\ell-1,(g)},d) - \sum_{d' \in \mathcal{D}_{i,j}^{\ell-1}}\frac{\nabla L_{i,j}^{(g)}(\mathbf{w}_{i,j}^{\ell-1,(g)},d')}{D_{i,j}^{\ell-1}}\right\rVert^2\label{sample_variance_gradient_data_asynchronous_no_surrogate} \\
    & = \frac{1}{\left( D_{i,j}^{\ell-1}-1 \right) \left( D_{i,j}^{\ell-1} \right)^2} \sum_{d \in \mathcal{D}_{i,j}^{\ell-1}} \left\lVert  D_{i,j}^{\ell-1}\nabla L_{i,j}^{(g)}(\mathbf{w}_{i,j}^{\ell-1,(g)},d) - \sum_{d' \in \mathcal{D}_{i,j}^{\ell-1}}\nabla L_{i,j}^{(g)}(\mathbf{w}_{i,j}^{\ell-1,(g)},d')\right\rVert^2 \nonumber\\
    & \stackrel{(i)}{\leq} \frac{D_{i,j}^{\ell-1}-1}{\left( D_{i,j}^{\ell-1}-1 \right) \left( D_{i,j}^{\ell-1} \right)^2} \sum_{d \in \mathcal{D}_{i,j}^{\ell-1}}\sum_{d' \in \mathcal{D}_{i,j}^{\ell-1}} \left\lVert\nabla L_{i,j}^{(g)}(\mathbf{w}_{i,j}^{\ell-1,(g)},d) - \nabla L_{i,j}^{(g)}(\mathbf{w}_{i,j}^{\ell-1,(g)},d')\right\rVert^2 \nonumber\\
    & = \frac{\left( D_{i,j}^{\ell-1}-1 \right) \Theta^2}{\left( D_{i,j}^{\ell-1}-1 \right) \left( D_{i,j}^{\ell-1} \right)^2}  \sum_{d \in \mathcal{D}_{i,j}^{\ell-1}}\sum_{d' \in \mathcal{D}_{i,j}^{\ell-1}} \twonormsquare{d - d'} \nonumber  \\
    & = \frac{\left( D_{i,j}^{\ell-1}-1 \right) \Theta^2}{\left( D_{i,j}^{\ell-1}-1 \right) \left( D_{i,j}^{\ell-1} \right)^2}  \sum_{d \in \mathcal{D}_{i,j}^{\ell-1}}\sum_{d' \in \mathcal{D}_{i,j}^{\ell-1}} \twonormsquare{d - d' - \lambda_{i,j}^{\ell-1,(g)}+\lambda_{i,j}^{\ell-1,(g)}} \nonumber \\
    & = \frac{\left( D_{i,j}^{\ell-1}-1 \right) \Theta^2}{\left( D_{i,j}^{\ell-1} \right)^2} \sum_{d \in \mathcal{D}_{i,j}^{\ell-1}}\sum_{d' \in \mathcal{D}_{i,j}^{\ell-1}} \frac{\twonormsquare{d - \lambda_{i,j}^{\ell-1,(g)}} + \twonormsquare{d' - \lambda_{i,j}^{\ell-1,(g)}} + 2 \innerproduct{d - \lambda_{i,j}^{\ell-1,(g)}}{d' - \lambda_{i,j}^{\ell-1,(g)}}  }{D_{i,j}^{\ell-1}-1} \nonumber\\
    & \stackrel{(ii)}{=} \frac{\left( D_{i,j}^{\ell-1}-1 \right) \Theta^2}{\left( D_{i,j}^{\ell-1} \right)^2} \sum_{d \in \mathcal{D}_{i,j}^{\ell-1}}\sum_{d' \in \mathcal{D}_{i,j}^{\ell-1}} \frac{\twonormsquare{d - \lambda_{i,j}^{\ell-1,(g)}} + \twonormsquare{d' - \lambda_{i,j}^{\ell-1,(g)}}  }{D_{i,j}^{\ell-1}-1} \nonumber \\
    & = \frac{\left( D_{i,j}^{\ell-1}-1 \right) \Theta^2}{\left( D_{i,j}^{\ell-1} \right)^2}  \frac{D_{i,j}^{\ell-1}\sum_{d \in \mathcal{D}_{i,j}^{\ell-1}}\twonormsquare{d - \lambda_{i,j}^{\ell-1,(g)}} + D_{i,j}^{\ell-1}\sum_{d' \in \mathcal{D}_{i,j}^{\ell-1}}\twonormsquare{d' - \lambda_{i,j}^{\ell-1,(g)}}  }{D_{i,j}^{\ell-1}-1} \nonumber \\
    & = \frac{2\left( D_{i,j}^{\ell-1}-1 \right) \Theta^2}{D_{i,j}^{\ell-1}} \tilde{S}_{i,j}^{\ell-1,(g)} \nonumber
\end{align} 
Plugging (\ref{sample_variance_gradient_data_asynchronous_no_surrogate}) into (\ref{surrogate_b_term}), we have
\begin{equation}\label{square_no_surrogate}
\begin{aligned}
    &\mathsf{E}\left[ \twonormsquare{\nabla F_{i,j}^{\mathsf{R}}(\mathbf{w}_{i,j}^{\ell-1,(g)}, \mathcal{B}_{i,j}^{\ell-1,(g)})} \right] \\
    &\leq  4(1-\frac{B_{i,j}^{\ell-1,(g)}}{D_{i,j}^{\ell-1}})\frac{\left( D_{i,j}^{\ell-1}-1 \right) \Theta^2}{B_{i,j}^{\ell-1,(g)}D_{i,j}^{\ell-1}}\tilde{S}_{i,j}^{\ell-1,(g)} + 2 \mathsf{E}\left[ \twonormsquare{\nabla F_{i,j}^{\mathsf{R}}(\mathbf{w}_{i,j}^{\ell-1,(g)}, \mathcal{D}_{i,j}^{\ell-1})} \right] 
\end{aligned}
\end{equation}
Plugging (\ref{expanded_a_term}) and (\ref{square_no_surrogate}) into (\ref{start_no_surrogate}), we have
\begin{align}
    &\mathsf{E}\left[ F_j(\mathbf{w}_{i,j}^{\ell,(g)}) \right] \nonumber\\
    & \leq \mathsf{E}\left[ F_j(\mathbf{w}_{i,j}^{\ell-1,(g)}) \right] - \frac{\eta_{j}^{(g)}}{2}\mathsf{E}\left[\twonormsquare{\nabla F_j(\mathbf{w}_{i,j}^{\ell-1,(g)})}\right] \nonumber\\
    &+ \frac{\eta_{j}^{(g)}}{2}\underbrace{\mathsf{E}\left[\twonormsquare{\nabla F_j(\mathbf{w}_{i,j}^{\ell-1,(g)}) - \nabla F_{i,j}^{\mathsf{R}}(\mathbf{w}_{i,j}^{\ell-1,(g)}, \mathcal{D}_{i,j}^{\ell-1})}\right]}_{(a)} \label{global_local_gradient_loss_different} \\
    & + \left( \beta \left( \eta_{j}^{(g)} \right)^2 - \frac{\eta_{j}^{(g)}}{2} \right) \mathsf{E}\left[\twonormsquare{\nabla F_{i,j}^{\mathsf{R}}(\mathbf{w}_{i,j}^{\ell-1,(g)}, \mathcal{D}_{i,j}^{\ell-1})}\right] \nonumber\\
    &+ 2 \beta \left( \eta_{j}^{(g)} \right)^2\left(1-\frac{B_{i,j}^{\ell-1,(g)}}{D_{i,j}^{\ell-1}}\right)\frac{\left( D_{i,j}^{\ell-1}-1 \right) \Theta^2}{B_{i,j}^{\ell-1,(g)}D_{i,j}^{\ell-1}}\tilde{S}_{i,j}^{\ell-1,(g)} \nonumber 
\end{align}
Term (a) in \eqref{global_local_gradient_loss_different} can be further expanded as follows.
\begin{align}
&\mathsf{E}\left[\twonormsquare{\nabla F_j(\mathbf{w}_{i,j}^{\ell-1,(g)}) - \nabla F_{i,j}^{\mathsf{R}}(\mathbf{w}_{i,j}^{\ell-1,(g)}, \mathcal{D}_{i,j}^{\ell-1})}\right] \nonumber \\ 
&  = \mathsf{E}\left[\twonormsquare{\nabla F_j(\mathbf{w}_{i,j}^{\ell-1,(g)}) - \nabla F_{i,j}(\mathbf{w}_{i,j}^{\ell-1,(g)}, \mathcal{D}_{i,j}^{\ell-1})- \rho\left(\mathbf{w}_{i,j}^{\ell-1, (g)} -  \mathbf{w}_{i,j}^{(g)} \right)}  \right] \nonumber\\ 
&\leq 2 \mathsf{E}\left[\twonormsquare{\nabla F_j(\mathbf{w}_{i,j}^{\ell-1,(g)}) - \nabla F_{i,j}(\mathbf{w}_{i,j}^{\ell-1,(g)}, \mathcal{D}_{i,j}^{\ell-1})}  \right] + 2\rho^2 \mathsf{E}\left[\twonormsquare{\mathbf{w}_{i,j}^{\ell-1, (g)} -  \mathbf{w}_{i,j}^{(g)} }\right] \nonumber\\
& \leq 2 \mathsf{E}\left[\twonormsquare{\nabla F_j(\mathbf{w}_{i,j}^{\ell-1,(g)}) - \nabla F_{i,j}(\mathbf{w}_{i,j}^{\ell-1,(g)}, \mathcal{D}_{i,j}^{\ell-1})}  \right] + 2\rho^2 \mathsf{E}\left[\twonormsquare{\sum_{k=0}^{\ell-2}\nabla F_{i,j}^{\mathsf{R}}(\mathbf{w}_{i,j}^{k,(g)}, \mathcal{B}_{i,j}^{k,(g)}) }\right] \nonumber \\
& \leq 2 \mathsf{E}\left[\twonormsquare{\nabla F_j(\mathbf{w}_{i,j}^{\ell-1,(g)}) - \nabla F_{i,j}(\mathbf{w}_{i,j}^{\ell-1,(g)}, \mathcal{D}_{i,j}^{\ell-1})}  \right]  + 2\rho^2\left(\eta_j^{(g)}\right)^2 (\ell -1) V_2 \label{global_reg_local_diff_norm}
\end{align}
Plugging \eqref{global_reg_local_diff_norm} into term (a) in \eqref{global_local_gradient_loss_different}, we have
\begin{align}
    &\mathsf{E}\left[ F_j(\mathbf{w}_{i,j}^{\ell,(g)}) \right] \nonumber \\
    & \leq \mathsf{E}\left[ F_j(\mathbf{w}_{i,j}^{\ell-1,(g)}) \right] - \frac{\eta_{j}^{(g)}}{2}\mathsf{E}\left[\twonormsquare{\nabla F_j(\mathbf{w}_{i,j}^{\ell-1,(g)})}\right] \nonumber \\
    & + \frac{\eta_{j}^{(g)}}{2}\mathsf{E}\left[\twonormsquare{\nabla F_j(\mathbf{w}_{i,j}^{\ell-1,(g)}) - \nabla F_{i,j}(\mathbf{w}_{i,j}^{\ell-1,(g)}, \mathcal{D}_{i,j}^{\ell-1})}\right] \label{global_local_gradient_loss_different_true} \\
    & + \left( \beta \left( \eta_{j}^{(g)} \right)^2 - \frac{\eta_{j}^{(g)}}{2} \right) \mathsf{E}\left[\twonormsquare{\nabla F_{i,j}^{\mathsf{R}}(\mathbf{w}_{i,j}^{\ell-1,(g)}, \mathcal{D}_{i,j}^{\ell-1})}\right] \label{make_negative}  \\ 
    & +  2 \beta \left( \eta_{j}^{(g)} \right)^2\left(1-\frac{B_{i,j}^{\ell-1,(g)}}{D_{i,j}^{\ell-1}}\right)\frac{\left( D_{i,j}^{\ell-1}-1 \right) \Theta^2}{B_{i,j}^{\ell-1,(g)}D_{i,j}^{\ell-1}}\tilde{S}_{i,j}^{\ell-1,(g)} + \rho^2\left(\eta_j^{(g)}\right)^3 (\ell -1) V_2 \nonumber
\end{align}
By making use of the assumption that
\begin{align}
    \mathsf{E}\left[\twonormsquare{\nabla F_j(\mathbf{w}_{i,j}^{\ell-1,(g)}) - \nabla F_{i,j}(\mathbf{w}_{i,j}^{\ell-1,(g)}, \mathcal{D}_{i,j}^{\ell-1})}\right] \leq \delta_{i,j}^{(g)}
\end{align}
we can upper bound the term in (\ref{global_local_gradient_loss_different_true}). Also, by choosing $\eta_{j}^{(g)} < 1/(2\beta)$, we can make the term in (\ref{make_negative}) negative. Therefore, we have
\begin{equation} \label{local_final_minus_prev}
\begin{aligned}
    &\mathsf{E}\left[ F_j(\mathbf{w}_{i,j}^{\ell,(g)}) - F_j(\mathbf{w}_{i,j}^{\ell-1,(g)}) \right] \\
    & \leq - \frac{\eta_{j}^{(g)}}{2}\mathsf{E}\left[\twonormsquare{\nabla F_j(\mathbf{w}_{i,j}^{\ell-1,(g)})}\right] + \frac{\eta_{j}^{(g)}}{2}\delta_{i,j}^{(g)} + \rho^2\left(\eta_j^{(g)}\right)^3 (\ell -1) V_2\\
    & + 2 \beta \left( \eta_{j}^{(g)} \right)^2\left(1-\frac{B_{i,j}^{\ell-1,(g)}}{D_{i,j}^{\ell-1}}\right)\frac{\left( D_{i,j}^{\ell-1}-1 \right) \Theta^2}{B_{i,j}^{\ell-1,(g)}D_{i,j}^{\ell-1}}\tilde{S}_{i,j}^{\ell-1,(g)} 
\end{aligned}
\end{equation}
By arranging the terms and telescoping, we have
\begin{equation}\label{local_no_surrogate}
\begin{aligned}
    & \mathsf{E}\left[ F_j(\mathbf{w}_{i,j}^{\mathsf{F},(g)}) - F_j(\mathbf{w}_j^{(g)}) \right] \leq -\frac{\eta_{j}^{(g)}}{2}\sum_{\ell = 0}^{e_{i,j}^{(g)}-1}\mathsf{E}\left[\twonormsquare{\nabla F_j(\mathbf{w}_{i,j}^{\ell,(g)})}\right] + \frac{\eta_{j}^{(g)}}{2}e_{i,j}^{(g)}\delta_{i,j}^{(g)} \\
    & + \frac{\rho^2}{2}\left(\eta_j^{(g)}\right)^3 e_{i,j}^{(g)}\left(e_{i,j}^{(g)}- 1\right) V_2 + 2 \beta \left( \eta_{j}^{(g)} \right)^2\sum_{\ell = 0}^{e_{i,j}^{(g)}-1}\left(1-\frac{B_{i,j}^{\ell,(g)}}{D_{i,j}^{\ell}}\right)\frac{\left( D_{i,j}^{\ell}-1 \right) \Theta^2}{B_{i,j}^{\ell,(g)}D_{i,j}^{\ell}}\tilde{S}_{i,j}^{\ell,(g)}   \\
\end{aligned}
\end{equation}
Before continuing our derivation, we define the regularized global loss function as
\begin{equation}
F_j^{\mathsf{R}, (g)}\big(\mathbf{w}\big) = F_j\big(\mathbf{w}\big) + \frac{\rho}{2}\twonormsquare{\mathbf{w} - \mathbf{w}_{i,j}^{(g)}}
\end{equation}
Then, we have
\begin{align} 
    &\mathsf{E} \left[ F_j(\mathbf{w}_j^{(G)}) - F_j(\mathbf{w}_j^{(G-1)}) \right] \leq \mathsf{E} \left[ F_j^{\mathsf{R}, (G-1)}(\mathbf{w}_j^{(G)}) - F_j(\mathbf{w}_j^{(G-1)}) \right] \nonumber \\
    & \leq \sum_{g=0}^{G-1}\sum_{i \in \mathcal{I}} X_{i,j}^{g, G-1}\mathsf{E} \left[F_j^{\mathsf{R}, (G-1)}\left((1-\alpha_j)\mathbf{w}_{j}^{(G-1)}+\alpha_j \mathbf{w}_{i,j}^{\mathsf{F},(g)}\right) - F_j(\mathbf{w}_j^{(G-1)}) \right] \nonumber \\
    & \leq \sum_{g=0}^{G-1}\sum_{i \in \mathcal{I}} X_{i,j}^{g, G-1} \mathsf{E} \left[ (1-\alpha_j) F_j^{\mathsf{R}, (G-1)}(\mathbf{w}_{j}^{(G-1)}) + \alpha_j F_j^{\mathsf{R}, (G-1)}(\mathbf{w}_{i,j}^{\mathsf{F},(g)}) -  F_j(\mathbf{w}_j^{(G-1)}) \right]\label{eq:global_loss_convex_from} \\
    & \leq \sum_{g=0}^{G-1}\sum_{i \in \mathcal{I}} X_{i,j}^{g, G-1} \mathsf{E} \Big[ (1-\alpha_j) F_j(\mathbf{w}_{j}^{(G-1)}) + \alpha_j F_j(\mathbf{w}_{i,j}^{\mathsf{F},(g)}) -  F_j(\mathbf{w}_j^{(G-1)}) \\
    &+ \frac{\alpha_j\rho}{2}\twonormsquare{\mathbf{w}_{i,j}^{\mathsf{F},(g)} - \mathbf{w}_{j}^{(G-1)}} \Big] \nonumber\\ 
    & \leq \sum_{g=0}^{G-1}\sum_{i \in \mathcal{I}} X_{i,j}^{g, G-1}\mathsf{E}\left[ -\alpha_j F_j(\mathbf{w}_{j}^{(G-1)}) + \alpha_j F_j(\mathbf{w}_{i,j}^{\mathsf{F},(g)})\right] \label{eq:global_loss_convex_to} \\
    & + \sum_{g=0}^{G-1}\sum_{i \in \mathcal{I}} X_{i,j}^{g, G-1}\frac{\alpha_j\rho}{2}\twonormsquare{\mathbf{w}_{i,j}^{\mathsf{F},(g)} - \mathbf{w}_{j}^{(G-1)}} \nonumber\\
    & \leq \sum_{g=0}^{G-1}\sum_{i \in \mathcal{I}} X_{i,j}^{g, G-1}\alpha_j\mathsf{E}\left[  \left( F_j(\mathbf{w}_{i,j}^{\mathsf{F},(g)}) -  F_j(\mathbf{w}_{j}^{(G-1)})\right) \right] + \sum_{g=0}^{G-1}\sum_{i \in \mathcal{I}} X_{i,j}^{g, G-1}\frac{\alpha_j\rho}{2}\twonormsquare{\mathbf{w}_{i,j}^{\mathsf{F},(g)} - \mathbf{w}_{j}^{(G-1)}}\nonumber\\
    & \leq \sum_{g=0}^{G-1}\sum_{i \in \mathcal{I}} X_{i,j}^{g, G-1}\alpha_j\mathsf{E}\left[  \left( F_j(\mathbf{w}_{i,j}^{\mathsf{F},(g)}) - F_j(\mathbf{w}_j^{(g)}) + F_j(\mathbf{w}_j^{(g)}) -  F_j(\mathbf{w}_{j}^{(G-1)})\right) \right]  \nonumber\\
    & + \sum_{g=0}^{G-1}\sum_{i \in \mathcal{I}} X_{i,j}^{g, G-1}\frac{\alpha_j\rho}{2}\twonormsquare{\mathbf{w}_{i,j}^{\mathsf{F},(g)} - \mathbf{w}_{j}^{(g)} + \mathbf{w}_{j}^{(g)} - \mathbf{w}_{j}^{(G-1)}}\nonumber\\
    & \leq \sum_{g=0}^{G-1}\sum_{i \in \mathcal{I}} X_{i,j}^{g, G-1}\alpha_j\mathsf{E}\left[ F_j(\mathbf{w}_{i,j}^{\mathsf{F},(g)}) - F_j(\mathbf{w}_j^{(g)})  \right] + \sum_{g=0}^{G-1}\sum_{i \in \mathcal{I}} X_{i,j}^{g, G-1}\alpha_j\mathsf{E}\left[ F_j(\mathbf{w}_j^{(g)}) -  F_j(\mathbf{w}_{j}^{(G-1)}) \right] \nonumber \\
    &+ \sum_{g=0}^{G-1}\sum_{i \in \mathcal{I}} X_{i,j}^{g, G-1}\alpha_j\rho\twonormsquare{\mathbf{w}_{i,j}^{\mathsf{F},(g)} - \mathbf{w}_{j}^{(g)}}  + \sum_{g=0}^{G-1}\sum_{i \in \mathcal{I}} X_{i,j}^{g, G-1}\alpha_j\rho\twonormsquare{\mathbf{w}_{j}^{(g)} - \mathbf{w}_{j}^{(G-1)}} \nonumber \\
    &  \leq \sum_{g=0}^{G-1}\sum_{i \in \mathcal{I}} X_{i,j}^{g, G-1}\alpha_j\mathsf{E}\left[ F_j(\mathbf{w}_{i,j}^{\mathsf{F},(g)}) - F_j(\mathbf{w}_j^{(g)})  \right] \nonumber\\
    &+ \sum_{g=0}^{G-1}\sum_{i \in \mathcal{I}} X_{i,j}^{g, G-1}\alpha_j\mathsf{E}\left[ \underbrace{F_j(\mathbf{w}_j^{(g)}) -  F_j(\mathbf{w}_{j}^{(G-1)})}_{(a)} \right] \label{intermediate_g_g_1_no_surrogate} \\
    &+ \sum_{g=0}^{G-1}\sum_{i \in \mathcal{I}} X_{i,j}^{g, G-1}\alpha_j\rho\left(\eta_j^{(g)}\right)^2 V_2 \left(e_{i,j}^{(g)}\right)^2  + \sum_{g=0}^{G-1}\sum_{i \in \mathcal{I}} X_{i,j}^{g, G-1}\alpha_j\rho\twonormsquare{\mathbf{w}_{j}^{(g)} - \mathbf{w}_{j}^{(G-1)}} \nonumber 
\end{align}
From \eqref{eq:global_loss_convex_from} to \eqref{eq:global_loss_convex_to}, we use the definition of the regularized loss function and the weakly convexity of the regularized global loss function.
Now we need to further upper bound the term $(a)$ in (\ref{intermediate_g_g_1_no_surrogate}). Using the smoothness, we have
\begin{equation}\label{difference_tao_g_1_no_surrogate} 
\begin{aligned}
    &F_j(\mathbf{w}_j^{(g)}) -  F_j(\mathbf{w}_j^{(G-1)}) \leq \innerproduct{\nabla F_j(\mathbf{w}_j^{(G-1)})}{\mathbf{w}_j^{(g)} -\mathbf{w}_j^{(G-1)}} + \frac{\beta}{2} \twonormsquare{\mathbf{w}_j^{(g)} -\mathbf{w}_j^{(G-1)}} \\
    & \leq \norm{\nabla F_j(\mathbf{w}_j^{(G-1)})} \norm{\mathbf{w}_j^{(g)} -\mathbf{w}_j^{(G-1)}} + \frac{\beta}{2} \twonormsquare{\mathbf{w}_j^{(g)} -\mathbf{w}_j^{(G-1)}} \\
    & \leq \sqrt{V_1} \norm{\mathbf{w}_j^{(g)} -\mathbf{w}_j^{(G-1)}} + \frac{\beta}{2} \twonormsquare{\mathbf{w}_j^{(g)} -\mathbf{w}_j^{(G-1)}} 
\end{aligned}
\end{equation}

We need to find the upper bound for $\twonormsquare{\mathbf{w}_j^{(g)} - \mathbf{w}_j^{(G-1)}}$. Based on the recursive relationship \Cref{lemma:recursive_diff_consecutive_global} and $\lVert \nabla F_{i,j}^{\mathsf{R}}(\mathbf{w},d) \rVert^2 \leq V_2$, we have 
\begin{equation}\label{tao_g_1_difference_norm_2_no_surrogate}
\twonormsquare{\mathbf{w}_j^{(g)} - \mathbf{w}_j^{(G-1)}} \leq 
\begin{cases}
\frac{K_j\alpha_j  \left(e_j^{\mathsf{max}}\right)^2\left(\eta_j^{\mathsf{max}}\right)^2V_2\left(1-\left(K_j\alpha_j\right)^{G-1}\right)}{1 - K_j\alpha_j}, &~\text{if} ~ K_j\alpha_j \neq 1  \\
 (G-1) \left(e_j^{\mathsf{max}}\right)^2\left(\eta_j^{\mathsf{max}}\right)^2V_2, &~\text{if}~ K_j\alpha_j = 1
\end{cases}
\end{equation}
Therefore, we also have
\begin{equation}\label{tao_g_1_difference_norm_1_no_surrogate} 
\norm{\mathbf{w}_j^{(g)} - \mathbf{w}_j^{(G-1)}} \leq 
\begin{cases}
\sqrt{\frac{K_j\alpha_j \left(e_j^{\mathsf{max}}\right)^2\left(\eta_j^{\mathsf{max}}\right)^2V_2\left(1-\left(K_j\alpha_j\right)^{G-1}\right)}{1 - K_j\alpha_j}} & ~\text{if} ~ K_j\alpha_j \neq 1\\
\sqrt{(G-1) \left(e_j^{\mathsf{max}}\right)^2\left(\eta_j^{\mathsf{max}}\right)^2V_2}  &~\text{if}~ K_j\alpha_j = 1
\end{cases}
\end{equation}
Plugging (\ref{tao_g_1_difference_norm_2_no_surrogate}), (\ref{tao_g_1_difference_norm_1_no_surrogate}) into (\ref{difference_tao_g_1_no_surrogate}), we have
\begin{equation}\label{g_G_min_1_norm}
\begin{aligned}
F_j(\mathbf{w}_j^{(g)}) -  F_j(\mathbf{w}_j^{(G-1)}) &\leq \sqrt{\frac{K_j\alpha_j \left(e_j^{\mathsf{max}}\right)^2\left(\eta_j^{\mathsf{max}}\right)^2 V_1 V_2\left(1-\left(K_j\alpha_j\right)^{G-1}\right)}{1 - K_j\alpha_j}} \\
&+ \frac{\beta}{2}\frac{K_j\alpha_j \left(e_j^{\mathsf{max}}\right)^2\left(\eta_j^{\mathsf{max}}\right)^2V_2\left(1-\left(K_j\alpha_j\right)^{G-1}\right)}{1 - K_j\alpha_j} 
\end{aligned}
\end{equation}
At this point, we need to add the concept drift to the bound based on \Cref{def:concept_drift_active_concept_drift} and \Cref{def:rectangular_functions}. Plugging \eqref{local_no_surrogate}, \eqref{tao_g_1_difference_norm_2_no_surrogate}, and \eqref{g_G_min_1_norm} into \eqref{intermediate_g_g_1_no_surrogate}, we have
\begin{align}
    &\mathsf{E} \left[ F_j(\mathbf{w}_j^{(G)}) - F_j(\mathbf{w}_j^{(G-1)}) \right] \label{G_G_minus_1_rhs}\\
    &\leq -\frac{\alpha_j}{2}\sum_{g=0}^{G-1}\sum_{i \in \mathcal{I}} X_{i,j}^{g, G-1}\eta_{j}^{(g)}\sum_{\ell = 0}^{e_{i,j}^{(g)}-1}\mathsf{E}\left[\twonormsquare{\nabla F_j(\mathbf{w}_{i,j}^{\ell,(g)})}\right] \frac{\alpha_j}{2}\sum_{g=0}^{G-1}\sum_{i \in \mathcal{I}} X_{i,j}^{g, G-1}\eta_{j}^{(g)}e_{i,j}^{(g)}\delta_{i,j}^{(g)} \nonumber \\
    &+ 2 \alpha_j\beta \sum_{g=0}^{G-1}\sum_{i \in \mathcal{I}} X_{i,j}^{g, G-1}\left( \eta_{j}^{(g)} \right)^2\sum_{\ell = 0}^{e_{i,j}^{(g)}-1}\left(1-\frac{B_{i,j}^{\ell,(g)}}{D_{i,j}^{\ell}}\right)\frac{\left( D_{i,j}^{\ell}-1 \right) \Theta^2}{B_{i,j}^{\ell,(g)}D_{i,j}^{\ell}}\tilde{S}_{i,j}^{\ell,(g)} \nonumber\\
    & + \frac{\rho^2 \alpha_j}{2}\sum_{g=0}^{G-1}\sum_{i \in \mathcal{I}} X_{i,j}^{g, G-1}\left(\eta_j^{(g)}\right)^3 e_{i,j}^{(g)}\left(e_{i,j}^{(g)}- 1\right) V_2  + \sum_{g=0}^{G-1}\sum_{i \in \mathcal{I}} X_{i,j}^{g, G-1}\alpha_j\rho\left(\eta_j^{(g)}\right)^2 V_2 \left(e_{i,j}^{(g)}\right)^2 \nonumber\\ 
    & + \sqrt{\frac{K_j\alpha_j^3  \left(e_j^{\mathsf{max}}\right)^2\left(\eta_j^{\mathsf{max}}\right)^2 V_1 V_2\left(1-\left(K_j\alpha_j\right)^{G-1}\right)}{1 - K_j\alpha_j}}\nonumber \\
    &+ \left(\frac{\beta}{2} + 1\right)\frac{K_j\alpha_j^2  \left(e_j^{\mathsf{max}}\right)^2\left(\eta_j^{\mathsf{max}}\right)^2 V_2\left(1-\left(K_j\alpha_j\right)^{G-1}\right)}{1 - K_j\alpha_j} \nonumber\\
    & + \alpha_j\sum_{g=0}^{G-1}\sum_{i \in \mathcal{I}} X_{i,j}^{g, G-1} \left(\sum_{0 \leq t \leq T^G}y_{i,j}^{\mathsf{AC}, (g)}(t)\Delta_{i,j}^{\mathsf{AC}}(t)  + \sum_{0 \leq t \leq T^G}y_{i,j}^{\mathsf{ID}, (g)}(t)\Delta_{i,j}^{\mathsf{ID}}(t) \right) \nonumber
\end{align}
We can further find the upper bound of the first term in the right-hand side of the inequality in \eqref{G_G_minus_1_rhs} as follows
\begin{align}
&-\frac{\alpha_j}{2}\sum_{g=0}^{G-1}\sum_{i \in \mathcal{I}} X_{i,j}^{g, G-1}\eta_{j}^{(g)}\sum_{\ell = 0}^{e_{i,j}^{(g)}-1}\mathsf{E}\left[\twonormsquare{\nabla F_j(\mathbf{w}_{i,j}^{\ell,(g)})}\right] \label{G_G_minus_1_rhs_1} \\
&\leq -\frac{\alpha_j}{2}\sum_{g=0}^{G-1}\sum_{i \in \mathcal{I}} X_{i,j}^{g, G-1}\eta_{j}^{(g)}\mathsf{E}\left[\twonormsquare{\nabla F_j(\mathbf{w}_{i,j}^{0,(g)})}\right] \label{G_G_minus_1_rhs_2}\\ 
& \leq -\frac{\alpha_j}{2}\sum_{g=0}^{G-1}\sum_{i \in \mathcal{I}} X_{i,j}^{g, G-1}\eta_{j}^{(g)}\mathsf{E}\left[\twonormsquare{\nabla F_j(\mathbf{w}_j^{(g)})}\right] \label{G_G_minus_1_rhs_3}
\end{align}
From \eqref{G_G_minus_1_rhs_1} to \eqref{G_G_minus_1_rhs_2}, we retain only one term in the inner summation. From \eqref{G_G_minus_1_rhs_2} to \eqref{G_G_minus_1_rhs_3}, we just used the definition that $\mathbf{w}_{i,j}^{0, (g)} = \mathbf{w}_j^{(g)}$ for the device $i$ that gets $\mathbf{w}_j^{0, (g)}$. Plugging \eqref{G_G_minus_1_rhs_3} into \eqref{G_G_minus_1_rhs}, rearranging the term, and dividing $\alpha_j/2$, we have 
\begin{align*}
    & \sum_{g=0}^{G-1}\sum_{i \in \mathcal{I}} X_{i,j}^{g, G-1}\eta_{j}^{(g)}\mathsf{E}\left[\twonormsquare{\nabla F_j(\mathbf{w}_j^{(g)})}\right] \\
    &\leq \frac{2}{\alpha_j}\mathsf{E} \left[ F_j(\mathbf{w}_j^{(G-1)}) - F_j(\mathbf{w}_j^{(G)}) \right]  + \sum_{g=0}^{G-1}\sum_{i \in \mathcal{I}} X_{i,j}^{g, G-1}\eta_{j}^{(g)}e_{i,j}^{(g)}\delta_{i,j}^{(g)}\\
    &+ 4\beta \sum_{g=0}^{G-1}\sum_{i \in \mathcal{I}} X_{i,j}^{g, G-1}\left( \eta_{j}^{(g)} \right)^2\sum_{\ell = 0}^{e_{i,j}^{(g)}-1}\left(1-\frac{B_{i,j}^{\ell,(g)}}{D_{i,j}^{\ell}}\right)\frac{\left( D_{i,j}^{\ell}-1 \right) \Theta^2}{B_{i,j}^{\ell,(g)}D_{i,j}^{\ell}}\tilde{S}_{i,j}^{\ell,(g)} \\
    & + 2\rho\sum_{g=0}^{G-1}\sum_{i \in \mathcal{I}} X_{i,j}^{g, G-1}\left(\eta_j^{(g)}\right)^2 e_{i,j}^{(g)}V_2\left(\frac{\rho}{2}\eta_j^{(g)}(e_{i,j}^{(g)}-1) + e_{i,j}^{(g)}\right) \\ 
    & + \sqrt{\frac{4 K_j\alpha_j  \left(e_j^{\mathsf{max}}\right)^2\left(\eta_j^{\mathsf{max}}\right)^2 V_1 V_2\left(1-\left(K_j\alpha_j\right)^{G-1}\right)}{1 - K_j\alpha_j}} \\
    &+ \left(\beta+2\right)\frac{K_j\alpha_j  \left(e_j^{\mathsf{max}}\right)^2\left(\eta_j^{\mathsf{max}}\right)^2 V_2\left(1-\left(K_j\alpha_j\right)^{G-1}\right)}{1 - K_j\alpha_j} \nonumber\\
    &+ \sum_{g=0}^{G-1}\sum_{i \in \mathcal{I}} X_{i,j}^{g, G-1} \left(\sum_{0 \leq t \leq T^G}y_{i,j}^{\mathsf{AC}, (g)}(t)\Delta_{i,j}^{\mathsf{AC}}(t)  + \sum_{0 \leq t \leq T^G}y_{i,j}^{\mathsf{ID}, (g)}(t)\Delta_{i,j}^{\mathsf{ID}}(t) \right) \nonumber
\end{align*}
Summing over the global aggregation and dividing by the number of global aggregations $G_j$, we have
\begin{align*}
    & \frac{1}{G_j}\sum_{g' = 0}^{G_j-1}\sum_{g=0}^{g'}\sum_{i \in \mathcal{I}} X_{i,j}^{g, g'}\eta_{j}^{(g)}\mathsf{E}\left[\twonormsquare{\nabla F_j(\mathbf{w}_j^{(g)})}\right] \leq \frac{2}{G_j\alpha_j}\mathsf{E} \left[ F_j(\mathbf{w}_j^{(0)}) - F_j(\mathbf{w}_j^{(G_j)}) \right]  \\
  & + \frac{1}{G_j}\sum_{g' = 0}^{G_j-1}\sum_{g=0}^{g'}\sum_{i \in \mathcal{I}} X_{i,j}^{g, g'}\eta_{j}^{(g)}e_{i,j}^{(g)}\delta_{i,j}^{(g)} \\
  &+ \frac{4\beta}{G_j} \sum_{g' = 0}^{G_j-1}\sum_{g=0}^{g'}\sum_{i \in \mathcal{I}} X_{i,j}^{g, g'}\left( \eta_{j}^{(g)} \right)^2\sum_{\ell = 0}^{e_{i,j}^{(g)}-1}\left(1-\frac{B_{i,j}^{\ell,(g)}}{D_{i,j}^{\ell}}\right)\frac{\left( D_{i,j}^{\ell}-1 \right) \Theta^2}{B_{i,j}^{\ell,(g)}D_{i,j}^{\ell}}\tilde{S}_{i,j}^{\ell,(g)} \\
    & +  \frac{2\rho}{G_j}\sum_{g' = 0}^{G_j-1}\sum_{g=0}^{g'}\sum_{i \in \mathcal{I}} X_{i,j}^{g, g'}\left(\eta_j^{(g)}\right)^2 e_{i,j}^{(g)}V_2\left(\frac{\rho}{2}\eta_j^{(g)}(e_{i,j}^{(g)}-1) + e_{i,j}^{(g)}\right) \\ 
    & + \frac{1}{G_j}\sum_{g' = 0}^{G_j-1}\sqrt{\frac{4 K_j\alpha_j  \left(e_j^{\mathsf{max}}\right)^2\left(\eta_j^{\mathsf{max}}\right)^2 V_1 V_2\left(1-\left(K_j\alpha_j\right)^{g'}\right)}{1 - K_j\alpha_j}} \\
    & + \frac{1}{G_j}\sum_{g' = 0}^{G_j-1}(\beta+2)\frac{K_j\alpha_j  \left(e_j^{\mathsf{max}}\right)^2\left(\eta_j^{\mathsf{max}}\right)^2 V_2\left(1-\left(K_j\alpha_j\right)^{g'}\right)}{1 - K_j\alpha_j} \\
    & + \frac{1}{G_j}\sum_{g' = 0}^{G_j-1}\sum_{g=0}^{g'}\sum_{i \in \mathcal{I}}X_{i,j}^{g, g'} \left(\sum_{0 \leq t \leq T^{G_j}}y_{i,j}^{\mathsf{AC}, (g)}(t)\Delta_{i,j}^{\mathsf{AC}}(t)  + \sum_{0 \leq t \leq T^{G_j}}y_{i,j}^{\mathsf{ID}, (g)}(t)\Delta_{i,j}^{\mathsf{ID}}(t) \right)
\end{align*}
To further simplify the convergence bound, let $\eta_j^{\mathsf{min}} = \min\{\eta_j^{(g)}\}_{g=0}^G$ 
\begin{align*}
    & \frac{1}{G_j}\sum_{g' = 0}^{G_j-1}\sum_{g=0}^{g'}\sum_{i \in \mathcal{I}} X_{i,j}^{g, g'}\mathsf{E}\left[\twonormsquare{\nabla F_j(\mathbf{w}_j^{(g)})}\right] \leq \frac{2}{G_j\eta_j^{\mathsf{min}}\alpha_j}\mathsf{E} \left[ F_j(\mathbf{w}_j^{(0)}) - F_j(\mathbf{w}_j^{(G_j)}) \right]  \\
  & + \frac{1}{G_j\eta_j^{\mathsf{min}}}\sum_{g' = 0}^{G_j-1}\sum_{g=0}^{g'}\sum_{i \in \mathcal{I}} X_{i,j}^{g, g'}\eta_{j}^{(g)}e_{i,j}^{(g)}\delta_{i,j}^{(g)} \\
  & + \frac{4\beta}{G_j\eta_j^{\mathsf{min}}} \sum_{g' = 0}^{G_j-1}\sum_{g=0}^{g'}\sum_{i \in \mathcal{I}} X_{i,j}^{g, g'}\left( \eta_{j}^{(g)} \right)^2\sum_{\ell = 0}^{e_{i,j}^{(g)}-1}\left(1-\frac{B_{i,j}^{\ell,(g)}}{D_{i,j}^{\ell}}\right)\frac{\left( D_{i,j}^{\ell}-1 \right) \Theta^2}{B_{i,j}^{\ell,(g)}D_{i,j}^{\ell}}\tilde{S}_{i,j}^{\ell,(g)} \\
    & +  \frac{2\rho}{G_j\eta_j^{\mathsf{min}}}\sum_{g' = 0}^{G_j-1}\sum_{g=0}^{g'}\sum_{i \in \mathcal{I}} X_{i,j}^{g, g'}\left(\eta_j^{(g)}\right)^2 e_{i,j}^{(g)}V_2\left(\frac{\rho}{2}\eta_j^{(g)}(e_{i,j}^{(g)}-1) + e_{i,j}^{(g)}\right) \\ 
    & + \frac{1}{G_j \eta_j^{\mathsf{min}}}\sum_{g' = 0}^{G_j-1}\sqrt{\frac{4 K_j\alpha_j  \left(e_j^{\mathsf{max}}\right)^2\left(\eta_j^{\mathsf{max}}\right)^2 V_1 V_2\left(1-\left(K_j\alpha_j\right)^{g'}\right)}{1 - K_j\alpha_j}} \\
    & +\frac{1}{G_j \eta_j^{\mathsf{min}}}\sum_{g' = 0}^{G_j-1} (\beta+2)\frac{K_j\alpha_j  \left(e_j^{\mathsf{max}}\right)^2\left(\eta_j^{\mathsf{max}}\right)^2 V_2\left(1-\left(K_j\alpha_j\right)^{g'}\right)}{1 - K_j\alpha_j} \\
    & + \frac{1}{G_j\eta_j^{\mathsf{min}}}\sum_{g' = 0}^{G_j-1}\sum_{g=0}^{g'}\sum_{i \in \mathcal{I}}X_{i,j}^{g, g'} \left(\sum_{t=0}^{T^{G_j}}y_{i,j}^{\mathsf{AC}, (g)}(t)\Delta_{i,j}^{\mathsf{AC}}(t)  + \sum_{t=0}^{T^{G_j}}y_{i,j}^{\mathsf{ID}, (g)}(t)\Delta_{i,j}^{\mathsf{ID}}(t) \right) 
\end{align*}

\section{Proof of Corollary \ref{corollary:simplified_left_hand_side}} \label{proof_of_corollary_1}
According to the definition of our device scheduling tensor in \eqref{eq:device_scheduling_cases}, if whenever the server updates the global model, it activates at least one device for local model training with the updated model, $\mathsf{Conv}_j$ first reduces to the weighted sum of $\mathsf{E}\left[\twonormsquare{\nabla F_j(\mathbf{w}_j^{(g)})}\right]$. Specifically, for $G_j$ global aggregations, let $m_j$ be the number of times that the global model $\mathbf{w}_j^{(g)}$ has been used to activate the local training of devices. Based on our formulation, we naturally have $\sum_j m_j = G_j$. Then, we have 
\begin{equation}
\mathsf{Conv}_j = \frac{\sum_j m_j\mathsf{E}\left[\twonormsquare{\nabla F_j(\mathbf{w}_j^{(g)})}\right]}{G_j} = \frac{\sum_j m_j\mathsf{E}\left[\twonormsquare{\nabla F_j(\mathbf{w}_j^{(g)})}\right]}{\sum_j m_j}
\end{equation}
which is lower bounded by $\min_{g= 0}^{G_j-1}\mathsf{E}[\lVert\nabla F_j(\mathbf{w}_j^{(g)})\rVert^2]$. For the second case, if the server activates exactly one device after every update, then $m_j = 1, \, \forall j$. Thus, we have, 
\begin{equation}
\mathsf{Conv}_j = \frac{\sum_j \mathsf{E}\left[\twonormsquare{\nabla F_j(\mathbf{w}_j^{(g)})}\right]}{G_j} 
\end{equation}

\section{simulation results for the other two non-iid dataset partitions}
\subsection{Simulation Results for Varied Partition}

\begin{figure*}[ht]
\centering
\includegraphics[width = \textwidth]{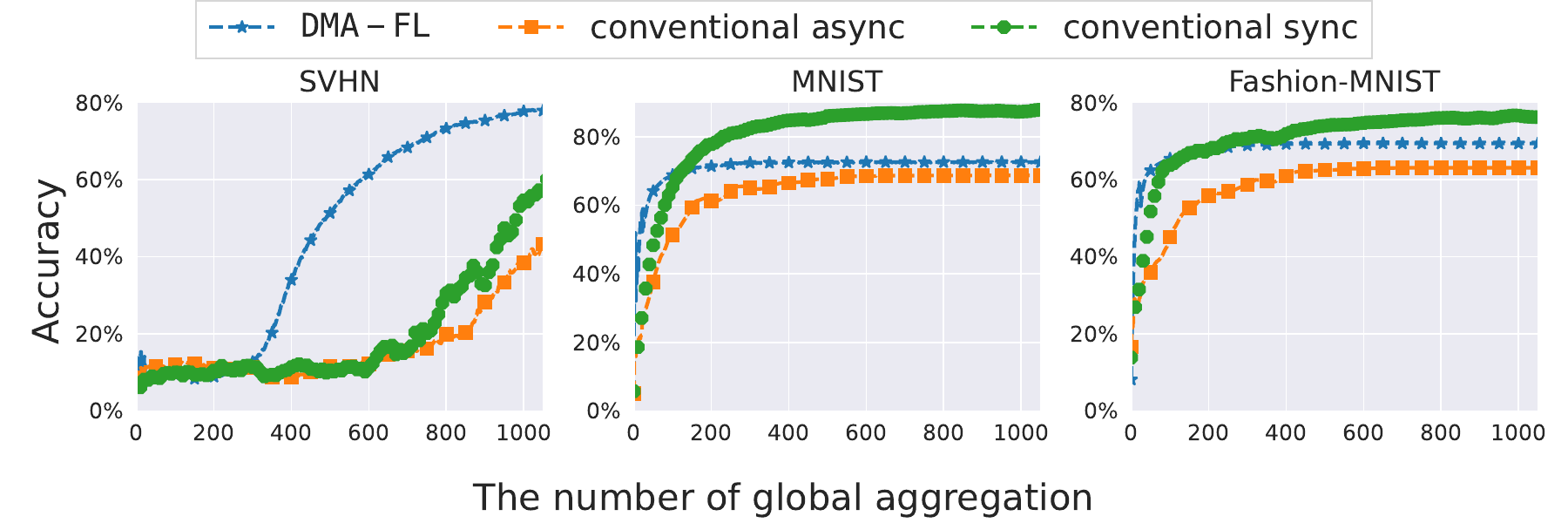}
\caption{ML training convergence of all schemes for all tasks. {\tt DMA-FL} boosts the performance of the naturally emphasized task (SVHN) dramatically while obtaining an accuracy convergence somewhere between the baselines for the other tasks.}
\label{fig:accuracy_vs_num_global_aggregation_all_task}
\vspace{3mm}
\centering
\includegraphics[width = \textwidth]{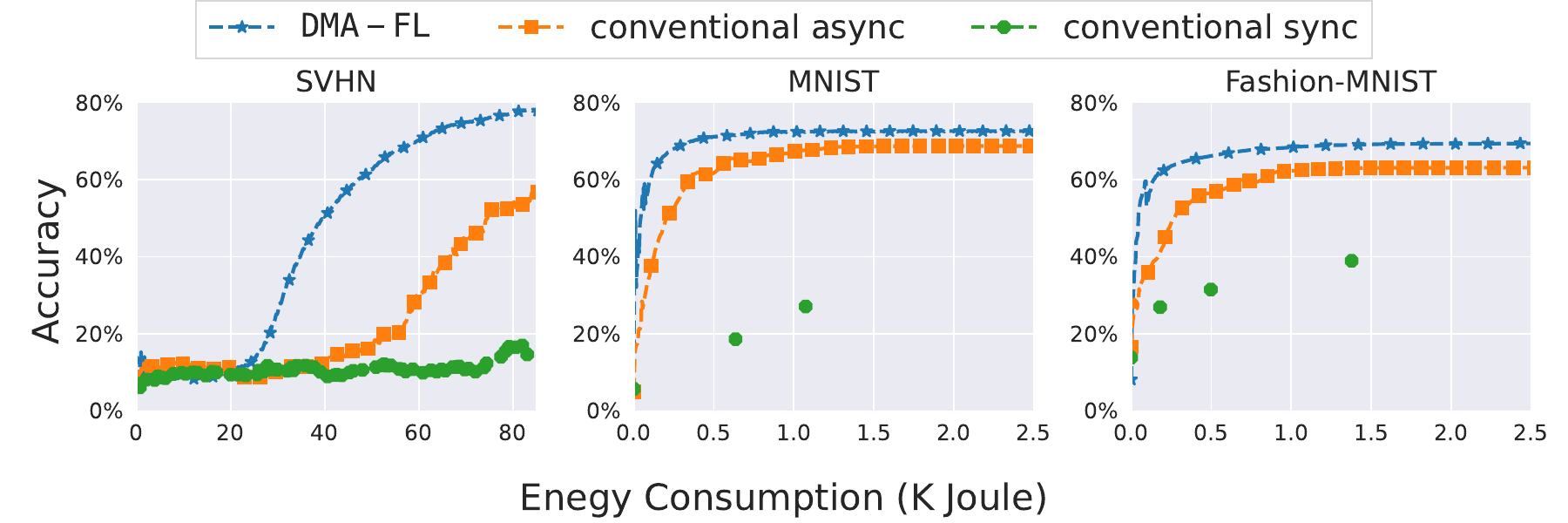}
\caption{Accuracy vs. energy consumption trade-off obtained by each method across all tasks. Even when one task is emphasized (SVHN), {\tt DMA-FL} is able to obtain different target accuracies under lower energy consumption across all the tasks. Specifically, for the synchronous scheme (green curve), the energy needed to reach $65\%$ (for SVHN), $70\%$ (for MNIST) and $70\%$ (for Fasion-MNIST) accuracy is $146.98$, $11.70$, and $28.29$ K Joule respectively.}
\label{fig:accuracy_vs_energy_consumption_all_task}
\end{figure*}

We first compare the performance of {\tt DMA-FL} with optimized device scheduling (obtained through solving $\mathcal{P}$) against two baselines: (i) conventional asynchronous FL \cite{xie2019asynchronous} and (ii) conventional synchronous FL \cite{mcmahan2017communication}.
In conventional asynchronous FL, the idle times of devices are randomly chosen, while the resource allocation is conducted according to $(\bm{\mathcal{P}})$ with the chosen idle times. In conventional asynchronous FL, each global aggregation is performed whenever the server gets a trained model. On the other hand, in the conventional synchronous scheme, each global aggregation is performed when the server receives \textit{all the trained local models}. 

Fig. \ref{fig:accuracy_vs_num_global_aggregation_all_task} compares the convergence behavior of the algorithms. For all tasks, we see that our proposed scheme (blue curve) has a superior performance compared to the conventional asynchronous scheme (orange curve), attributed to our proposed scheme optimizing both device scheduling and idle times of devices. Also, for all tasks, the conventional synchronous (green curve) scheme outperforms the conventional asynchronous scheme. This is because each global aggregation in the synchronous regime translates to having all the devices engaging in uplink transmissions, which naturally would lead to better performance when only one device engages in uplink transmission as in the asynchronous scheme. The substantial accuracy improvements obtained by {\tt DMA-FL} on the SVHN task is due to the optimization allocating more resources to handle its higher training complexity compared with MNIST and FMNIST.

Although the conventional synchronous scheme achieves a higher accuracy than our proposed scheme for the lower complexity tasks (i.e., MNIST and FMNIST), Fig. \ref{fig:accuracy_vs_energy_consumption_all_task} shows that this comes at the expense of much higher resource utilization. Specifically, in Fig. \ref{fig:accuracy_vs_energy_consumption_all_task}, we plot the observed energy consumption for each task to reach the corresponding accuracy levels in Fig. \ref{fig:accuracy_vs_num_global_aggregation_all_task}. 
As can be seen, the asynchronous training styles are much more resource efficient since they can skip engaging the stragglers (i.e., devices with higher energy consumption) in model aggregations.
Further, our proposed scheme can reach the same accuracy as the conventional asynchronous method under less energy consumption due to the optimized device scheduling. 

\subsection{Simulation Results for 2-label Partition}
\begin{figure}[ht]
\centering
\includegraphics[width = \textwidth]{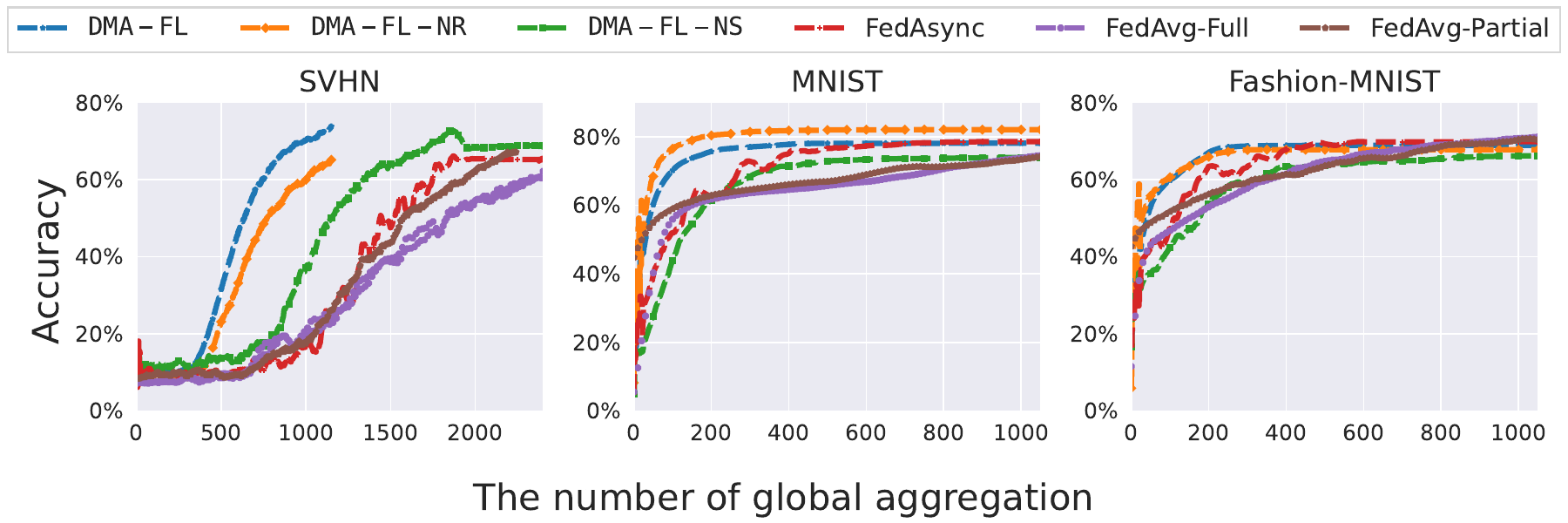}
\caption{ML training performance over global aggregations of all schemes for all tasks, under the non-iid partitioning of each device having 2 labels. {\tt DMA-FL} significantly outperforms the baselines on the SVHN task, and performs comparably to the best baselines on MNIST and Fashion-MNIST. As we will see in Figure \ref{fig:acc_energy_2_labels}, the baselines also incur significantly higher energy consumption on each task.}
\label{fig:acc_iteration_2_labels}
\includegraphics[width = \textwidth]{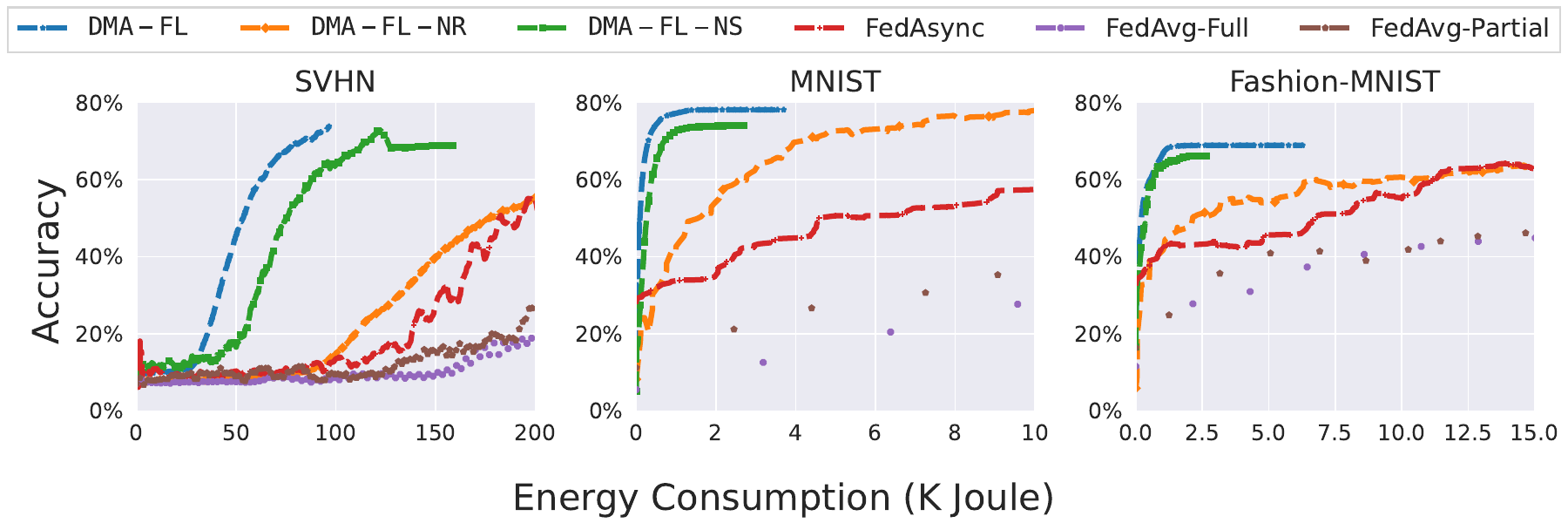}
\caption{Accuracy vs. energy consumption trade-off obtained by each method across all tasks with a non-iid partitioning of each device having 2 labels. Our proposed scheme can reach a target level of accuracy with significantly less energy consumption in comparison with all the baselines on each task.}
\label{fig:acc_energy_2_labels}
\end{figure}
Fig. \ref{fig:acc_iteration_2_labels} compares the convergence behavior of the algorithms under the 2-label partitioning from Comment 3. Fig. \ref{fig:acc_energy_2_labels} presents the corresponding energy consumption plots. Overall, the key messages regarding substantial improvement in accuracy-energy tradeoff obtained by {\tt DMA-FL} remain the same with the addition of these new baselines. Specifically, in Fig. \ref{fig:acc_iteration_2_labels}, we see that {\tt DMA-FL} obtains substantial improvements in convergence compared to all baselines on SVHN, marginal improvement in MNIST, and comparable performance on Fashion-MNIST. The improvement on SVHN is expected due to the higher complexity of this task, translating to larger loss bounds in the objective function of $\bm{\mathcal{P}}$. In Fig. \ref{fig:acc_energy_2_labels}, we see that the baselines require substantially higher energy consumption to reach target accuracy levels, validating the gains provided by {\tt DMA-FL}'s joint optimization of device scheduling and resource allocation for heterogeneous asynchronous FL.

We also see that the rest of the asynchronous schemes ({\tt DMA-FL-NR}, and {\tt DMA-FL-NS}, and FedAsync) outperform the synchronous schemes (FedAvg-Full and FedAvg-Partial) in terms of energy consumption in Fig. \ref{fig:acc_energy_2_labels}. The asynchronous training styles are more resource efficient since they can skip engaging devices with higher energy consumption during specific model aggregation iterations. Among the asynchronous schemes, we can verify the benefits of optimized resource allocation in terms of improved energy efficiency by comparing {\tt DMA-FL} and {\tt DMA-FL-NS} with {\tt DMA-FL-NR} and FedAsync. As shown, {\tt DMA-FL} and {\tt DMA-FL-NS} obtain a better accuracy-energy tradeoff than the other asynchronous schemes in Fig. \ref{fig:acc_energy_2_labels}. On the other hand, in Fig. \ref{fig:acc_iteration_2_labels}, we see that {\tt DMA-FL-NR} performs closest to {\tt DMA-FL} (even outperforming it for MNIST), obtaining better convergence over global aggregations than {\tt DMA-FL-NS}. Since {\tt DMA-FL-NR} is not considering resource consumption, it optimizes the device scheduling for convergence speed, but consumes significant energy on each task. {\tt DMA-FL} balances both of these objectives to obtain the best overall performance.
\end{document}